\documentclass[twoside]{article}

\usepackage[accepted]{aistats2024}

\usepackage{macros}

\begin{document}

\twocolumn[
\aistatstitle{Leveraging Continuous Time to Understand Momentum When Training Diagonal Linear Networks}

\aistatsauthor{ Hristo Papazov$^{*}$ \And Scott Pesme$^{*}$ \And  Nicolas Flammarion }

\aistatsaddress{ EPFL\\ \texttt{hristo.papazov@epfl.ch} \And  EPFL\\ \texttt{scott.pesme@epfl.ch} \And EPFL\\ \texttt{nicolas.flammarion@epfl.ch} }
]

\begin{abstract}
In this work, we investigate the effect of momentum on the optimisation trajectory of gradient descent. We leverage a continuous-time approach in the analysis of momentum gradient descent with step size $\gamma$ and momentum parameter $\beta$ that allows us to identify an intrinsic quantity $\lambda = \frac{ \gamma }{ (1 - \beta)^2 }$ which uniquely defines the optimisation path and provides a simple acceleration rule. When training a $2$-layer diagonal linear network in an overparametrised regression setting, we characterise the recovered solution through an implicit regularisation problem. We then prove that small values of $\lambda$ help to recover sparse solutions. Finally, we give similar but weaker results for stochastic momentum gradient descent. We provide numerical experiments which support our claims.
\end{abstract}

\section{Introduction}

Momentum methods~\citep{sutskever_momentum} have now become a staple of optimal neural network training due to the provided gains in both optimisation efficiency and generalisation performance. This pivotal role is underscored by the widespread use of momentum in the successful training of most state-of-the-art deep networks, including CLIP~\citep{radford2021learning}, Chinchilla~\citep{hoffmann2022training}, GPT-3~\citep{brown2020language}, and PaLM~\citep{chowdhery2022palm}.

Originating in the work of \cite{polyak}, momentum  first featured in the heavy-ball method devised to accelerate convergence in convex optimisation. 
However, when applied to neural network training, momentum exhibits a distinct and complementary characteristic: a steering towards models with superior generalisation performance compared to networks trained with gradient descent.  We note that while the effect of momentum on optimisation has been researched extensively~\citep{defazio_momentum_lyapunov,heavy_ball_saddles}, the generalisation aspect of momentum has been left relatively underexplored.

The performance of gradient descent methods presents intriguing challenges from a theoretical perspective. First, establishing convergence is highly non-trivial. Second, the existence of numerous global minima for the training objective, some of which generalise poorly, adds to the puzzle~\citep{understanding_DL}. To elucidate this second point, the notion of implicit regularisation has come to the forefront. It posits that the optimisation process implicitly favors solutions with strong generalisation properties, even in the absence of explicit regularisation. The canonical example is overparametrised linear regression with more trainable parameters than the number of samples. While there exist infinitely many solutions that fit the data, gradient methods navigate in a restricted parameter subspace and converge towards the solution closest in terms of the $\ell_2$ distance~\citep{lemaire}.

In this work, we aim to expand our understanding of the implicit bias of momentum by analysing its impact on the optimisation trajectory in $2$-layer diagonal linear networks. The $2$-layer diagonal linear network has garnered significant attention recently~\citep{kernel_rich_regimes,vavskevivcius2019implicit,haochen2020understanding,pesme2021implicit,pillaudvivien2022labelnoise}. Despite its apparent simplicity, this network has surprisingly shed light on training behaviours typically associated with much more complex architectures. Some of these insights include the influence of initialisation~\citep{kernel_rich_regimes}, the impact of noise~\citep{pesme2021implicit}, and the role of the step size~\citep{even_pesme_sgd}. Consequently, this architecture serves as an excellent surrogate model for gaining a deeper understanding of intricate phenomena such as the role of momentum in the generalisation performance.

\subsection{Main Contributions}

In this paper, we investigate the influence of momentum on the optimisation trajectory of neural networks trained with momentum gradient descent (MGD). Leveraging the continuous-time approximation of MGD  -- momentum gradient flow (MGF), we show that the optimisation trajectory strongly depends on the key quantity $\l = \frac{\g}{(1-\b)^2}$, where $\g$ and $\b$ denote the step size and momentum parameter of MGD, respectively. Surprisingly, this continuous-time framework experimentally proves to be a good approximation of the discrete trajectory even for large values of $\gamma$.

We proceed to list our main contributions.
\begin{itemize}
    \item First, using the key quantity $\l$, we derive a straightforward acceleration rule that maintains the optimisation path while accelerating the optimisation speed.
    \item Then, focusing on MGF on 2-layer diagonal linear networks, we precisely characterise the recovered solution and prove that for suitably small values of $\l$, MGF recovers solutions which generalise better than the ones selected by gradient flow (GF) in a sparse regression setting.
    \item Finally, we provide similar but slightly weaker results for stochastic MGD.
\end{itemize}

\subsection{Related Works}

\myparagraph{Momentum and Acceleration.}
Momentum algorithms have their roots in acceleration methods, and many studies have investigated their convergence speed when optimising both convex and non-convex functions: ~\citep{ghadimi2015global,flammarion2015averaging,kidambi2018insufficiency,can2019accelerated,sebbouh2021almost,mai2020convergence,liu2020improved,cutkosky2020momentum,defazio_momentum_lyapunov,orvieto2020role,sebbouh2021almost}. Moreover, apart from accelerating training, heavy-ball methods come with the additional advantage of always escaping saddle points~\citep{jin2018accelerated,heavy_ball_saddles}.

\myparagraph{Momentum and Continuous-Time Models.}
Building upon the foundational work of \citet{alvarez_convex_heavy_ball, heavy_ball_friction}, researchers have analysed accelerated gradient methods using second-order differential equations.
\citet{su2014differential} extended the previous ODE to encompass the Nesterov accelerated method, demonstrating convergence rates similar to the discrete case. \cite{wibisono2016variational} adopted a variational perspective to scrutinise the mechanics of acceleration. A significant advancement emerged with the introduction of Lyapunov analysis, undertaken by \cite{wilson2021lyapunov,sanz2021connections,moucer2023systematic}. This analytical approach sheds light on the stability and convergence properties of these methods. Further refinement has been achieved by \cite{high_resolution_odes}, who developed high-resolution ODEs tailored to various momentum-based acceleration techniques and able to distinguish between Nesterov's Accelerated Gradient and  Polyak's Heavy Ball methods.  Finally, error bounds for the discretisation of MGF have been developed by \cite{kovachki_continuous_momentum}.

\myparagraph{Momentum and Implicit Bias.}
\cite{sutskever_momentum, leclerc2020two} have empirically shown significant generalisation improvements in architectures trained with momentum on common vision tasks. Building on these empirical observations, \cite{jelassi_momentum_generalization}  designed a synthetic binary classification problem where a 2-layer convolutional neural network trained with MGD provably generalises better than gradient descent (GD). Recently, \cite{first_order_heavy_ball} reveal that the MGD trajectory closely resembles the gradient flow trajectory of a regularised loss. Through the specific regularisation, the authors argue that the MGD trajectory favors flatter minima than the GD trajectory. The study's findings apply to any reasonable loss, but due to the finite time horizon restriction, cannot characterise the solution to which MGD converges. Additionally,  \cite{deep_diagonal_momentum} show that in deep diagonal linear networks with identical weights across layers, increasing the depth biases the optimisation towards sparse~solutions.

\section{From Discrete to Continuous}\label{sec:prelim}

\myparagraph{Momentum Gradient Descent.} 
We consider minimising a differentiable function $F : \R^d \to \R$ using \textit{momentum gradient descent} (MGD) with step size $\gamma > 0$ and momentum parameter $\beta \in [0, 1)$. Initialised at two points $(w_0, w_1) \in \R^{2d}$, the iterates follow the discrete recursion for $k \geq 1$:
\begin{equation*} \label{mgd}
  \scalebox{0.9}{$   w_{k+1} = w_k - \g \nabla F(w_k) + \beta (w_k - w_{k-1}).$}
    \tag{$\mathrm{MGD}(\g, \b)$}
\end{equation*}

\myparagraph{Momentum Gradient Flow.}
Directly analysing the discrete recursion \ref{mgd} appears intractable in many settings. To overcome this difficulty, we follow the classical approach of considering a second order differential equation of the form
\begin{equation} \label{mgf:ab}
    \begin{aligned}
        a \ddot{w}_t + b \dot{w}_t + \nabla F(w_t) = 0
    \end{aligned}
\end{equation}
with leading coefficient $a \geq 0 $ and damping coefficient $b > 0$. In fact, without loss of generality, the previous differential equation can be reduced to a new one which depends on a single parameter $\lambda$. Indeed, assume that $w_t$ follows ODE~\eqref{mgf:ab} with initialisation $(w_{t = 0}, \dot{w}_{t=0}) = (w_0, \dot{w}_0)$, then a simple chain rule shows that $\tilde{w}_t = w_{b t}$ follows 
\begin{align*}
        \frac{a}{b^2} \ddot{\tilde{w}}_t +  \dot{\tilde{w}}_t + \nabla F(\tilde{w}_t) = 0,
\end{align*}
with initialisation $(\tilde{w}_{t=0}, \dot{\tilde{w}}_{t =0}) = (w_0, b \dot{w}_0)$. Hence, up to a time reparametrisation, it is sufficient to consider the following differential equation which depends on a unique parameter $\lambda \geq 0$:
\begin{align*}\label{mgf:lambda}
    \lambda \ddot{w}_t +  \dot{w}_t + \nabla F(w_t) = 0. \tag{$\mathrm{MGF}(\lambda)$}
\end{align*}
We call the differential equation \ref{mgf:lambda} \textit{momentum gradient flow} (MGF) with parameter $\lambda$. To show the link with the \ref{mgd} recursion, we discretise \ref{mgf:lambda} with a second-order central difference, first-order backward difference, and discretisation step $\varepsilon >0$ as carried out by \cite{kovachki_continuous_momentum}:
\begin{align}\label{mgf:discretisation}
 \scalebox{1}{$ \lambda \ \frac{w_{k+1} - 2 w_k + w_{k-1}}{\varepsilon^2} + \frac{w_k - w_{k-1}}{\varepsilon} + \nabla F(w_k) = 0. $}
\end{align}
Rewriting, we obtain
\begin{align*}
    w_{k+1} =  w_k - \frac{\varepsilon^2}{\l} \nabla F(w_k) + (1 - \frac{ \varepsilon}{\lambda}) (w_k - w_{k-1}),
\end{align*}
which corresponds to momentum gradient descent with parameters $\gamma = \frac{\varepsilon^2}{\lambda}$ and $\beta = 1 - \frac{\varepsilon}{\lambda}$. Solving for $\varepsilon$ and $\lambda$ leads to the following proposition:
\begin{proposition}\label{prop:discretisation} For $(w_0, w_1) \in \R^{2d}$, consider momentum gradient flow \ref{mgf:lambda} with
$$\lambda = \frac{\gamma}{(1- \beta)^2}$$ 
and initialisation $w_{t = 0} = w_0$, $\dot{w}_{t=0} = (w_1 - w_0)/\sqrt{\lambda \gamma}$.
Then, discretising as \eqref{mgf:discretisation} with discretisation step $\varepsilon = \sqrt{\lambda \gamma} = \g/(1-\b)$ leads to the momentum gradient descent recursion \ref{mgd} with step size $\gamma$, momentum parameter $\beta$, and initialisation $(w_0, w_1)$.
\end{proposition}

\Cref{prop:discretisation} motivates studying \ref{mgf:lambda} as a continuous proxy for \ref{mgd} assuming that the discretisation \eqref{mgf:discretisation} closely approximates the continuous~path.

\myparagraph{Discretisation Error Bounds.}
Unfortunately, applying known discretisation error bounds to our setting leads to very pessimistic bounds. Indeed, for step size $\gamma$ and momentum parameter $\beta$, consider the iterates $w_k$ from \ref{mgd} initialised at $(w_0, w_1)$. Now, let $w(t)$ be the solution of \ref{mgf:lambda} with $\lambda = \gamma / (1-\beta)^2$ and the appropriate initialisation from \Cref{prop:discretisation}. 
Then, for a finite horizon $K > 0$, classical discretisation error bounds (see \cite{kovachki_continuous_momentum}, Theorem 4) lead to a catastrophic
\begin{equation*}
    \sup_{k \leq K} \Vert w_k - w(k \varepsilon) \Vert \leq \exp(C K) \varepsilon,
\end{equation*}
where the constant $C$ depends on $\lambda$ and $F$.
Such an exponential dependence in the time horizon $K$ questions the suitability of momentum gradient flow as a good proxy for momentum gradient descent. However, empirically, the above bound appears excessively pessimistic (see \Cref{fig:mgfmgd}: Left and Middle). The MGF and MGD trajectories behave similarly in various settings, even with non-convex losses $F$ and relatively large step sizes $\g$ (see \Cref{experiments} for additional experiments).

\begin{figure}[t!]
    \centering
\includegraphics[width=.48\textwidth]{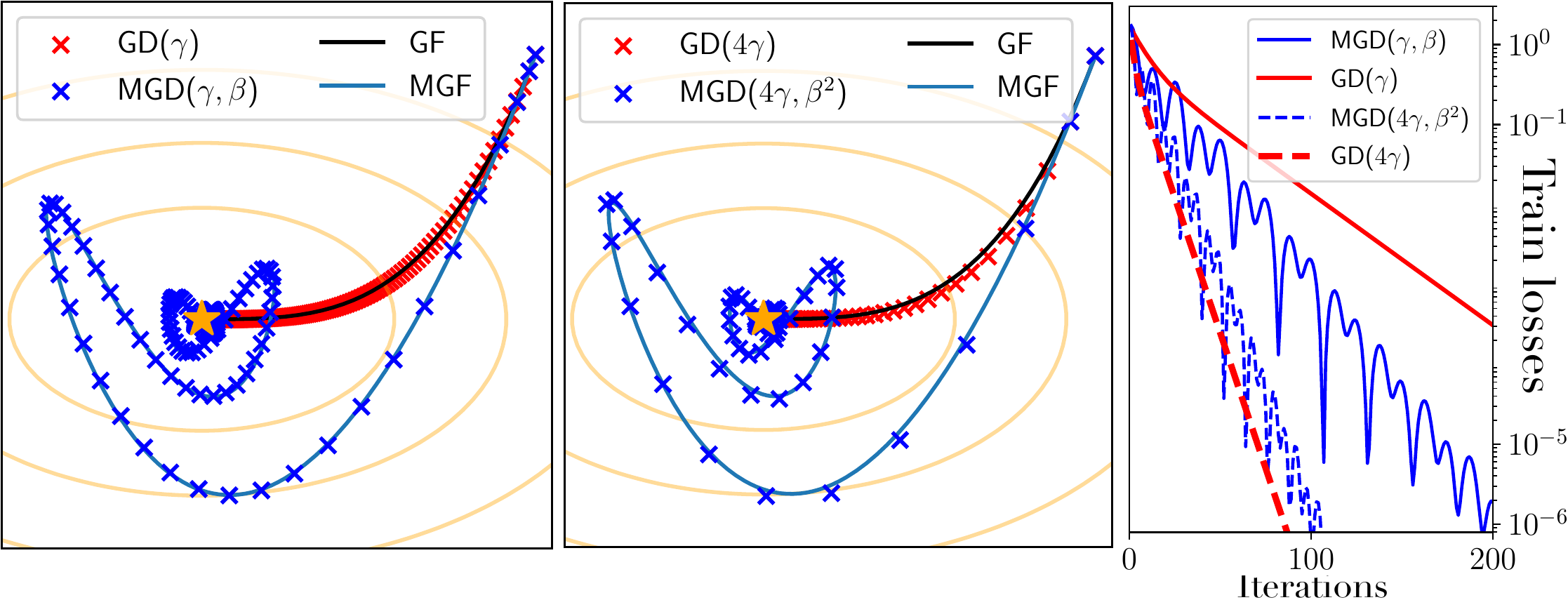}    \vspace{-6mm}
    \caption{(M)GD over a $2$D quadratic. \textit{Left and Middle}:
    The (M)GD trajectories closely follow the continuous trajectories of (M)GF as suggested by \Cref{prop:discretisation}. \textit{Right}: MGD$(4 \gamma, \beta^2)$ follows the same trajectory as MGD$(\gamma, \beta)$ but twice as fast as suggested by \Cref{cor:speedup}. In contrast, GD$(4 \gamma)$ runs four times faster than~GD$(\gamma)$.}
    \label{fig:mgfmgd}
\end{figure}

\myparagraph{Intertwined Roles of $\gamma$ and $\beta$.}
When the discretisation accurately follows the continuous path, \Cref{prop:discretisation} implies that the trajectory of \ref{mgd} is solely determined by a single parameter $\lambda = \gamma / (1 - \beta)^2$, intertwining step size and momentum as observed in \Cref{fig:mgfmgd,fig:teacher}. \textbf{Consequently, $\g$ and $\b$ serve interchangeable roles in influencing the trajectory of \ref{mgd}.} Note that this single-parameter dependence aligns with empirical results from \cite{leclerc2020two} where generalisation performance with large step sizes can be replicated with momentum and smaller step sizes. Though the quantity $\gamma / (1 - \beta)^2$ spontaneously appears in works studying MGD~\citep{first_order_heavy_ball}, to the best of our knowledge, its natural presence was never clearly explained and motivated.

\myparagraph{MGD Acceleration Rule.} Though all couples $(\gamma, \beta)$ with the same same value of $\lambda$ yield the same trajectory, the iterates do not follow this path at the same speed.
\begin{corollary}[Acceleration rule]\label{cor:speedup}
Let $\mathrm{MGD}(\g, \b)$ initialised at $w_0 = w_1 \in \R^{d}$ correspond to the discretisation of \ref{mgf:lambda} with discretisation step $\varepsilon$. Now, for $\rho \in \R_{>0}$, consider the different parameter couple
\[
\hat{\g} = \rho^2 \gamma \quad \text{and} \quad \hat{\beta} =  1 - \rho (1 - \beta) \approx_{\beta \to 1} \beta^\rho.\footnote{The approximation symbol abbreviates the Taylor-expansion bound $1 - \rho(1-\b) = \b^\rho +O((1-\b)^2)$.}
\] 
Then, since $\hat{\g} / (1 - \hat{\b})^2 = \lambda$, $\mathrm{MGD}(\hat{\g}, \hat{\b})$ initialised at $w_0 = w_1$ becomes the discretisation of 
the same \ref{mgf:lambda} but with discretisation step $\hat{\varepsilon} = \rho \cdot \varepsilon$.
\end{corollary}
Following the notations of the previous corollary for an integer $\rho \geq 2$ and letting $w_k$ and $\hat{w}_k$ denote the iterates of MGD$(\gamma, \beta)$ and MGD$(\hat{\g}, \hat{\beta})$, respectively, \Cref{cor:speedup} implies that we expect $w_{\rho \cdot k}$ and $\hat{w}_k$ to be close. This is in contrast with gradient descent, where scaling the step size by a factor $\rho^2$ leads to a speedup of $\rho^2$. This acceleration rule is illustrated in \Cref{fig:mgfmgd} with $\rho = 2$.

\myparagraph{Optimisation Regimes.}
The link between $\lambda$, $\gamma$, and $\beta$ highlights several regimes:
\begin{itemize}[left=0mm, topsep=0pt]
    \item \textbf{$\beta$ large -- the iterates converge arbitrarily slow.} Taking $\beta$ close to $1$ while keeping $\gamma$ constant leads to $\lambda \gg 1$. As explained previously, a chain rule shows that $\Tilde{w}_t = w_{\sqrt{\lambda} t}$ follows the ODE $\ddot{\tilde{w}}_t + \lambda^{-1 / 2} \cdot \dot{\tilde{w}}_t  + \nabla F(\tilde{w}_t) = 0$. Consequently, the damping parameter $\lambda^{-1/2}$ goes to $0$, and we expect the iterates to heavily oscillate and converge arbitrarily slowly.
    \item \textbf{$\gamma$ small -- the iterates follow GF.} Taking $\gamma \to 0$ while keeping $\beta$ fixed leads to $\lambda \ll 1$, and \ref{mgf:lambda} boils down to gradient flow. We expect the \ref{mgd} iterates to be close to the discretisation of GF with discretisation step $\varepsilon = \gamma / (1 - \beta)$. That is, \ref{mgd} will approximate GD with step size $\gamma / (1 - \beta)$. Hence, MGD gains a speed-up of $1 / (1-\beta)$ over GD without a change of trajectory.
    \item \textbf{The ``momentum'' regime.} In this regime, $\gamma$ and $\beta$ are such that $\lambda$ is non-degenerate, and gradient flow cannot capture the trajectory of \ref{mgd}. Hence, $\beta$ has an impact on the optimisation path, and the iterates can still converge in reasonable time.
\end{itemize}
\begin{figure}
    \centering
\includegraphics[width=.4\textwidth]{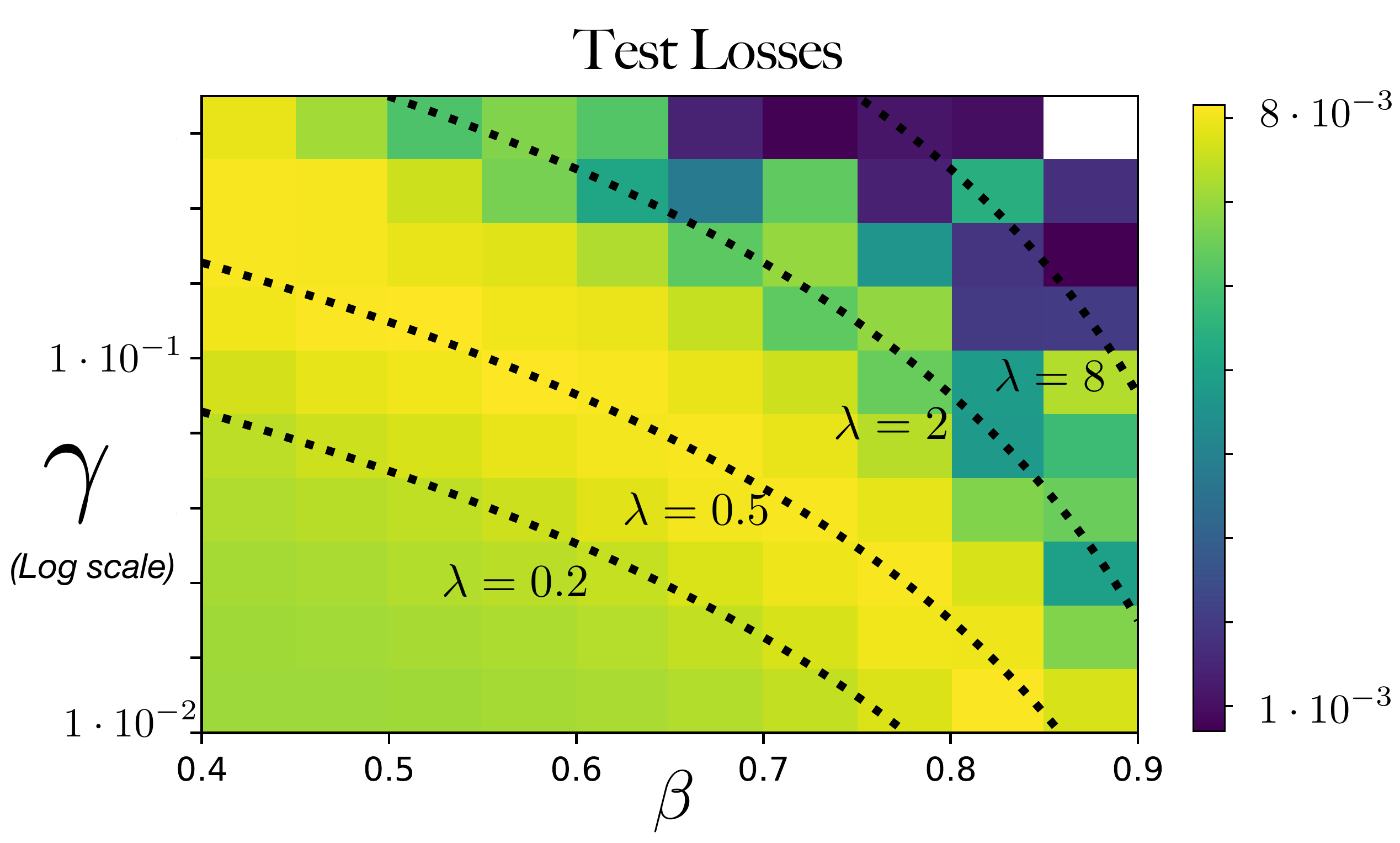}
    \vspace{-3mm}
    \caption{Teacher-student framework with a fully-connected $1$-hidden layer ReLU network. The level lines of the test loss after training with \ref{mgd} correspond to values of $\gamma, \beta$ which have a fixed value $\lambda = \gamma / (1 - \beta)^2$, as predicted by \Cref{prop:discretisation}.
    } 
    \label{fig:teacher}
\end{figure}

\section{Momentum Gradient Flow over Diagonal Linear Networks}\label{mgf_dln}

\myparagraph{Overparametrised Linear Regression.}
We consider a linear regression over $n$ samples $(x_i, y_i)_{i=1}^n$ with inputs $x_i$ living in $\R^d$ and scalar outputs $y_i \in \R$. We assume the dimension $d$ to be larger than the number of samples $n$, in which case there exists an infinite number of vectors $\t^\star$ which perfectly fit the dataset with $y_i = \ip{\t^\star}{x_i}$ for all $1 \leq i \leq n$. We call these vectors \textit{interpolators} and we denote by $\mathcal{S}$ the set of such vectors:  
$\mathcal{S} = \{\t^\star \in \R^d : y_i = \ip{\t^\star}{x_i}, \ \forall i \in [n]\}$. 
Note that $\mathcal{S}$ is an affine space of dimension at least $(d - n)$ equal to $\t^\star + \mathrm{span}(x_1, \dots, x_n)^\perp$ for any interpolator $\t^\star$. We consider the quadratic loss:
\begin{equation}\label{eq:least_squares_loss}
    \L(\t) = \frac{1}{2n} \sum_{i=1}^n (y_i - \langle x_i, \t \rangle)^2.
\end{equation}

\myparagraph{MGF over Least Squares.}
A classical result found in \cite{lemaire} and \cite{optimization_geometry} shows that when initialised at $\theta_0$, gradient flow over the quadratic loss \eqref{eq:least_squares_loss} converges to the orthogonal projection of the initialisation on $\mathcal{S}$:  $\arg \min_{\t^\star \in \mathcal{S}} \Vert \t^\star - \theta_0 \Vert_2$. This next proposition from \cite{alvarez_convex_heavy_ball} shows that momentum does not fundamentally change the implicit bias.
\begin{proposition}[\cite{alvarez_convex_heavy_ball}]
Initialised at $\theta_0$ with initial speed $\dot{\theta}_0$, momentum gradient flow \ref{mgf:lambda} over the least squares loss \eqref{eq:least_squares_loss} converges towards 
\begin{align*}
\argmin_{\theta^\star \in \mathcal{S}} \Vert \theta^\star - (\theta_0 + \lambda \dot{\theta}_0) \Vert_2.    
\end{align*}
\end{proposition}
\ref{mgf:lambda} recovers the same solution as gradient flow but with an effective initialisation $\theta_0 + \lambda \dot{\theta}_0$ which takes into account the drift along $\mathrm{span}(x_1, \cdots, x_n)^\perp$ due to the initial speed $\dot{\theta}_0$. Note that in practice, $\dot{\theta}_0$ is chosen equal to $0$, in which case the presence of momentum has no effect on the recovered solution. 

To better understand momentum's effect on neural networks, we move beyond simple linear parametrization.

\myparagraph{$2$-Layer Diagonal Linear Network.}
We consider a toy neural network, which corresponds to reparametrising the regression vector as $\theta = u \odot v$ for weights $(u, v) \in \R^{2d}$. This parametrisation can be viewed as a simple neural network $x \mapsto \langle u, \sigma(\mathrm{diag}(v) x) \rangle$, where the output weights are $u$, the inner weights are the diagonal matrix $ \mathrm{diag}(v)$, and where the activation function $\sigma$ is the identity. The loss function over the trainable weights $w = (u, v) \in \R^{2d}$ now writes
\begin{equation*}
    F(w) = \L(u \odot v) =  \frac{1}{2n} \sum_{i=1}^n (y_i - \langle x_i, u \p v \rangle)^2,
\end{equation*}
where $\odot$ denotes the Hadamard product.
Despite the simplicity of this reparametrisation, the loss function $F$ is non-convex and challenging to analyse.

\myparagraph{Momentum Gradient Flow.}
We consider momentum gradient flow \ref{mgf:lambda} with parameter $\lambda \geq 0$ over the diagonal-linear-network loss $F$:
\begin{align}\label{MGF:diag}
\begin{split}
 &\lambda \ddot{u}_t + \dot{u}_t + \nabla \L(\theta_t) \odot v_t = 0 \\
 &\lambda \ddot{v}_t + \dot{v}_t + \nabla \L(\theta_t) \odot u_t = 0.
\end{split}
\end{align}
We initialise the flow with zero speed $\dot{u}_0 = \dot{v}_0 = 0$, and apart from requiring the quantity $\vert u_0^2 - v_0^2 \vert$ to have non-zero coordinates\footnote{If initially $u_{i, 0} = \pm v_{i, 0}$ for some coordinate $i \in [d]$, then $u_{i, t} = \pm v_{i, t}, \ \forall t \geq 0$. Hence, imposing $\vert u_0^2 - v_0^2 \vert \neq 0$ becomes equivalent to working with $2d$ distinct weights. See \Cref{integral_formulas} for the full argument from uniqueness.}, we impose no further constraints on the weight initialisations $(u_0, v_0)$. In what follows, we often rely on the reparametrisation $(w_{+, t}, w_{-, t}) \coloneqq (u_t + v_t, u_t - v_t)$ which makes our formulas more succinct. We will also make use of the \textit{initialisation scale} $\a$, which we define as $\a \coloneqq \max(\n{u_0}_\infty, \n{v_0}_\infty)$ and consider as a small quantity.

\myparagraph{Balancedness.} In our results, the \textit{balancedness} of the weights plays a key role. We recall its definition here.
\begin{definition}[Balancedness]
    The \textit{balancedness}\footnote{The absolute value in the definition must be understood coordinate-wise.} of the weights of the diagonal linear network corresponds to the quantity $\D_t \coloneqq \vert u_t^2 - v_t^2 \vert \in \R_{\geq 0}^d$. We define $\D_\infty \coloneqq \lim_{t \to \infty} \D_t$ as the \textit{asymptotic balancedness}.
\end{definition}
Notice that with the above definition we simply adapted the classical notion of balancedness for general linear neural networks~\citep[see][]{NEURIPS2018_fe131d7f, arora_implicit_matrix} to our toy setting. In the case of gradient flow, a simple derivation shows that balancedness is a conserved quantity: i.e., $\D_t = \D_0$ for all $t \geq 0$. However, the evolution of $\D_t$ becomes more complicated as soon as $\l > 0$, and our findings emphasise that the \textit{asymptotic balancedness} $\D_\infty$ plays a crucial role in the generalisation properties of the recovered solution.

\myparagraph{Experimental Details.}
In our numerical experiments, we explore the effects of momentum in the noiseless sparse regression setting with \textbf{uncentered data} as in~\citep{pmlr-v162-nacson22a}. Specifically, we choose $(x_i)_{i=1}^n \stackrel{\scriptscriptstyle \text{i.i.d.}}{\sim} \mathcal{N}(\mu \mathbf{1}, \s^2 I_d)$ and $y_i = \ip{x_i}{\t_s^\star}$ for $i \in [n]$, where $\t_s^\star$ is $s$-sparse with nonzero entries equal to $1/\sqrt{s}$. The use of uncentered data is necessary in order to experimentally observe a clear impact of momentum over the training trajectory (see \Cref{fig:app:final_comparison} for experiments with centered data). We train a 2-layer diagonal linear network with (M)GD and (M)GF with a uniform initialisation $u_0 = \a \cdot \mathbf{1}$, $v_0 = 0$, where $\a = 0.01$. For the plots presented in the main part of our paper, we fixed $(n, d, s) = (20, 30, 5)$, $(\mu, \s) = (1, 1)$. We show results averaged over 5 replications. We refer the reader to \Cref{experiments} for additional experiments where we vary the parameters of the data distribution (e.g., centered data), change the architecture of the trained model, and give further details on the implementation of the (M)GF simulation.

\myparagraph{Notations.} We let $X = (x_1, \dots, x_n)^\top \in \R^{n \times d}$ denote the feature matrix and $y = (y_1, \dots, y_n) \in \R^n$ -- the output vector. For a vector $z \in \R^d$ and a scalar function $f : \R \to \R$, the action of $f$ on $z$ must be understood element-wise: $f(z) \in \R^d$ represents the vector $(f(z_k))_{k=1}^d$. Inequalities between vectors will also be interpreted as holding coordinate-wise. Additionally, when we write $q_{\pm}$ for some place-holder quantity $q$, we mean that we refer to both $q_+$ and $q_-$. For example: $w_{\pm, t} = (u_t \pm v_t)$. Finally, for a strictly convex function $\Phi : \R^d \to \R$, which we call a \textit{potential}, the Bregman divergence is defined as the nonnegative quantity $D_\Phi(\theta_1, \theta_2) = \Phi(\theta_1) - \Phi(\theta_2) - \ip{\nabla \Phi(\theta_2)}{\theta_1 -\theta_2}, \ \forall \theta_1, \theta_2 \in \R^d$.

\subsection{Implicit Bias of Gradient Flow}\label{gf_bias}

Before analysing the effect of momentum, we start by recalling the known results for gradient flow on diagonal linear networks, which corresponds to taking $\lambda = 0$ in \cref{MGF:diag}. \citet{kernel_rich_regimes} show that the predictors $\theta_t = u_t \odot v_t$ converge towards an interpolator $\t^{\texttt{GF}}$ uniquely defined by the following constrained minimisation problem: 
\begin{equation}\label{implicit_bias_GF}
    \t^{\texttt{GF}} = \argmin_{\t^\star \in \mathcal{S}} D_{\psi_{\D_0}}(\t^\star, \t_0),
\end{equation}
where for $\D \in \R_{> 0}^d$, $\psi_{\D} : \R^d \to \R$ denotes the hyperbolic entropy function \citep{pmlr-v117-ghai20a} at scale~$\D$:
\begin{equation} \label{hyperbolic_entropy}
  \scalebox{0.86}{$\psi_\D(\t) = \frac{1}{4} \sum_{i=1}^d \left( 2\t_i \arcsinh \left( \frac{2\t_i}{\D_i} \right) - \sqrt{4\t_i^2 + \D_i^2} + \D_i \right),$}
\end{equation}
and $D_{\psi_{\Delta}}$ is the Bregman divergence. Note that through \cref{implicit_bias_GF}, $\theta^{\texttt{GF}}$ corresponds to the Bregman-projection of the initialisation on the set of interpolators.

\myparagraph{Effect of the Initialisation Scale.}
For a small initialisation scale $\a$, $\t_0 = O(\a^2)$ becomes much smaller than any interpolator $\t^\star \in \mathcal{S}$. Hence, $D_{\psi_{\D_0}}(\t^\star, \t_0)$ roughly equals $D_{\psi_{\D_0}}(\t^\star, 0)$, and \cref{implicit_bias_GF} should be thought of as
\begin{align}\label{eq:IB_GF}
    \t^{\texttt{GF}} \approx \argmin_{\t^\star \in \mathcal{S}} \psi_{\D_0}(\t^\star).
\end{align}
This last equation highlights the fact that the recovered solution simply depends on the initial balancedness $\Delta_0$, making it a key quantity.
Importantly, the hyperbolic entropy is a convex function which interpolates between the $\ell_1$ and $\ell_2$ norms as the magnitude of $\D_0$ goes from 0 to $+\infty$ (see \cite{kernel_rich_regimes}, Theorem 2). So, as $\D_0 = O(\a^2)$ goes to 0, $\psi_{\D_0}$ becomes asymptotically identical to the $\ell_1$-norm (see \Cref{technical_lemmas}). Hence, as seen through \cref{eq:IB_GF}, a small initialisation scale $\a$ leads to the recovery of a solution with a small $\ell_1$-norm, which facilitates sparse recovery and explains why this setting is referred to as the ``rich" or ``feature-learning" regime. On the other hand, larger initialisation scales lead to the so-called ``kernel" or ``lazy" regime, where gradient flow selects small-$\ell_2$-norm solutions. \textbf{Overall, the smaller the initialisation scale, the closer the retrieved solution will be to the minimum-$\ell_1$-norm solution.} We refer the reader to the work of \cite{wind2023implicit} for precise recovery bounds. However, as noted in \cite{even_pesme_sgd}, the picture remains incomplete if we do not take into account the homogeneity of $\D_0$. Indeed, initialisations with entries of different magnitudes can hinder the recovery of a sparse vector. However, in our case, our experiments (for uncentered data) verify that the overall magnitudes of $\D_0$ and $\D_\infty$ are sufficient to explain the effects of momentum. We therefore put aside potential homogeneity considerations.

\subsection{\!Implicit Bias of Momentum Gradient Flow} 

We now move on to describe the impact of momentum on the solution recovered by \ref{mgf:lambda}. Our work proceeds under the following assumption.

\begin{assumption}[Boundedness]\label{boundedness}
    The optimisation trajectory $(u_t, v_t)_{t \geq 0}$ of MGF \eqref{MGF:diag} is bounded.
\end{assumption}

Unfortunately, even though \Cref{boundedness} holds true in all our experiments, the boundedness of the trajectory of a second-order gradient flow has only been established under stronger assumption on the loss function~\citep{alvarez_convex_heavy_ball,quasiconvex_heavy_ball,PL_convergence_heavy_ball}. We defer further details to~\Cref{convergence}. Crucially, the boundedness assumption allows us to prove the convergence of the iterates, and we let $(u_\infty, v_\infty) \coloneqq \lim_{t \to \infty} (u_t, v_t)$. Our goal now becomes to characterise the recovered predictor which we denote with $\t^{\texttt{MGF}} \coloneqq u_\infty \odot v_\infty$. For our proofs, we make the following additional assumption.

\begin{assumption}[Balancedness]\label{non-zero-balancedness}
    The asymptotic balancedness $\D_\infty$ has non-zero coordinates: $\D_{\infty, i} > 0$ for all $i \in [d]$.
\end{assumption}

Again, \Cref{non-zero-balancedness} holds true empirically in all our experiments, and in \Cref{small_lambda}, we prove that small values of $\l$ lead to nonzero asymptotic balancedness. Positing \Cref{non-zero-balancedness} allows us to prove that the recovered solution $\t^{\texttt{MGF}}$ interpolates the dataset.

\subsubsection{General Characterisation of MGF Bias}

In our main result for MGF, we prove that the iterates converge towards an interpolator characterised as the solution of a constrained minimisation problem which involves the hyperbolic entropy \eqref{hyperbolic_entropy} scaled at the asymptotic balancedness $\D_\infty$. Moreover, we derive an insightful description of the asymptotic balancedness in terms of the full optimisation trajectory which allows us to compare the generalisation properties of MGF and GF for small values of $\l$. Before stating our main continuous-time theorem, we define two integral quantities which appear in our results.

\begin{lemma}\label{IntegralQuantities} The following integral quantities $\O_+$ and $\O_-$ are well-defined and finite:
\begin{align*}
    & \O_\pm \coloneqq \int_{0}^\infty \mathrm{m.p.v.} \int_{0}^t \Big( \frac{\dot{w}_{\pm, s}}{w_{\pm, s}} \Big)^2 e^{-\frac{t-s}{\lambda}} \sgn(w_{\pm, t} w_{\pm, s}) \dd s \ \dd t
\end{align*}
where $\sgn(\cdot)$ denotes the sign function, $w_{\pm, t} = u_t \pm v_t$, and $\mathrm{m.p.v.}$ denotes a modified Cauchy principal value defined in \Cref{app:sec:notations}.
\end{lemma}
The fact that the weights $w_{\pm, t}$ can cross zero necessitates the use of the modified Cauchy principal value since otherwise the integrals would diverge. Now, for succinctness, let us introduce the integral quantities
\begin{equation*}
I_\pm \coloneqq \O_\pm + \Lambda_\pm,
\end{equation*}
where the terms $\Lambda_\pm$ vanish whenever the balancedness $\D_t$ remains strictly positive for all $t \in [0, \infty]$. The precise form of $\Lambda_\pm$ is uninformative and can be found in \Cref{BigLambda}, \Cref{integral_formulas}.
We now proceed to characterise the recovered solution $\t^{\texttt{MGF}}$.

\begin{restatable}{theorem}{MainMGFGeneral}\label{main_mgf:general}
The solution $\t^{\texttt{MGF}}$ of MGF \eqref{MGF:diag} interpolates the dataset and satisfies the following implicit regularisation:
\begin{align*}
    \t^{\texttt{MGF}} = \argmin_{\t^\star \in \mathcal{S}} \  D_{\psi_{\D_\infty}}(\t^\star, \tilde{\t}_0).
\end{align*}
In the above expression, $D_{\psi_{\D_\infty}}$ denotes the Bregman divergence with potential $\psi_{\D_\infty}$, where the asymptotic balancedness equals
\begin{align*}
    \D_\infty = \D_0 \odot \exp \big( - (I_{+} + I_{-}) \big)
\end{align*}
and
$\tilde{\t}_0 = \frac{1}{4} ( w_{+,0}^2 \odot \exp\left( - 2 I_+ \right) - w_{-,0}^2 \odot \exp\left( - 2 I_- \right) )$
denotes a perturbed initialisation term.
\end{restatable}
The proof of \Cref{main_mgf:general} appears in \Cref{app:main_mgf:general} as well as explicit formulas for $\D_\infty$ and $\tilde{\t}_0$. We explain the significance and shed more light on the different parts of \Cref{main_mgf:general} below. 

\myparagraph{Perturbed Initialisation $\tilde{\theta}_0$.}
In all our experiments, we observe that the perturbed initialisation $\tilde{\theta}_0$ remains negligible in the sense that for any interpolator $\t^\star \in \mathcal{S}$, $\n{\tilde{\theta}_0}_2 \ll \n{\t^\star}_2$. Moreover, in the next section, we prove that whenever the balancedness remains nonzero during training, $\tilde{\theta}_0$ becomes smaller than $\a^2$, where $\a$ stands for the initialisation scale. Hence, exactly for the same reasons as for gradient flow, the implicit regularisation problem from \Cref{main_mgf:general} should be though of as
\begin{equation}\label{eq:MGF_approx}
    \t^{\texttt{MGF}} \approx  \argmin_{\theta^\star \in S} \psi_{\Delta_\infty}(\theta^\star).
\end{equation}
\Cref{consequences} provides more details. Thus, the asymptotic balancedness $\D_\infty$ becomes the key quantity governing the properties of the recovered solution. 

\myparagraph{Key Role of $\D_\infty$.}
If during optimisation the weights become more balanced, i.e. $ \D_\infty <  \D_0 $, then as discussed previously, 
based on the properties of $\psi_{\Delta_\infty}$, the recovered solution will enjoy better sparsity guarantees than the solution of gradient flow. 
\Cref{fig:mgf} illustrates this point: the smaller the magnitude of $\Delta_\infty$, the better the generalisation. Finally note that by \cref{implicit_bias_GF,eq:MGF_approx}, $\t^{\texttt{MGF}}$ approximately equals the solution recovered from gradient flow initialised at $u_0 = \sqrt{\D_\infty}, v_0 = 0$, which we denote by $\t_{\Delta_\infty}^{\texttt{GF}}$. We observe $||\t^{\texttt{MGF}} - \t_{\Delta_\infty}^{\texttt{GF}}||_2/||\t_{\Delta_\infty}^{\texttt{GF}}||_2 < 0.01$ in all our experiments, which validates the approximation in \cref{eq:MGF_approx}.

\begin{figure}
    \centering
    \includegraphics[width=.45\textwidth]{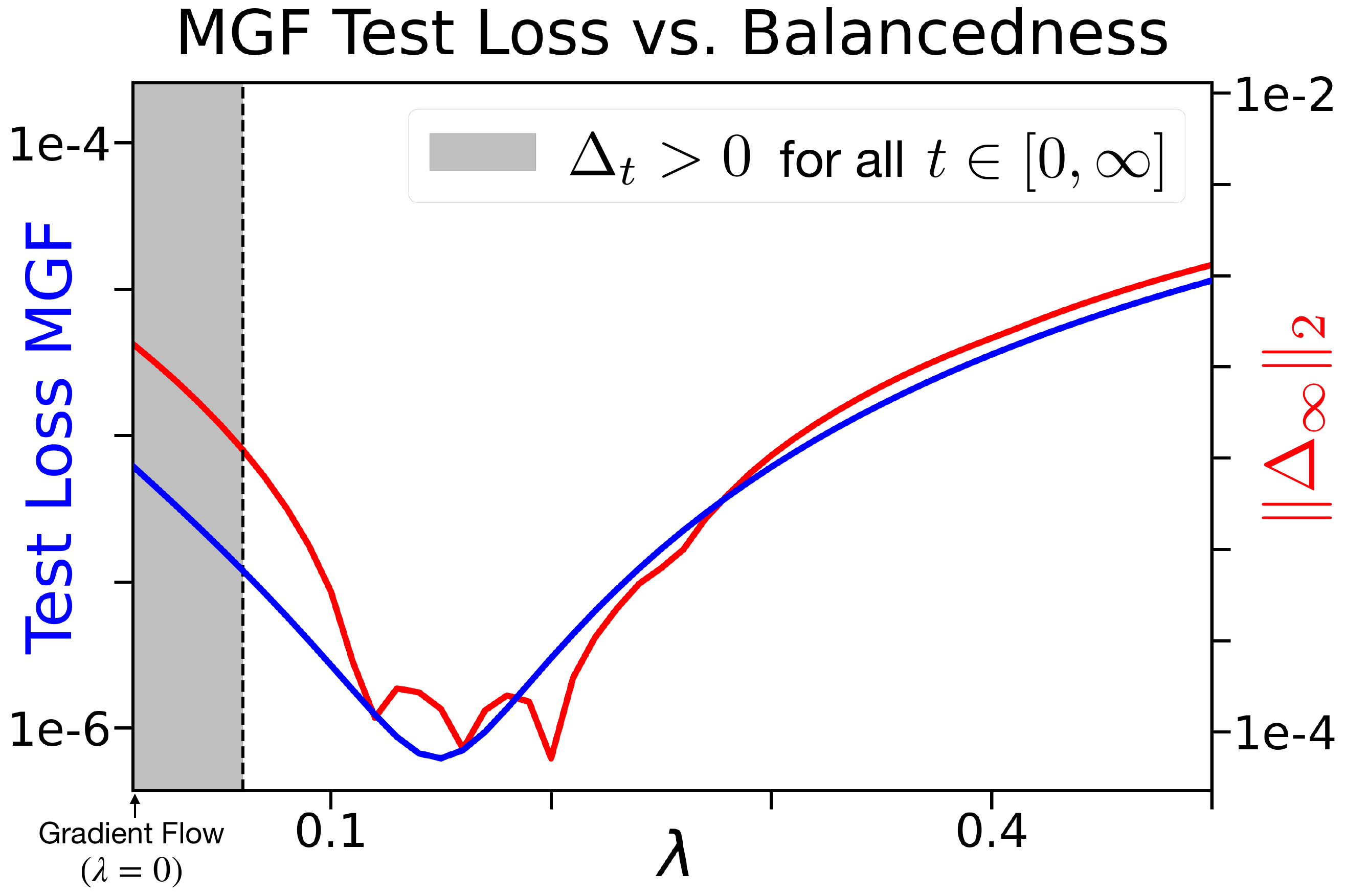}
    \vspace{-3mm}
    \caption{
    Test loss (in blue) and magnitude of balancedness (in red) at convergence of \ref{mgf:lambda} over a diagonal linear network in a sparse regression setting with uncentered data. As predicted by \Cref{main_mgf:general}, a more balanced solution generalises better. The shaded zone corresponds to values of $\lambda$ for which the balancedness never hits zero during training and for which \Cref{main_mgf} therefore holds.
    }
    \label{fig:mgf}
\end{figure}

\myparagraph{Path-Dependent Quantity.} Unfortunately, the asymptotic balancedness depends on the whole optimisation trajectory in a very intricate way, and we cannot compare $\n{\D_\infty}$ and $\n{\D_0}$. Thus, in general, we cannot meaningfully compare the recovered interpolators $\t^{\texttt{MGF}}$ and $\t^{\texttt{GF}}$.
However, in the following section we prove that with the additional assumption that the balancedness remains nonzero, we have $\D_\infty < \D_0$.

\subsection{Provable Benefits of Momentum for Small Values of $\lambda$}\label{small_lambda_regime}

In this subsection, we prove that for small values of the momentum flow parameter $\lambda$, the recovered solution becomes more balanced (and therefore sparser) than the solution of gradient flow. As a starting point for our argument, notice that if the balancedness $\D_t = |u_t^2 - v_t^2| = |w_{+, t} w_{-, t}|$ remains strictly positive throughout training, then the weights $w_{\pm, t}$ never change sign. Hence, the integral quantities $\Lambda_\pm$ become 0, and $\O_\pm > 0$. Thus, $I_\pm > 0$, which combined with \Cref{main_mgf:general} implies the following corollary.

\begin{restatable}{corollary}{MainMGF}\label{main_mgf}
    For $\lambda > 0$, if the balancedness $\D_t$ remains strictly positive during training (i.e. $\Delta_t \neq 0$ for $t \in [0, +\infty]$), then the perturbed initialisation satisfies $|\tilde{\t}_0| < \alpha^2$ and
\begin{align*}
   \D_\infty &= \D_0 \odot \exp \Big( - \lambda \int_0^\infty \Big( \frac{\dot{w}_{+, t}}{w_{+, t}} \Big)^2  +  \Big( \frac{\dot{w}_{-, t}}{w_{-, t}} \Big)^2 \dd t \Big).
\end{align*}
Importantly, $\D_\infty < \D_0$.
\end{restatable}

In words, the above corollary (proved in \Cref{app:main_mgf}) implies that if the balancedness $\Delta_t$ does not hit zero during training, then (i) the perturbation term $\tilde{\t}_0$ is provably negligible, (ii) the asymptotic balancedness is coordinate-wise smaller than initial balancedness $\Delta_0$ which translates into a solution with better sparsity properties than the gradient flow interpolator. This regime corresponds to the gray zone in \Cref{fig:mgf}. The following proposition proved in \Cref{technical_lemmas} demonstrates that for small values of $\l$, the balancedness remains strictly positive.

\begin{restatable}{proposition}{SmallLambda}\label{small_lambda}
    For $\lambda \leq \frac{n}{\Vert y \Vert_2^2} \cdot  (\min_{i \leq d} \D_{0, i})$, the balancedness $ \D_t$ never vanishes: $\D_t \neq 0, \ \forall t \in [0, +\infty]$. 
\end{restatable}
Hence, through \Cref{small_lambda} and \Cref{main_mgf}, we show that small values of $\lambda$ lead to solutions with better sparse recovery guarantees.

\myparagraph{Limitations of Our Analysis.} In \Cref{integral_formulas}, we prove that $\D_t$ can vanish at most a finite number of times. Experimentally, $\D_t$ never hits $0$ for much larger values of $\l$ than $\frac{n}{\Vert y \Vert_2^2} \cdot  (\min_{i \leq d} \D_{0, i})$, making the bound from \Cref{small_lambda} relatively loose. In \Cref{fig:mgf}, we empirically observe an interval $(0, \lambda_{max})$  in which MGF($\lambda$) outperforms GF in terms of generalisation. Moreover, there exists an optimal value $\lambda^\star$ (roughly corresponding to the smallest $\D_\infty$) which brings about the most improvement compared to gradient flow. Unfortunately, as observed \Cref{fig:mgf}, the balancedness vanishes for $\lambda = \lambda^\star$, and therefore \Cref{main_mgf} does not cover the optimal value. Also note that $(0, \lambda_{max})$ and $\lambda^\star$ depend on the data.

\myparagraph{Behaviour of $\Delta_\infty$ for Small Values of $\lambda$.}
Unfortunately, determining the precise effect of $\lambda$ on $\Delta_\infty$ is challenging. Nonetheless, for small $\lambda$, we informally show in \Cref{behavior_of_delta} that 
\begin{align*}
    \D_\infty^2 \underset{\lambda \to 0}{\approx} \D_0^2 \odot \exp \big( - 2 \lambda \int_0^\infty \nabla \L(\theta_s)^2 \dd s \big).
\end{align*}
This approximate equivalence for small $\lambda$ echoes the implicit bias of SGD \citep{even_pesme_sgd,pesme2021implicit}, which involves a similar formulation for the effective initialisation where the step size $\gamma$ appears instead of $\lambda$. Note that the above approximation suggests that for small values of $\lambda$, $\D_\infty$ monotonically decreases with $\lambda$ as experimentally confirmed by \Cref{fig:mgf}. 

\subsection{Sketch of Proof} \label{mgf_sketch}

\myparagraph{Implicit Bias through a Second-Order Time-Varying Mirror Flow.}
A natural way of showing the implicit regularisation \eqref{implicit_bias_GF} of gradient flow on a 2-layer diagonal linear network goes through proving that the predictors $\theta^{\texttt{GF}}_t$ follow the mirror flow $\dd \nabla \psi_{\Delta_0}(\theta^{\texttt{GF}}_t) = - \nabla \L(\theta^{\texttt{GF}}_t) \dd t$. 
In our setting, we prove that the predictors $\theta^{\texttt{MGF}}_t$ follow a second-order time-varying mirror flow. Specifically, we define a family of potentials $(\Phi_t)_{t \geq 0}$ with
$\Phi_t(\theta) \coloneqq \psi_{ \D_t }(\theta) - \langle \phi_t, \theta \rangle$ where $\psi_{ \D_t }$ corresponds to the hyperbolic entropy \eqref{hyperbolic_entropy} depending on the balancedness $\D_t$ and a perturbation function $\phi_t$. We then prove the following proposition.
\begin{proposition}\label{prop:2nd_order_MF}
The predictors $\theta_t^{\texttt{MGF}}$ follow a momentum mirror flow with time-varying potentials $\Phi_t$:
    \begin{align*}
        \lambda \frac{\dd^2  \nabla \Phi_t (\theta_t^{\texttt{MGF}}) }{\dd t^2} + \frac{\dd  \nabla \Phi_t (\theta_t^{\texttt{MGF}}) }{\dd t} + \nabla \L(\theta_t^{\texttt{MGF}}) = 0.
    \end{align*}
\end{proposition}
The implicit regularisation follows from integrating the ODE: $\nabla \Phi_\infty (\theta^{\texttt{MGF}}) = - \int_0^\infty \nabla \L(\theta_t^{\texttt{MGF}}) \dd t \in \texttt{span}(x_1, \dots, x_n)$ which exactly corresponds to the KKT conditions of the constrained minimisation from \Cref{main_mgf:general}. Assuming that $w_{\pm, t}$ do not change sign, the proof of \Cref{prop:2nd_order_MF} comes naturally and relies on the writing $w_{\pm, t}= \sgn (w_{\pm, 0}) \exp( \rho_{\pm, t})$. When the iterates cross $0$, this reparametrisation does not hold anymore. The analysis can still be carried out by decomposing $\R_{\geq 0}$ into intervals on which the iterates have constant sign and appropriately sticking the intervals using a modified Cauchy principal value.

\begin{figure*}[h!]
    \centering
    \includegraphics[width=1\textwidth]{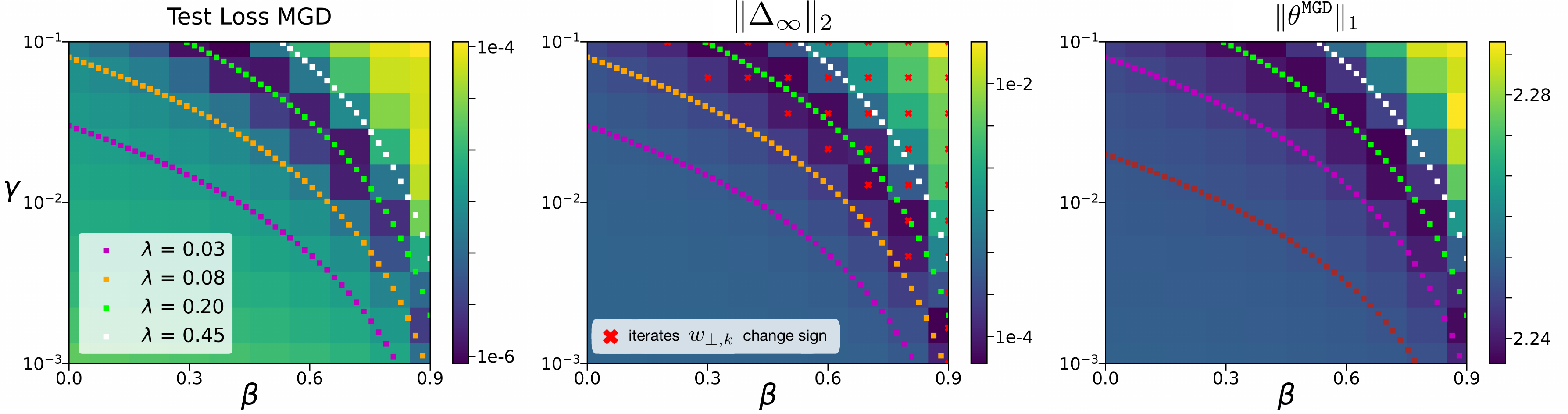}
    \vspace{-3mm}
    \caption{(Non-stochastic) MGD over a diagonal linear network in a sparse regression setting with uncentered data.  As predicted by \Cref{prop:discretisation}, the three quantities at convergence only depend on the single parameter $\lambda \coloneqq \gamma / (1 - \beta)^2$.
    As predicted by \Cref{main_mgd:general}, a more balanced solution (\textit{center plot}) leads to a solution with a smaller $\ell_1$-norm (\textit{right plot}), which in turn translates into better generalisation (\textit{left plot}). Finally, as predicted by \Cref{main_mgd}, the trajectories for which the iterates do not cross zero satisfy $\Delta_\infty < \Delta_0$, where $\Delta_0$ (approximately) corresponds to the asymptotic balancedness for $\beta = 0$ and $\gamma = 10^{-3}$.}
    \label{fig:mgd}
\end{figure*}

\section{Momentum SGD over Diagonal Linear Networks}\label{sec:MGD}

In this section, we move from continuous to discrete time and focus on the original \ref{mgd} recursion for which we can prove similar but slightly weaker results than the ones for MGF. In fact, our results hold for stochastic momentum gradient descent (SMGD) with any batch size $B \in [n]$. For step size $\g >0$ and momentum parameter $\b \in [0,1)$, the SMGD recursion writes as follows:
\begin{align}\label{SMGD:diag}
\begin{split}
 & u_{k+1} = u_k - \g\nabla \L_{\cB_k}(\t_k) \p v_k + \b(u_k - u_{k-1}) \\
 & v_{k+1} = v_k - \g\nabla \L_{\cB_k}(\t_k) \p u_k + \b(v_k - v_{k-1}),
\end{split}
\end{align}
where
$L_{\cB_k}(\t) \coloneqq  \frac{1}{2B}\sum_{i\in\cB_k}(y_i - \langle u \odot v ,x_i\rangle)^2$ corresponds to the partial loss over the batch  $\cB_k\subset[n]$  of size $B$. The batches could be sampled with or without replacement. As for continuous time,
we let $\t_k = u_k \p v_k$ correspond to the regression predictor. We initialise at $u_1 = u_0$ and $v_1 = v_0$, and we again consider the balancedness of the weights $\D_k \coloneqq \vert u_k^2 - v_k^2 \vert$ for $k \geq 0$, the reparametrised iterates $w_{\pm, k} \coloneqq u_k \pm v_k$, and the \textit{initialisation scale} $\a \coloneqq \max(\n{u_0}_\infty, \n{v_0}_\infty)$.  In contrast to our continuous-time prerequisites where we only assumed boundedness of the optimisation trajectory, here we assume that the iterates converge:
\begin{assumption}[Convergence]\label{ass:discrete:convergence}
    The iterates $(u_k, v_k)$ converge towards the limiting weights $(u_\infty, v_\infty)$. We denote by $\theta^{\texttt{SMGD}} \coloneqq u_\infty \odot v_\infty$ the recovered predictor.
\end{assumption}
As in continuous time, we again assume that the asymptotic balancedness is nonzero.
\begin{assumption}[Balancedness]\label{ass:discrete:balancedness}
 The asymptotic balancedness $\Delta_\infty \coloneqq |u_\infty^2 - v_\infty^2|$ has non-zero coordinates.
\end{assumption}

Similar to \Cref{IntegralQuantities}, we define two discrete infinite sums which depend on the entire trajectory and appear in our discrete-time result.
\begin{restatable}{lemma}{lemmaconvergenceS}
    \label{lemma:discrete:I}
The following two sums $S_+$ and $S_-$ converge to finite vectors:
     \begin{equation*}
S_{\pm} = \frac{1}{1-\beta} \sum_{k=1}^{\infty} \left[ r \Big( \frac{w_{\pm, k+1}}{w_{\pm, k} } \Big) + \b  r \Big( \frac{w_{\pm, k}}{w_{\pm, k+1}} \Big) \right],
\end{equation*}
where $r(z) = (z - 1) - \ln(|z|)$ for $z \neq 0$.
\end{restatable}

Importantly, the function $r(z)$ from \Cref{lemma:discrete:I} is positive for $z > 0$.
Contrary to the continuous-time case, in discrete time, the iterates $w_{\pm, k}$ never exactly equal zero. Indeed, since $\nabla L$ is linear, we have that for all $k \geq 0$, $w_{\pm, k}(\g, \b)$ is a polynomial in $(\g, \b)$. Therefore, the set of pairs $(\g, \b)$ for which there exists $k \geq 0$ such that $w_{\pm, k}(\g, \b) = 0$ is a negligible set in $\R^2$. The iterates therefore `jump' over zero, making the sums from \Cref{lemma:discrete:I} well-defined.

\subsection{General Characterisation of SMGD Bias}

The following theorem represents the discrete counterpart of \Cref{main_mgf:general} and generalises \cite[Theorem 1]{even_pesme_sgd} which considers SGD without momentum.

\begin{restatable}{theorem}{thmdiscrete}\label{main_mgd:general}
The solution $\t^{\texttt{SMGD}}$ of SMGD \eqref{SMGD:diag} interpolates the dataset and satisfies the following implicit regularisation:
\begin{align*}
    \t^{\texttt{SMGD}} = \argmin_{\t^\star \in \mathcal{S}} \  D_{\psi_{\D_\infty}}(\t^\star, \tilde{\t}_0).
\end{align*}
In the above expression, $D_{\psi_{\D_\infty}}$ denotes the Bregman divergence with potential $\psi_{\D_\infty}$, where the asymptotic balancedness equals
 \begin{equation*}
\D_\infty  = \D_0 \odot \exp \big ( -(S_{+} + S_{-}) \big ) 
\end{equation*}
and 
$\tilde{\t}_0 = \frac{1}{4}( w_{+,0}^2 \odot \exp ( -2 S_+)  ) - w_{-,0}^2 \odot \exp ( - 2 S_-   ) ) $
denotes a perturbed initialisation term.
\end{restatable}

Due to the strong similarities with \Cref{main_mgf:general}, we proceed by making similar comments. In our experiments, the norm of the perturbed initialisation $\tilde{\theta}_0$ remains much smaller than that of any interpolator $\theta^\star$. Hence, arguing as before, the implicit regularisation problem from \Cref{main_mgd:general} should be though of as
\begin{equation}\label{eq:SMGD_approx}
    \t^{\texttt{SMGD}} \approx  \argmin_{\theta^\star \in S} \psi_{\Delta_\infty}(\theta^\star).
\end{equation}
Again, the asymptotic balancedness $\D_\infty$ controls the generalisation properties of the recovered solution. Thus, if $\n{\D_\infty(\g, \b)}_2 < \n{\D_\infty(\g', \b')}_2$, we expect the interpolator $\t^{\texttt{SMGD}}(\g, \b)$ to be sparser than $\t^{\texttt{SMGD}}(\g', \b')$. \Cref{fig:mgd} illustrates this point: the smaller the magnitude of $\Delta_\infty$ (center plot), the better the sparsity of the interpolator (right plot), which translates into better generalisation (left plot). Unfortunately, as for MGF, the asymptotic balancedness $\D_\infty$ depends on the whole optimisation trajectory in an intricate way, which prevents us from extracting an insightful formula for $\D_\infty$ in terms of $\g$ and $\b$. However, \Cref{fig:mgd} indicates that $\D_\infty$ effectively depends on the single parameter $\lambda = \gamma / (1 - \beta)^2$. As in \Cref{fig:teacher}, $\lambda$ again clearly appears to be the relevant quantity which governs the performance of MGD, and not $\gamma$ and $\beta$ considered individually. These empirical observations support the idea that even for `practical' step sizes $\g$ and momentum parameters $\b$, \ref{mgd} closely follows \ref{mgf:lambda}.

\Cref{fig:mgd} also clearly shows that the asymptotic balancedness decreases as the key quantity $\l$ increases over an interval $[0, \lambda^\star]$ where $\lambda^\star$ denotes the parameter inducing the best generalisation performances. Then, for $\l$ above $\l^\star$, the magnitude of $\D_\infty$ starts to grow and the sparsity of the solutions deteriorates. We expect proving this phenomenon to be very challenging. Such a proof would require a fine-grained analysis of the sums $S_\pm$, which becomes already quite involved when $\beta = 0$ as performed by \citet{even_pesme_sgd}.

Now, similar to the continuous-time result, the following corollary shows that if the iterates do not change sign, then the asymptotic balancedness becomes smaller than the initial balancedness.

\begin{restatable}{corollary}{MainMGD}\label{main_mgd}
For $\gamma, \beta > 0$, if the iterates $w_{\pm, k} = (u_k \pm v_k)$ do not change sign during training, then $|\tilde{\theta}_0| < \alpha^2$ and $\D_\infty < \D_0$.
\end{restatable}

The above corollary implies that the recovered solution $\t^{\texttt{SMGD}}$ must perform at least as well as the gradient flow interpolator $\t^{\texttt{GF}}$. However, in contrast to the continuous case and even though we believe it to be true, we were unable to prove that the SMGD iterates do not change sign for small values of $\lambda$.

\section{Conclusion} Considering an appropriate second-order differential equation which discretises into MGD, we highlight the existence of a single key quantity $\lambda = \gamma / (1 - \beta)^2$ which fully determines the trajectory of MGF. This continuous-time perspective also provides a simple acceleration rule and insight into several relevant optimisation regimes. Then, focusing on $2$-layer diagonal linear networks, we prove that the asymptotic balancedness $\Delta_\infty$ solely governs the generalisation performances of MGF and SMGD. We additionally prove that small values of $\lambda$ aid the recovery of sparse MGF solutions. Future work should consider MGF/MGD optimisation on more complex architectures and understand precisely the non-trivial effect of $\lambda$ on the asymptotic balancedness $\Delta_\infty$.
\newpage

\myparagraph{Acknowledgements.}
The authors would like to thank Jana Vuckovic for investigating various discretisation methods for MGF, Yoana Nakeva and Mathieu Even for providing feedback on the original version of the paper, and Aditya Varre for the helpful discussions which led to the proofs of \Cref{main_mgf:general} and \Cref{main_mgd:general}. This project was supported by the Swiss National Science Foundation (grant number 212111).

\bibliography{references}

\newpage
\appendix
\onecolumn

\section*{Organisation of the Appendix.}

The appendix is organised as follows:
\begin{itemize}
\item In \Cref{app:sec:notations}, we introduce additional notation and provide further comments on the discretisation methods.
\item In \Cref{reparametrisation}, we present a useful reparametrisation of our problem.
\item In \Cref{proof_mgf}, we offer proofs for our continuous-time results, specifically \Cref{main_mgf:general}, \Cref{main_mgf}, and \Cref{prop:2nd_order_MF}.
\item In \Cref{proof_mgd}, we detail the proofs for our discrete-time results, namely \Cref{lemma:discrete:I}, \Cref{main_mgd:general} and \Cref{main_mgd}.
\item In \Cref{technical_lemmas}, we introduce technical lemmas necessary for our main results and prove \Cref{small_lambda}.
\item In \Cref{experiments}, we provide more details on the main-paper experiments and showcase further experimental results.
\end{itemize}

\section{Additional Notations and Comments on Discretisation Methods}\label{app:sec:notations}

\myparagraph{Vector Operations.}
Moving forward, all arithmetic operations and real-valued functions will be considered as being applied coordinate-wise. In other words, if $a$ and $b$ are vectors in $\R^d$ and $p, q \in \Q$, then $a^p b^q \in \R^d$ will be used as a shorthand for the vector with entries $\{a_i^p b_i^q\}_{i=1}^d$. And for any $f : \R \to \R $, $f(a)$ will represent the vector with entries $\{f(a_i)\}_{i=1}^d$. Inequalities between vectors will also be interpreted as holding coordinate-wise.

\myparagraph{Mirror Maps.} Various definitions of a \textit{mirror map} $\Phi : \R^d \to (-\infty, +\infty]$ exist in the optimization literature~\citep[see][]{nemirovsky_yudin, mirror_flow}, and a common one coincides with the concept of a Legendre function~\citep[see][]{Bauschke2017ADL, Bauschke1997}. In our proofs, we do not deal with extended real-valued functions, and the term mirror map is applied to $C^\infty$-smooth strictly convex functions with coercive gradients. In particular, our mirror maps are of Legendre type.

For such a mirror map $\Phi : \R^d \to \R$, we define the Bregman divergence $D_\Phi(\theta_1, \theta_2)$ for $\theta_1, \theta_2 \in \R^d$ as
\begin{equation*}
    D_\Phi(\theta_1, \theta_2) = \Phi(\theta_1) - \Phi(\theta_2) - \ip{\nabla \Phi(\theta_2)}{\theta_1 -\theta_2}.
\end{equation*}
Notice that due to the strict convexity of $\Phi$, $D_\Phi(\theta_1, \theta_2) > 0$ whenever $\theta_1 \neq \theta_2$.

\myparagraph{Modified Cauchy Principal Value.}
Let $f: \R_{\geq 0} \to [-\infty, +\infty]$ be an extended real-valued function with a finite set of poles $\mathcal{T} = \{T_1, T_2, \dots, T_N\}$ (\textit{i.e.} points $t \in \R_{\geq 0}$ at which $f(t) = \pm \infty$) such that $f$ is continuous on $\R_{\geq 0} \setminus \mathcal{T}$. Let $0 < T_1 < \dots < T_N$. Let $T \in \mathcal{T}$ and let $\eps > 0$ be small enough such that $(T-\eps, T + \eps) \cap \mathcal{T} = \{T\}$. Recall that, provided the limit below exists, the Cauchy principal value $\mathrm{p.v.} \int_{T-\eps}^{T+\eps} f(t) \dd t$ is defined as
\begin{equation*}
   \mathrm{p.v.} \int_{T-\eps}^{T+\eps} f(t) \dd t \coloneqq \lim_{\d \to 0} \left[ \int_{T-\eps}^{T-\d} f(t) \dd t + \int_{T+\d}^{T+\eps} f(t) \dd t \right].
\end{equation*}

Now, let $\eps_m > 0$ be such that $(T_m-\eps_m, T_m + \eps_m) \cap \mathcal{T} = \{T_m\}$ for $m \in [N]$. Moreover, let $T_0 = \eps_0 = 0$ and $T_{N+1} = +\infty$. Suppose $f$ has finite Cauchy principal values at all poles. Then, for any $\tau \geq 0$ such that $\tau \notin \mathcal{T}$, we could define $\mathrm{p.v.} \int_0^\tau f(t) \dd t$ as
\begin{equation*}
    \mathrm{p.v.} \int_0^\tau f(t) \dd t \coloneqq \sum_{m: T_{m+1} < \tau} \left[ \mathrm{p.v.} \int_{T_m - \eps_m}^{T_m + \eps_m} f(t) \dd t + \int_{T_m + \eps_m}^{T_{m+1} - \eps_{m+1}} f(t) \dd t \right] + \mathrm{p.v.} \int_{T_k - \eps_k}^{T_k + \eps_k} f(t) \dd t + \int_{T_k + \eps_k}^\tau f(t) \dd t,
\end{equation*}
where $T_k < \tau < T_{k+1}$.

For our proofs of \Cref{IntegralQuantities} and \Cref{main_mgf:general}, we require a modification to the Cauchy principal value. For the aforementioned function $f$ with the described properties and for $T \in \mathcal{T}, \ \eps > 0$ such that $(T-\eps, T + \eps) \cap \mathcal{T} = \{T\}$, we define the modified principal value $\mathrm{m.p.v.} \int_{T-\eps}^{T+\eps} f(t) \dd t$ as
\begin{equation}\label{mpv}
   \mathrm{m.p.v.} \int_{T-\eps}^{T+\eps} f(t) \dd t \coloneqq \lim_{\d \to 0} \left[ \int_{T-\eps}^{T-\d} f(t) \dd t \cdot e^{\frac{\d}{\l}}+ \int_{T+\d}^{T+\eps} f(t) \dd t \cdot e^{-\frac{\d}{\l}} \right],
\end{equation}
where $\l$ denotes our familiar MGF parameter. We also extend the $\mathrm{m.p.v.}$ definition to integrals $\int_0^\tau f(t) \dd t$ for arbitrary $\tau \geq 0$ by mimicking the Cauchy-principal-value construction:
\begin{equation*}
    \mathrm{m.p.v.} \int_0^\tau f(t) \dd t \coloneqq \sum_{m: T_{m+1} < \tau} \left[ \mathrm{m.p.v.} \int_{T_m - \eps_m}^{T_m + \eps_m} f(t) \dd t + \int_{T_m + \eps_m}^{T_{m+1} - \eps_{m+1}} f(t) \dd t \right] + \mathrm{m.p.v.} \int_{T_k - \eps_k}^{T_k + \eps_k} f(t) \dd t + \int_{T_k + \eps_k}^\tau f(t) \dd t,
\end{equation*}
where $T_k < \tau < T_{k+1}$. Note that the above definition implies that whenever $f$ has no poles on an interval $(a,b) \subset \R_{\geq 0}$, then
\begin{equation*}
    \mathrm{m.p.v.} \int_a^b f(t) \dd t = \int_a^b f(t) \dd t.
\end{equation*}

\myparagraph{Additional Comments on the Discretisation of \ref{mgf:lambda}.}
Following our discussion from \Cref{sec:prelim}, we want to point out that that there are other ways of discretising
\begin{align*}
    \lambda \ddot{w}_t +  \dot{w}_t + \nabla F(w_t) = 0.
\end{align*}
Indeed, instead of discretising as \eqref{mgf:discretisation} in the main paper
\begin{align*}
  \lambda \ \frac{w_{k+1} - 2 w_k + w_{k-1}}{\varepsilon^2} + \frac{w_k - w_{k-1}}{\varepsilon} + \nabla F(w_k) = 0,
\end{align*}
one could also consider a central first-order difference:
\begin{align*}
\lambda \ \frac{w_{k+1} - 2 w_k + w_{k-1}}{\varepsilon^2} + \frac{w_{k+1} - w_{k-1}}{2 \varepsilon} + \nabla F(w_k) = 0.
\end{align*}
Rearranging, this leads to
\begin{align*}
    w_{k+1} =  w_k - \frac{\varepsilon^2}{\l ( 1 + \frac{\varepsilon}{2 \lambda})} \nabla F(w_k) + \frac{1 - \frac{\varepsilon}{2 \lambda}}{1 + \frac{\varepsilon}{2 \lambda}} (w_k - w_{k-1}),
\end{align*}
which corresponds to momentum with $\gamma = \frac{\varepsilon^2}{\l ( 1 + \frac{\varepsilon}{2 \lambda})}$ and $\beta = \frac{1 - \frac{\varepsilon}{2 \lambda}}{1 + \frac{\varepsilon}{2 \lambda}}$. Solving for $\varepsilon$ and $\lambda$, we get
\begin{align*}
    \lambda = \frac{(1 + \beta) \gamma}{2(1 - \beta)^2} \qquad \text{and} \qquad \varepsilon = \frac{\gamma}{1 - \beta}.
\end{align*}
Hence, we obtain the same discretisation step $\varepsilon$ as in \Cref{prop:discretisation} and a slightly different expression for $\lambda$. However, note that the two versions of $\lambda$ become indistinguishable for large values of $\beta$ since $\frac{1 + \beta}{2} \to_{\beta \to 1} 1$. Experimentally, running \ref{mgf:lambda} with the two different values for $\lambda$ leads to similar results. Thus, the discretisation scheme from the main paper was chosen due to the more concise definition of $\lambda$ in this case.

\section{$(w_+, w_-)$-Reparametrisation}\label{reparametrisation}

\myparagraph{MGF Reparametrisation.}
We recall that we consider momentum gradient flow \ref{mgf:lambda} with parameter $\lambda > 0$ over the diagonal-linear-network loss $F((u,v))) = \L(u \p v)$:
\begin{align*}
\begin{split}
 &\lambda \ddot{u}_t + \dot{u}_t + \nabla L(\theta_t) \odot v_t = 0; \\
 &\lambda \ddot{v}_t + \dot{v}_t + \nabla L(\theta_t) \odot u_t = 0.
\end{split}
\end{align*}
For proof-writing convenience, we consider the simple reparametrisation outlined below.

In order to eliminate the cross-dependencies in $(u, v)$ in the above equations, it is natural to consider the quantities $(w_{+,t}, w_{-,t})$ where $w_{\pm, t} = u_t \pm v_t$ for $t \ge 0$. Hence, we get the following reparametrised ODE:
\begin{equation}\label{w_mgf}
    \left \{
    \begin{aligned}
        & \l \ddot{w}_{\pm, t} + \dot{w}_{\pm, t} \pm \nabla \L(\t_t) \p w_{\pm, t} = 0; \\
        & w_{\pm, 0} = u_0 \pm v_0, \en \dot{w}_{\pm, 0} = 0.
    \end{aligned}
    \right.
\end{equation}

Notice that with these new quantities, we have
\begin{align*}
    \t_t = \frac{w_{+, t}^2 - w_{-, t}^2}{4} \quad \text{and} \quad \D_t = |w_{+, t} w_{-, t}|.
\end{align*}

\myparagraph{MGD Reparametrisation.}
For the discrete-time setting, we follow the same reparametrisation from the MGD recursion:
\begin{align*}
\begin{split}
 & u_{k+1} = u_k - \g\nabla L(\t_k) \p v_k + \b(u_k - u_{k-1}); \\
 & v_{k+1} = v_k - \g\nabla L(\t_k) \p u_k + \b(v_k - v_{k-1}).
\end{split}
\end{align*}
We let $w_{\pm, k} = u_k \pm v_k$ for $k \geq 0$. Then, for $k \geq 1$, the equations above transform into
\begin{equation}\label{w_mgd}
    \left \{
    \begin{aligned}
        & w_{\pm, k+1} = w_{\pm, k} \mp \g \nabla \L(\t_k) \p w_{\pm, k} + \b(w_{\pm, k} - w_{\pm, k-1}); \\
        & w_{\pm, 1} = w_{\pm, 0} = u_0 \pm v_0.
    \end{aligned}
    \right.
\end{equation}

Again, with the newly defined quantities, we have
\begin{align*}
    \t_k = \frac{w_{+, k}^2 - w_{-, k}^2}{4} \quad \text{and} \quad \D_k = |w_{+, k} w_{-, k}|.
\end{align*}

\section{Continuous-Time Theorems} \label{proof_mgf}

\subsection{Convergence of Momentum Gradient Flow}\label{convergence}

Momentum gradient flow (with $\l > 0$),
\begin{equation*}
    \l \ddot{w_t} + \dot{w_t} + \nabla F(w_t) = 0,
\end{equation*}
also known in the optimisation literature as the heavy-ball with friction ODE or the heavy-ball dynamical system with constant damping coefficient, has been the object of extensive mathematical study over the years \citep{analytic_nonlinearity, heavy_ball_friction, alvarez_convex_heavy_ball, quasiconvex_heavy_ball, lyapunov_heavy_ball, PL_convergence_heavy_ball}. If we abstract away from the diagonal linear network setting and consider an unspecified loss $F \in C^1(\R^D, \R_{\geq 0})$ with locally Lipschitz gradient, we can still identify a useful Lyapunov function, which perhaps motivated the study of the ODE in the first place. The function in question happens to be the energy of the system
\begin{equation}\label{energy}
    E_t = F(w_t) + \frac{\l}{2} \n{\dot{w}_t}_2^2,
\end{equation}
whose nonpositive time-derivative $\dot{E}_t = -\n{\dot{w}_t}_2^2$ allows us to prove the global existence and uniqueness of a solution to MGF [\cite{heavy_ball_friction}, Theorem 3.1] in this more general setting. We note that by an easy inductive argument, when the function $F$ is $C^k$-smooth, the MGF solution $w_t$ is $C^{k+1}$-smooth. Hence, in our setting where the diagonal-neural-network loss $F$ is $C^\infty$-smooth, the learning trajectory $w_t$ is also $C^\infty$-smooth.

\myparagraph{Convergence under \Cref{boundedness}.}

Under the assumption of a bounded trajectory -- $w_t \in L^\infty(0, \infty)$, one can prove the following convergences~\citep{heavy_ball_friction}:
\begin{equation*}
    \lim_{t \to \infty} \dot{w}_t = \lim_{t \to \infty} \nabla F(w_t) = 0.
\end{equation*}
However, even when bounded, the iterates $w_t$ need not converge as demonstrated by the coercive function from Section 4.3 in \citep{heavy_ball_friction}.
Nevertheless, when the loss $F$ is also analytic, as in the case of diagonal linear networks, assuming boundedness, one can further prove iterate convergence $\lim_{t \to \infty} w_t = w_\infty$ [\cite{analytic_nonlinearity}].

Unfortunately, without assuming boundedness, iterate convergence has been established only in the cases of convex loss [\cite{alvarez_convex_heavy_ball}], quasiconvex loss [\cite{quasiconvex_heavy_ball}], and loss satisfying the Polyak-Lojasiewicz inequality [\cite{PL_convergence_heavy_ball}]. Thus, the square loss for a diagonal linear network (and neural networks in general) falls out of the scope of these few favorable cases due to non-convexity and an abundance of local and global minima. For that reason, we posit Assumption \ref{boundedness}, which holds true empirically in all our experiments on diagonal linear networks.

\myparagraph{Convergence to 0 Loss under \Cref{non-zero-balancedness}.}

Let us now go back to the specific case of diagonal linear networks where the loss is given by $F(w) = \L(u \p v)$ for $w = (u, v)$. Notice that from the discussion above, if we assume boundedness of the trajectory, we have
\begin{equation*}
    \lim_{t \to \infty} \nabla F(w_t) = (\nabla \L(\theta_\infty)\p v_\infty, \nabla \L(\t_\infty) \p u_\infty) = 0.
\end{equation*}
Therefore, since $\nabla \L(\t_\infty) \p \D_\infty = 0$, if the balancedness at infinity $\D_\infty$ has nonzero coordinates, we can conclude that $\nabla \L(\t_\infty) = 0$. Recalling that $\L$ is convex, we get that $\L(\t_\infty) = 0$. Hence, $\t_\infty$ interpolates the dataset.

\subsection{Proof of Time-Varying Momentum Mirror Flow}\label{mirror_flow_proof}

In our discussion in \Cref{convergence}, we saw that assuming
\begin{enumerate}
    \item[1)] iterate boundedness: $u_t, v_t \in L^\infty(0, \infty)$, and
    \item[2)] nonzero balancedness at infinity: $\D_{\infty, i} \neq 0, \en \forall i \in [d]$,
\end{enumerate}
we can prove that MGF over a diagonal linear network \eqref{MGF:diag} converges to an interpolator $\t_\infty$.\footnote{Note that we also refer to $\t_\infty$ as $\t^{\texttt{MGF}}$.} Before we jump into the proof of \Cref{prop:2nd_order_MF}, we need to establish the following lemma.

\begin{lemma}\label{gradient_integral}
    Assuming that $u_t, v_t \in L^\infty(0, \infty)$ and $\D_{\infty, i} \neq 0, \en \forall i \in [d]$, the following integral limit exists:
    \begin{equation*}
        \lim_{t \to \infty} \int_0^t \nabla \L(\t_s) ds = \int_0^\infty \nabla \L(\t_t) \dd t.
    \end{equation*}
    Consequently,
    \begin{equation*}
        \lim_{t \to \infty} \int_0^t \nabla \L(\t_s) e^{-\frac{t-s}{\l}} ds = 0.
    \end{equation*}
\end{lemma}

\begin{proof}
    Let us consider the $(w_+, w_-)$-reparametrisation of MGF \eqref{MGF:diag} given by \Cref{w_mgf}:
    \begin{equation*}
        \l \ddot{w}_{\pm,t} + \dot{w}_{\pm,t} \pm \nabla \L(\t_t) \p w_{\pm,t} = 0.
    \end{equation*}
    Since we assumed that $\D_\infty$ has nonzero coordinates, there exists $T \geq 0$ such that for all $t \geq T$, $w_{\pm, t}$ have nonzero coordinates. Hence, for $t \geq T$, we can safely divide by $w_{\pm, t}$ to obtain
    \begin{equation*}
        \l \frac{\dd^2 \ln |w_{\pm, t}|}{\dd t^2} + \frac{\dd \ln |w_{\pm, t}|}{\dd t} \pm \nabla \L(\t_t) + \l \left(\frac{\dot{w}_{\pm, t}}{w_{\pm, t}}\right)^2 = 0.
     \end{equation*}

    Let us notice a couple of things. First, as we discussed in \Cref{convergence}, the boundedness of the iterates forces $w_{\pm, t}$ to converge to some vectors $w_{\pm, \infty}$ with nonzero coordinates since we assumed the coordinates of $\D_\infty = w_{+, \infty}w_{-, \infty}$ are nonzero. Hence,
    \begin{equation*}
        \n{\frac{\dot{w}_{\pm, t}^2}{w_{\pm, t}^2}}_\infty \leq \texttt{const} \cdot (\n{\dot{u}_t^2}_\infty + \n{\dot{v}_t^2}_\infty),
    \end{equation*}
    where the RHS is integrable as we saw in the proof of \Cref{small_lambda}. Second, from the discussion in \Cref{convergence}, we know that $\lim_{t \to \infty} \dot{w}_{\pm, t} = 0$, so $\lim_{t \to \infty} \frac{\dd \ln |w_{\pm, t}|}{\dd t} = 0$.

    Now, for $t \geq T$,
    \begin{align*}
        \int_T^t \nabla \L(\t_s) ds & = \mp \left( \l \int_T^t \frac{\dd^2 \ln |w_{\pm, s}|}{\dd t^2} ds + \int_T^t \frac{\dd \ln |w_{\pm, s}|}{\dd t} ds + \int_T^t \left(\frac{\dot{w}_{\pm, t}}{w_{\pm, t}}\right)^2 ds \right)\\
        & = \mp \left( \l \frac{\dd \ln |w_{\pm, s}|}{\dd t} \Big|_T^t + \ln |w_{\pm, s}| \Big|_T^t + \int_T^t \left(\frac{\dot{w}_{\pm, t}}{w_{\pm, t}}\right)^2 ds \right).
    \end{align*}
    So, using the above observations and letting $t \to \infty$ yields
    \begin{equation*}
        \lim_{t \to \infty} \int_T^t \nabla \L(\t_s) ds = \mp \left( -\l \frac{\dd \ln |w_{\pm, T}|}{\dd t} - \ln |w_{\pm, T}| +  \ln |w_{\pm, \infty}| - \int_T^\infty \left(\frac{\dot{w}_{\pm, t}}{w_{\pm, t}}\right)^2 \dd t \right).
    \end{equation*}
    Thus,  we conclude that $\lim_{t \to \infty} \int_0^t \nabla \L(\t_s) ds$ exists, and therefore, $\lim_{t \to \infty} \int_0^t \nabla \L(\t_s) e^{-\frac{t-s}{\l}} ds = 0$.
\end{proof}

We are now well-equipped to prove \Cref{prop:2nd_order_MF}. We note that we phrased \Cref{prop:2nd_order_MF} rather succinctly in the main part of the paper due to space considerations. In what follows, we restate \Cref{prop:2nd_order_MF} by precisely specifying the underlying assumptions.

\begin{proposition*}
    Assume the solution $(u_t, v_t)$ of MGF \eqref{MGF:diag} is bounded. If we also assume that the balancedness at infinity $\D_\infty$ has nonzero coordinates, then there exists a time $T \geq 0$, after which the predictors $\theta_t = u_t \odot v_t$ follow a momentum mirror flow with time-varying potentials $\Phi_t$:
        \begin{align*}
            \lambda \frac{\dd^2  \nabla \Phi_t (\theta_t) }{\dd t^2} + \frac{\dd  \nabla \Phi_t (\theta_t) }{\dd t} + \nabla \L(\theta_t) = 0.
        \end{align*}
    Furthermore, if we assume that the balancedness $\D_t$ remains nonzero for $t \in [0, +\infty]$, then the momentum mirror flow holds for every $t \geq 0$.
\end{proposition*}

\begin{proof}
    We will consider the $(w_+, w_-)$-reparametrisation of momentum gradient flow \eqref{MGF:diag} introduced in \Cref{reparametrisation}. For convenience of the reader, we recall this reparametrisation here:
    \begin{equation*}
        \l \ddot{w}_{\pm, t} + \dot{w}_{\pm, t} \pm \nabla \L(\t_t) \p w_{\pm, t} = 0.
    \end{equation*}

    Now, let $\xi : \R_{\geq 0} \to \R^d$ be the $C^\infty(0, \infty)$ solution of the following ODE:
    \begin{equation*}
        \l \ddot{\xi}_t + \dot{\xi}_t + \nabla \L(\t_t) = 0,
    \end{equation*}
    with the constraint $\xi_0 = \dot{\xi}_0 = 0$.
    Hence, by \Cref{convolution},
    \begin{equation*}
        \xi_t = -\int_0^t \nabla \L(\t_s) (1-e^{-\frac{t-s}{\l}}) \dd s,
    \end{equation*}
    and by \Cref{gradient_integral},
    \begin{equation*}
        \xi_\infty = -\int_0^\infty \nabla \L(\t_t) \dd t.
    \end{equation*}
    Thus, $\xi_t \in \texttt{span}(x_1, \dots, x_n)$, $\forall t \in [0, +\infty]$.
    
    Having fixed $\xi_t$, we define the quantities $\a_{\pm, t}$ for every $t \in [0, +\infty]$ through the following relation:
    \begin{equation*}
        \a_{\pm, t} = w_{\pm, t} \exp(\mp \xi_t).
    \end{equation*}
    So, $\D_t = |w_{+, t}w_{-, t}| = |\a_{+, t}\a_{-, t}|$. Furthermore,
    \begin{align*}
        \t_t & = \frac{1}{4}(w_{+, t}^2 - w_{-, t}^2) \\
        & = \frac{1}{4}(\a_{+, t}^2\exp(2\xi_t) - \a_{-, t}^2\exp(-2 \xi_t)) \\
        & = \frac{1}{2} \D_t \sinh\left( 2\xi_t + \ln \frac{|\a_{+, t}|}{|\a_{- ,t}|} \right)
    \end{align*}

    Since we assumed that $\D_\infty$ has nonzero coordinates, there exists $T \geq 0$ such that for all $t \geq T$, $w_{\pm, t}$ have nonzero coordinates. Hence, for $t \geq T$, the logarithm $\ln \frac{|\a_{+, t}|}{|\a_{- ,t}|}$ is well-defined. If we assume positive balancedness for $t \in [0, +\infty]$, then we can choose $T = 0$. From now until the end of the proof, whenever a time-dependent quantity features division by $\D_t$, we will tacitly assume that $t \geq T$.
    
    Let us now introduce the helper quantity $\phi_t$ through the following identity:
    \begin{equation*}
        \phi_t = \frac{1}{2}\ln \frac{|\a_{+, t}|}{|\a_{- ,t}|} = \frac{1}{2}\arcsinh\left( \frac{\a_{+, t}^2 - \a_{-, t}^2}{2\D_t} \right).
    \end{equation*}
    Then,
    \begin{equation*}
        \frac{1}{2}\arcsinh\left( \frac{2\t_t}{\D_t} \right) - \phi_t = \xi_t \in \texttt{span}(x_1, \dots, x_n).
    \end{equation*}
    So, if we consider the time-varying potential
    \begin{equation}\label{potential}
    \begin{split}
        \Phi_t(\t) & = \frac{1}{4} \sum_{i=1}^d \left( 2\t_i \arcsinh \left( \frac{2\t_i}{\D_{t, i}} \right) - \sqrt{4\t_i^2 + \D_{t, i}^2} + \D_{t, i} \right) - \ip{\phi_t}{\t} \\
        & = \psi_{\D_t}(\t) - \ip{\phi_t}{\t},
    \end{split}
    \end{equation}
    where $\psi_{\D_t}$ is the hyperbolic entropy defined in \Cref{hyperbolic_entropy}, then
    \begin{equation*}
        \nabla \Phi_t(\t) = \frac{1}{2}\arcsinh\left( \frac{2\t}{\D_t} \right) - \phi_t.
    \end{equation*}
    Notice that $\nabla^2 \Phi_t = \texttt{diag}\left( 1/\sqrt{4\t^2 + \D_t^2} \right) \succ 0$. Hence, $\Phi_t$ is a mirror map. Furthermore, $\nabla \Phi_t(\t_t) = \xi_t$ for $t \geq T$, so
    \begin{equation*}
        \lambda \frac{\dd^2  \nabla \Phi_t (\theta_t) }{\dd t^2} + \frac{\dd  \nabla \Phi_t (\theta_t) }{\dd t} + \nabla \L(\theta_t) = 0.
    \end{equation*}
\end{proof}

\subsection{Proof of \Cref{main_mgf:general}} \label{app:main_mgf:general}

We are now ready to prove our main result for the implicit bias of momentum gradient flow on diagonal linear networks.

\MainMGFGeneral*

We split the proof into two parts for conceptual clarity. In the first part, we utilise the time-varying mirror flow from \Cref{prop:2nd_order_MF} to derive the implicit regularisation $\t^{\texttt{MGF}} = \argmin_{\t^\star \in \mathcal{S}} \  D_{\psi_{\D_\infty}}(\t^\star, \tilde{\t}_0)$. Then, in the second part, we prove that the integral quantities $I_\pm$ from \Cref{IntegralQuantities} are well-defined, and we give the trajectory-dependent characterisations of the asymptotic balancedness $\D_\infty$ and the perturbed initialisation $\tilde{\t}_0$.

\subsubsection{Proof of Implicit Regularisation.}

In \Cref{prop:2nd_order_MF}, we proved that whenever the MGF trajectory is bounded and the coordinates of $\D_\infty$ are nonzero, there exists a time $T \geq 0$, after which the predictors $\t_t$ follow a momentum mirror flow with potentials given by \Cref{potential}. Recall that for $t \geq T$,
\begin{equation*}
    \nabla \Phi_t(\t_t) = \frac{1}{2}\arcsinh\left( \frac{2\t}{\D_t} \right) - \phi_t = - \xi_t \in \texttt{span}(x_1, \dots, x_n).
\end{equation*}
where $\xi_t = -\int_0^t \nabla \L(\t_s) (1-e^{-\frac{t-s}{\l}}) ds$, $\a_{\pm, t} = w_{\pm, t} \exp(\mp \xi_t)$, and $\phi_t = \frac{1}{2}\arcsinh\left( \frac{\a_{+, t}^2 - \a_{-, t}^2}{2\D_t} \right)$.

Now, as $t \to \infty$, $\xi_t$ and the MGF iterates converge, so we know that $\nabla \Phi_\infty(\t_\infty) \in \texttt{span}(x_1, \dots, x_n)$, where $\Phi_\infty(\t) = \psi_{\D_\infty}(\t) - \ip{\phi_\infty}{\t}$. Thus, we can use the familiar Bregman-Cosine-Theorem trick to characterise the interpolator $\t_\infty$. We proceed with this characterisation.

Let $\tilde{\t}_0$ be a perturbation term such that $\nabla \Phi_\infty(\tilde{\t}_0) = 0$. Equivalently,
\begin{gather*}
    \frac{1}{2}\arcsinh\left( \frac{2\tilde{\t}_0}{\D_\infty^2} \right) - \phi_\infty = 0 \iff \nonumber \\
    \arcsinh\left( \frac{2\tilde{\t}_0}{\D_\infty^2} \right) - \arcsinh\left( \frac{\a_{+, \infty}^2 - \a_{-, \infty}^2}{2\D_\infty^2} \right) = 0 \iff \nonumber \\ 
    \tilde{\t}_0 = \frac{\a_{+, \infty}^2 - \a_{-, \infty}^2}{4}.
\end{gather*}
Note that $\a_{\pm, \infty} = w_{\pm, \infty}\exp(\pm \int_0^\infty \nabla \L(\t_t) \dd t)$ by \Cref{gradient_integral} and $\D_\infty = |\a_{+,\infty} \a_{-,\infty}|$.

Now, let $\t^\star \in \mathcal{S}$ be an arbitrary interpolator of the dataset. Then, $\t^\star - \t_\infty \in \mathrm{ker}(X) = \texttt{span}(x_1, \dots, x_n)^\perp$. Hence, the Bregman Cosine Theorem yields
\begin{align*}
    D_{\Phi_\infty}(\t^\star, \tilde{\t}_0) & = D_{\Phi_\infty}(\t^\star, \t_\infty) + D_{\Phi_\infty}(\t_\infty, \tilde{\t}_0) + \ip{\t^\star - \t_\infty}{\nabla \Phi(\t_\infty) - \nabla \Phi(\tilde{\t}_0)} \\
    & = D_{\Phi_\infty}(\t^\star, \t_\infty) + D_{\Phi_\infty}(\t_\infty, \tilde{\t}_0),
\end{align*}
where we used that $\nabla \Phi(\t_\infty) - \nabla \Phi(\tilde{\t}_0) \in \texttt{span}(x_1, \dots, x_n)$.
Thus,
\begin{align*}
    \t_\infty & = \argmin_{\t^\star \in \mathcal{S}} D_{\Phi_\infty}(\t^\star, \tilde{\t}_0) \\
    & = \argmin_{\t^\star \in \mathcal{S}} \Phi_\infty(\t^\star).
\end{align*}

Finally, notice that $\nabla \psi_{\D_\infty}(\tilde{\t}_0) = \frac{1}{2}\arcsinh\left( \frac{2\tilde{\t}_0}{\D_\infty} \right) = \phi_\infty$ as we showed above. Hence,
\begin{equation*}
    D_{\psi_{\D_\infty}}(\t, \tilde{\t}_0) = \Phi_\infty(\t) - \psi_{\D_\infty}(\tilde{\t}_0) + \ip{\nabla \psi_{\D_\infty}(\tilde{\t}_0)}{\tilde{\t}_0}.
\end{equation*}
Thus, we conclude that
\begin{equation*}
    \t_\infty = \argmin_{\t^\star \in \mathcal{S}} \Phi_\infty(\t^\star) = \argmin_{\t^\star \in \mathcal{S}} D_{\psi_{\D_\infty}}(\t^\star, \tilde{\t}_0).
\end{equation*}
\qed

\subsubsection{Proof of Trajectory-Dependent Characterisation.}\label{integral_formulas}

We just showed that the recovered interpolator by MGF solves the constrained minimisation problem $\t_\infty = \argmin_{\t^\star \in \mathcal{S}} D_{\psi_{\D_\infty}}(\t^\star, \tilde{\t}_0)$, where $\D_\infty = |\a_{+,\infty} \a_{-,\infty}|$, $\tilde{\t}_0 = (\a_{+, \infty}^2 - \a_{-, \infty}^2)/4$, and $\a_{\pm, \infty} = w_{\pm, \infty}\exp(\pm \int_0^\infty \nabla \L(\t_t) \dd t)$. Clearly, these opaque characterisations of $\D_\infty$ and $\tilde{\t}_0$ prevent us from describing how the magnitude of these quantities compares to the magnitude of the initial balancedness $\D_0$ and the initialisation scale $\a = \max(|u_0|, |v_0|)$. Ideally, we would like to find formulas for $\D_\infty$ and $\tilde{\t}_0$ which show that $\tilde{\t}_0 \ll \t^\star$, $\forall \t^\star \in \mathcal{S}$ and $\D_\infty < \D_0$ so that we can conclude that $\t^{\texttt{MGF}} \approx \argmin_{\t^\star \in \mathcal{S}} \psi_{\D_\infty}(\t^\star)$ enjoys better sparsity guarantees than $\t^{\texttt{GF}} \approx \argmin_{\t^\star \in \mathcal{S}} \psi_{\D_0}(\t^\star)$. In what follows, we derive such formulas.

In our subsequent arguments, for a vector $z \in \R^d$ and a coordinate $i \in [d]$, we will denote with $z(i)$ the $i^{\text{th}}$ coordinate of $z$ in order to reduce the index bloat. Let us consider again the $(w_+, w_-)$-reparametrisation of MGF discussed in \Cref{reparametrisation}:
\begin{equation}\label{w_ode}
    \l \ddot{w}_{\pm,t} + \dot{w}_{\pm,t} \pm \nabla \L(\t_t) \p w_{\pm,t} = 0.
\end{equation}
 Notice that if for some $T > 0$ and $i \in [d]$, $w_{+, T}(i) = 0$, then $\dot{w}_{+,T}(i)$ must be nonzero. Indeed, as we argued in \Cref{convergence}, MGF \eqref{MGF:diag} admits a unique global solution. And if $w_{+, T}(i) = \dot{w}_{+,T}(i) = 0$, then we could construct another solution $(w_{+,t}', w_{-,t}')$ of MGF such that $w_{+, t}'(i) = \dot{w}_{+, t}'(i) = 0, \en \forall t \geq 0$, and $w_{\pm, T} = w_{\pm, T}'$. By uniqueness, we get that $w_{\pm, t} = w_{\pm,t }', \en \forall t \geq 0$ However, the newly constructed solution will not be consistent with the imposed initialisation $\D_0 \neq 0$,. Hence, $\D_0 \neq 0$ prevents $w_{+,t}(i)$ and $\dot{w}_{+}(i)$ from hitting 0 simultaneously. Similarly, this situation cannot occur for $w_{-,t}$.

Until further notice, we fix a coordinate $i \in [d]$ and consider \cref{w_ode} only in the $i^{\text{th}}$ coordinate without explicit mention. If $w_{\pm, T} = 0$ for some $T > 0$, then $\dot{w}_{\pm, T} \neq 0$. Hence, for some small $\d > 0$, $\dot{w}_{\pm, t}$ does not change sign on $[T - \d, T+ \d]$, so $w_{\pm, t}$ either strictly increases or decreases on $[T - \d, T+ \d]$. Therefore, $w_{\pm, t} \neq 0$ on $[T-\d, T+\d] \setminus \{T\}$ implying that $w_{\pm, t}$ equals 0 at most a countable number of times. Recall that by \Cref{non-zero-balancedness}, there exists a time $T_\infty$ after which $w_{\pm}$ does not change sign. Therefore, if we assume that $w_{\pm}$ vanishes on infinitely many points $T_1 < T_2 < \dots < T_\infty$, then by compactness, the limit $\tau = \lim_{m \to \infty} T_m$ exists. Since $w_\pm$ is continuous, we infer that $w_{\pm, \tau} = 0$. Moreover, by the Mean Value Theorem, for every $m \geq 1$, there exists $T_m' \in (T_m, T_{m+1})$ such that $\dot{w}_{T_m'} = 0$. Notice that $\lim_{m \to \infty} T_m' = \tau$ as well. Hence, by continuity, $w_{\pm, \tau} = \dot{w}_{\pm, \tau} = 0$ -- a contradiction.

Hence, $w_{\pm}$ vanishes on a finite set of points. Let us order these vanishing times as $0 < T_1 < \dots < T_N$ and let $T_0 = 0$ and $T_{N+1} = +\infty$. Observe that for $t \notin \mathcal{T} = \{T_i : i \in [N]\}$, we can safely divide both sides of \cref{w_ode} by $w_{\pm,t}$ to obtain
\begin{equation*}
        \l \frac{\ddot{w}_{\pm,t}}{w_{\pm,t}} + \frac{\dot{w}_{\pm,t}}{w_{\pm,t}} \pm \nabla \L(\t_t)= 0.
\end{equation*}
The last expression is equivalent to
\begin{equation*}
        \l \left ( \frac{\ddot{w}_{\pm,t}}{w_{\pm,t}} - \left(\frac{\dot{w}_{\pm, t}}{w_{\pm, t}}\right)^2 \right ) + \frac{\dot{w}_{\pm,t}}{w_{\pm,t}} \pm \nabla \L(\t_t) + \l \left(\frac{\dot{w}_{\pm, t}}{w_{\pm, t}}\right)^2= 0,
\end{equation*}
which can be rewritten as
\begin{equation*}
    \l \frac{\dd^2 \ln (\sgn(w_{\pm, t})w_{\pm, t})}{\dd t^2} + \frac{\dd \ln (\sgn(w_{\pm, t})w_{\pm, t})}{\dd t} \pm \nabla \L(\t_t) + \l \left(\frac{\dot{w}_{\pm, t}}{w_{\pm, t}}\right)^2 = 0.
\end{equation*}

Let us define a new function $g_\pm : \R_{\geq 0} \setminus \mathcal{T} \to \R^d$ through the relation $g_{\pm, t} = \ln (\sgn(w_{\pm, t})w_{\pm, t})$. Then, $g_\pm$ is $C^\infty$-smooth on $\R_{\geq 0} \setminus \mathcal{T}$ and satisfies the following ODE:
\begin{equation}\label{g_ode}
    \l \ddot{g}_{\pm, t} + \dot{g}_{\pm, t} \pm \nabla \L(\t_t) + \l \left(\frac{\dot{w}_{\pm, t}}{w_{\pm, t}}\right)^2 = 0.
\end{equation}

\paragraph{Induction on Vanishing Times.} Now, we proceed to prove by induction on $N-1 \geq m \geq 0$ that for $\tau \in (T_m, T_{m+1})$ the following 3 things hold:
\begin{itemize}
    \item The following integral quantities\footnote{See \Cref{mpv} for the definition of $\mathrm{m.p.v.}$} exist and are finite:
    \begin{equation*}
        \mathrm{m.p.v.} \int_{0}^\tau \Big( \frac{\dot{w}_{\pm, t}}{w_{\pm, t}} \Big)^2 e^{-\frac{\tau-s}{\lambda}} \sgn(w_{\pm, \tau}w_{\pm, t}) \dd t \quad \text{and} \quad \int_{0}^\tau \mathrm{m.p.v.} \int_{0}^t \Big( \frac{\dot{w}_{\pm, s}}{w_{\pm, s}} \Big)^2 e^{-\frac{t-s}{\lambda}} \sgn(w_{\pm, t} w_{\pm, s} ) \dd s \ \dd t.
    \end{equation*}
    \item The following identity holds:
    \begin{equation*}
        \dot{g}_{\pm, \tau} = - \mathrm{m.p.v.} \int_{0}^\tau \left[ \Big( \frac{\dot{w}_{\pm, t}}{w_{\pm, t}} \Big)^2 \pm \frac{1}{\l} \nabla \L(\t_t) \right] e^{-\frac{\tau-t}{\lambda}} \sgn(w_{\pm, \tau} w_{\pm, t}) \dd t - \frac{1}{\l}\sum_{k = 1}^m (-1)^{m-k} e^{-\frac{\tau - T_k}{\l}}.
    \end{equation*}
    \item The following identity holds:
    \begin{align*}
        g_{\pm, \tau} & = g_{\pm, 0} - \int_{0}^\tau \mathrm{m.p.v.} \int_{0}^t \left[ \Big( \frac{\dot{w}_{\pm, s}}{w_{\pm, s}} \Big)^2 \pm \frac{1}{\l} \nabla \L(\t_s) \right] e^{-\frac{t-s}{\lambda}} \sgn(w_{\pm, t} w_{\pm, s}) \dd s \ \dd t \\
        & - \sum_{k=1}^m (-1)^{m-k} \left(1 - e^{-\frac{\tau - T_k}{\l}}\right) + 2\sum_{1 \leq i < j \leq m } (-1)^{j-i} \left(1 - e^{-\frac{T_j - T_i}{\l}}\right).
    \end{align*}
\end{itemize}

Recall that $\nabla \L(\t_t)$ is a bounded function, so if the modified principal value from the first bullet point exists, then the modified principal values in the above identities are also well-defined.

\textbf{Base case: $m = 0$.}
Recall from the proof of \Cref{prop:2nd_order_MF} in \Cref{mirror_flow_proof} that $\dot{w}_{\pm} \in L^2(0, \infty)$. Now, since $w_{\pm, t}$ does not change signs on the interval $(T_0, \tau)$, we know that $1/w_{\pm, t} = \Omega(1)$. Hence,
\begin{equation*}
    \Big( \frac{\dot{w}_{\pm, s}}{w_{\pm, s}} \Big)^2 e^{-\frac{t-s}{\lambda}} \in L^1(0, \infty).
\end{equation*}
Similarly, $\Big( \frac{\dot{w}_{\pm, s}}{w_{\pm, s}} \Big)^2$ is integrable on all intervals $[T_i + \eps, T_{i+1} - \eps]$ for any small $\eps >0$.
Consequently, the integral quantities
\begin{equation*}
    \mathrm{m.p.v.} \int_{0}^\tau \Big( \frac{\dot{w}_{\pm, t}}{w_{\pm, t}} \Big)^2 e^{-\frac{t-s}{\lambda}} \sgn(w_{\pm, \tau} w_{\pm, t}) \dd t = \int_{0}^\tau \Big( \frac{\dot{w}_{\pm, t}}{w_{\pm, t}} \Big)^2 e^{-\frac{t-s}{\lambda}} \dd t
\end{equation*}
and
\begin{align*}
    \int_{0}^\tau \mathrm{m.p.v.} \int_{0}^t \Big( \frac{\dot{w}_{\pm, s}}{w_{\pm, s}} \Big)^2 e^{-\frac{t-s}{\lambda}} \sgn(w_{\pm, t} w_{\pm, s} ) \dd s \ \dd t & = \int_{0}^\tau \int_{0}^t \Big( \frac{\dot{w}_{\pm, s}}{w_{\pm, s}} \Big)^2 e^{-\frac{t-s}{\lambda}} \dd s \ \dd t \\
    & = \int_0^\tau \Big( \frac{\dot{w}_{\pm, s}}{w_{\pm, s}} \Big)^2 (1-e^{-\frac{\tau-t}{\l}}) \dd t
\end{align*}
are well-defined.
Moreover, after applying \Cref{convolution} to \cref{g_ode}, we get
\begin{align*}
    & \dot{g}_{\pm, \tau} = - \int_0^\tau \Big( \frac{\dot{w}_{\pm, t}}{w_{\pm, t}} \Big)^2 e^{-\frac{\tau-t}{\l}} \dd t \ \mp \frac{1}{\l}\int_0^\tau \nabla \L(\t_t) e^{-\frac{\tau-t}{\l}} \dd t \\
    & g_{\pm, \tau} = g_{\pm, 0} - \l \int_0^\tau \Big( \frac{\dot{w}_{\pm, s}}{w_{\pm, s}} \Big)^2 (1-e^{-\frac{\tau-t}{\l}}) \dd t \ \mp \int_0^\tau \nabla \L(\t_t) (1-e^{-\frac{\tau-t}{\l}}) \dd t,
\end{align*}
which concludes the proof of the base case.

\textbf{Induction step: $m \to m+1$.}
For $m \geq 0$, assume that for every $\tau \in [0, T_{m+1}) \setminus \mathcal{T}$ the expressions
\begin{align*}
    & \dot{g}_{\pm, \tau} = - \mathrm{m.p.v.} \int_{0}^\tau \left[ \Big( \frac{\dot{w}_{\pm, t}}{w_{\pm, t}} \Big)^2 \pm \frac{1}{\l} \nabla \L(\t_t) \right] e^{-\frac{\tau-t}{\lambda}} \sgn(w_{\pm, \tau} w_{\pm, t}) \dd t - \frac{1}{\l}\sum_{k = 1}^m (-1)^{m-k} e^{-\frac{\tau - T_k}{\l}}
\end{align*}
and
\begin{align*}
    g_{\pm, \tau} & = g_{\pm, 0} - \int_{0}^\tau \mathrm{m.p.v.} \int_{0}^t \left[ \Big( \frac{\dot{w}_{\pm, s}}{w_{\pm, s}} \Big)^2 \pm \frac{1}{\l} \nabla \L(\t_s) \right] e^{-\frac{t-s}{\lambda}} \sgn(w_{\pm, t} w_{\pm, s}) \dd s \ \dd t \\
    & - \sum_{k=1}^m (-1)^{m-k} \left(1 - e^{-\frac{\tau - T_k}{\l}}\right) + 2\sum_{1 \leq i < j \leq m } (-1)^{j-i} \left(1 - e^{-\frac{T_j - T_i}{\l}}\right).
    \end{align*}
are true and well-defined. We now want to extend the validity of these identities to $\tau \in (T_{m+1}, T_{m+2})$. For ease of notation during the induction step, let $T_{m+1} = T$, $w_\pm = w$, and $g_{\pm} = g$. Let $\eps > 0$ and let $T_\pm = T \pm \eps$.

Now, applying \Cref{convolution} to \cref{g_ode} yields
\begin{align*}
    & \dot{g}_{\tau} = \dot{g}_{T_+} e^{-\frac{\tau-T_+}{\l}} - \int_{T_+}^\tau \left[ \Big( \frac{\dot{w}_{t}}{w_{t}} \Big)^2 \pm \frac{1}{\l} \nabla \L(\t_t) \right] e^{-\frac{\tau-t}{\lambda}} \dd t \\
    & g_{\tau} = g_{T_+} + \dot{g}_{T_+} \int_{T_+}^\tau e^{-\frac{t-T_+}{\l}}  \dd t - \int_{T_+}^\tau \int_{T_+}^t \left[ \Big( \frac{\dot{w}_{s}}{w_{s}} \Big)^2 \pm \frac{1}{\l} \nabla \L(\t_s) \right] e^{-\frac{t-s}{\lambda}} \dd s \ \dd t.
\end{align*}
For further ease of notation and with some abuse of notation, let $f_t = \Big( \frac{\dot{w}_{t}}{w_{t}} \Big)^2 \pm \frac{1}{\l} \nabla \L(\t_t)$ on $\R_{\geq 0} \setminus \mathcal{T}$. We will shortly prove that $g_{T_+} - g_{T_-} = O(\eps)$ and $\dot{g}_{T_+} + \dot{g}_{T_-} + \frac{1}{\l} = O(\eps)$.\footnote{Whenever we write an equation of the form $A = B + O(\eps^r)$ for some $r > 0$, we mean that $A = B + C$, where $|C| = O(\eps^r)$.} Hence, the following limits will hold:
\begin{align*}
    & \dot{g}_{\tau} = \lim_{\eps \to 0} \left[ -\frac{1}{\l} e^{-\frac{\tau-T_+}{\l}} - \dot{g}_{T_-} e^{-\frac{\tau-T_+}{\l}} - \int_{T_+}^\tau f_t e^{-\frac{\tau-t}{\lambda}} \dd t \right] \\
    & g_{\tau} = \lim_{\eps \to 0} \left[ g_{T_-} - \frac{1}{\l} \int_{T_+}^\tau e^{-\frac{t-T_+}{\l}}  \dd t - \dot{g}_{T_-} \int_{T_+}^\tau e^{-\frac{t-T_+}{\l}}  \dd t - \int_{T_+}^\tau \int_{T_+}^t f_s e^{-\frac{t-s}{\lambda}} \dd s \ \dd t \right].
\end{align*}

\myparagraph{Induction step for $\dot{g}_\tau$.} Let us begin to untangle the first limit by substituting $\dot{g}_{T_-}$ with its integral formula given by the induction hypothesis. Notice that
\begin{align*}
    \dot{g}_{T_-} e^{-\frac{\tau-T_+}{\l}} & = - \mathrm{m.p.v.} \int_{0}^{T_-} f_t e^{-\frac{T_- - t}{\lambda}} \sgn(w_{T_-} w_t) \dd t \cdot e^{-\frac{\tau-T_+}{\l}} - \frac{e^{-\frac{\tau-T_+}{\l}}}{\l}\sum_{k = 1}^m (-1)^{m-k} e^{-\frac{T_- - T_k}{\l}} \\
    & = \mathrm{m.p.v.} \int_{0}^{T_-} f_t e^{-\frac{\tau - t}{\lambda}} \sgn(w_\tau w_t) \dd t \cdot e^{\frac{2\eps}{\l}} + \frac{e^{\frac{2\eps}{\l}}}{\l}\sum_{k = 1}^m (-1)^{(m+1)-k} e^{-\frac{\tau- T_k}{\l}},
\end{align*}
where we used that $\sgn(\tau) = -\sgn(T_-)$ since $w$ changes signs at $T$. Hence, we have that
\begin{align*}
    \dot{g}_{\tau} & = - \lim_{\eps \to 0} \Big[ \frac{1}{\l} e^{-\frac{\tau-T_+}{\l}} + \frac{e^{\frac{2\eps}{\l}}}{\l}\sum_{k = 1}^m (-1)^{(m+1)-k} e^{-\frac{\tau- T_k}{\l}} \\
    & + \mathrm{m.p.v.} \int_{0}^{T_-} f_t e^{-\frac{\tau - t}{\lambda}} \sgn(w_\tau w_t) \dd t \cdot e^{\frac{2\eps}{\l}} + \int_{T_+}^\tau f_t e^{-\frac{\tau-t}{\lambda}} \sgn(w_\tau w_t) \dd t \Big] \\
    & = \mp \int_0^\tau \frac{1}{\l} \nabla \L(\t_t) \sgn(w_\tau w_t) - \frac{1}{\l}\sum_{k=1}^{m+1} (-1)^{(m+1)-k} e^{-\frac{\tau- T_k}{\l}} \\
    & - \lim_{\eps \to 0} \left[ \mathrm{m.p.v.} \int_{0}^{T_-} \Big( \frac{\dot{w}_t}{w_t} \Big)^2 e^{-\frac{\tau - t}{\lambda}} \sgn(w_\tau w_t) \dd t \cdot e^{\frac{2\eps}{\l}} + \int_{T_+}^\tau \Big( \frac{\dot{w}_t}{w_t} \Big)^2 e^{-\frac{\tau-t}{\lambda}} \sgn(w_\tau w_t) \dd t \right],
\end{align*}
where the limit on the last line formally equals the modified principal value $\mathrm{m.p.v.} \int_{0}^\tau \Big( \frac{\dot{w}_t}{w_t} \Big)^2 e^{-\frac{\tau-s}{\lambda}} \sgn(w_\tau w_t) \dd t$ whose existence we want to prove as part of the induction step. In fact, notice that we just proved the existence of $\mathrm{m.p.v.} \int_{0}^\tau \Big( \frac{\dot{w}_t}{w_t} \Big)^2 e^{-\frac{\tau-s}{\lambda}} \sgn(w_\tau w_t) \dd t$ since both $\dot{g}_{\tau}$ and $\mp \int_0^\tau \frac{1}{\l} \nabla \L(\t_t) \sgn(w_\tau w_t)$ are finite quantities. Hence, for $\tau \in (T_{m+1}, T_{m+2})$,
\begin{equation*}
    \dot{g}_{\tau} = - \mathrm{m.p.v.} \int_{0}^\tau \left[ \Big( \frac{\dot{w}_{t}}{w_{t}} \Big)^2 \pm \frac{1}{\l} \nabla \L(\t_t) \right] e^{-\frac{\tau-t}{\lambda}} \sgn(w_{\tau} w_{t}) \dd t - \frac{1}{\l}\sum_{k=1}^{m+1} (-1)^{(m+1)-k} e^{-\frac{\tau- T_k}{\l}}.
\end{equation*}

\myparagraph{Induction step for $g_\tau$.} We move on to untangle the limit which equals $g_\tau$. By the induction hypothesis,
\begin{align*}
    & \dot{g}_{T_-} = - \mathrm{m.p.v.} \int_{0}^{T_-} f_t e^{-\frac{T_- - t}{\lambda}} \sgn(w_{T_-} w_t) \dd t - \frac{1}{\l}\sum_{k=1}^{m} (-1)^{m-k} e^{-\frac{T_- - T_k}{\l}}\\
    & g_{T_-} = g_0 - \int_{0}^{T_-}\mathrm{m.p.v.}\int_{0}^t f_s e^{-\frac{t-s}{\lambda}} \sgn(w_t w_s) \dd s \ \dd t \\
    & - \sum_{k=1}^m (-1)^{m-k} \left(1 - e^{-\frac{T_- - T_k}{\l}}\right) + 2\sum_{1 \leq i < j \leq m } (-1)^{j-i} \left(1 - e^{-\frac{T_j - T_i}{\l}}\right).
\end{align*}
Again, we can substitute $\sgn(w_{T_-})$ with $-\sgn(w_\tau)$, and after performing the familiar integral and limit manipulations, we obtain
\begin{align*}
    g_\tau & = g_0 \mp \int_{0}^\tau \int_{0}^t \frac{1}{\l} \nabla \L(\t_s) e^{-\frac{t-s}{\lambda}} \sgn(w_t w_s) \dd s \ \dd t - \lim_{\eps \to 0} \left[ A_\eps + B_\eps + C_\eps \right] \\
    & - \sum_{k=1}^{m+1} (-1)^{(m+1)-k} \left(1 - e^{-\frac{\tau - T_k}{\l}}\right) + 2\sum_{1 \leq i < j \leq m+1} (-1)^{j-i} \left(1 - e^{-\frac{T_j - T_i}{\l}}\right),
\end{align*}
where
\begin{align*}
    & A_\eps = \int_{0}^{T_-}\mathrm{m.p.v.}\int_{0}^t \Big( \frac{\dot{w}_s}{w_s} \Big)^2 e^{-\frac{t-s}{\lambda}} \sgn(w_t w_s) \dd s \ \dd t \\
    & B_\eps = \int_{T_+}^{\tau} \mathrm{m.p.v.} \int_0^{T_-} \Big( \frac{\dot{w}_s}{w_s} \Big)^2 e^{-\frac{t-s}{\lambda}} \sgn(w_t w_s) \dd s \ \dd t \cdot e^{\frac{2\eps}{\l}}\\
    & C_\eps = \int_{T_+}^\tau \int_{T_+}^t \Big( \frac{\dot{w}_s}{w_s} \Big)^2 e^{-\frac{t-s}{\lambda}} \sgn(w_t w_s) \dd s \ \dd t.
\end{align*}

\begin{figure}[h!]
    \centering
    \includegraphics[width=0.5\textwidth]{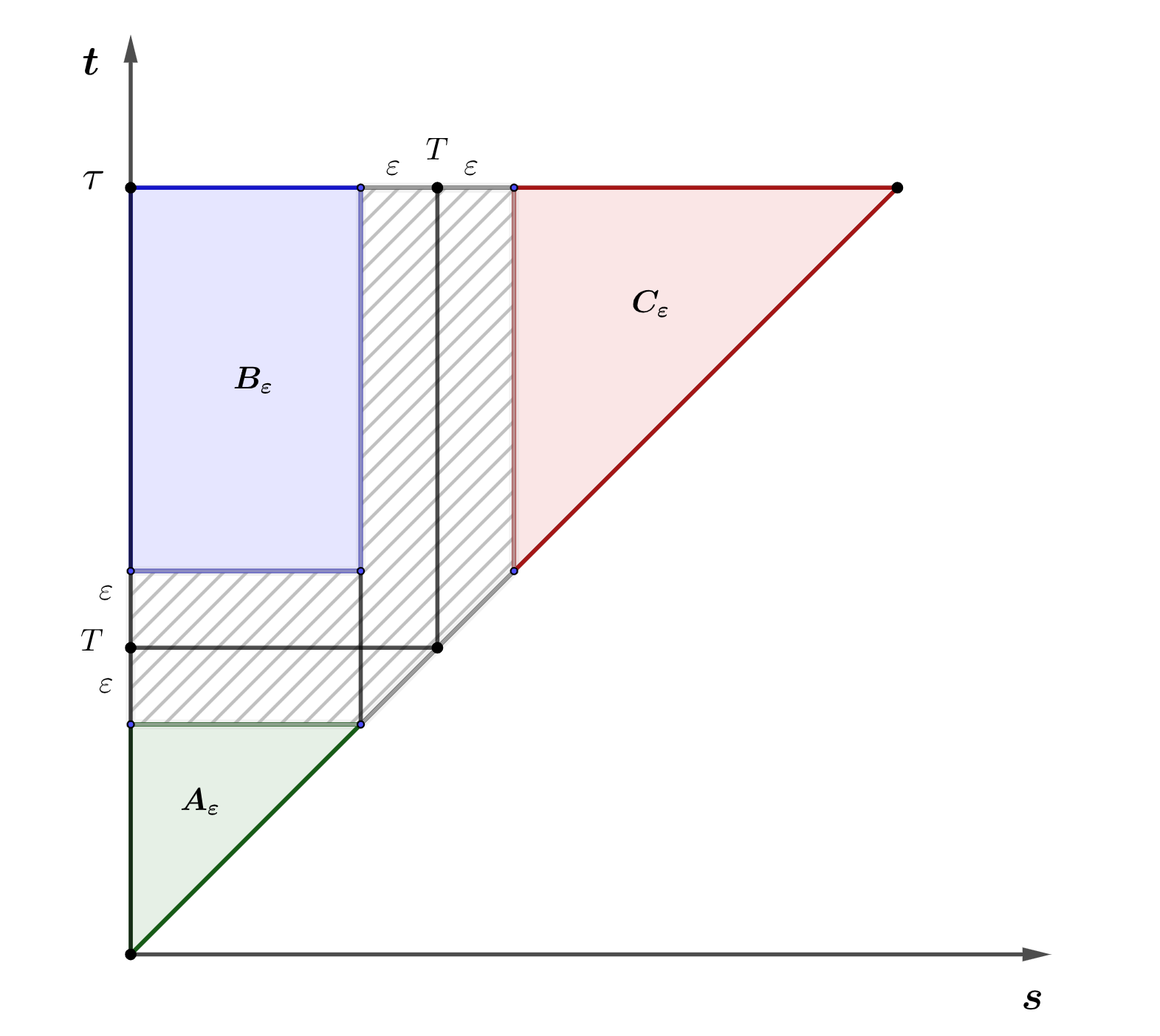}
    \vspace{-3mm}
    \caption{A visualisation of the areas over which we integrate $\Big( \frac{\dot{w}_s}{w_s} \Big)^2 e^{-\frac{t-s}{\lambda}} \sgn(w_t w_s)$ in the above limit.}
    \label{fig:app:integral_decomposition}
\end{figure}

Notice that formally the limit $\lim_{\eps \to 0} [A_\eps + B_\eps + C_\eps]$ equals the integral quantity
\begin{equation*}
    \int_{0}^\tau \mathrm{m.p.v.} \int_{0}^t \Big( \frac{\dot{w}_{s}}{w_{s}} \Big)^2 e^{-\frac{t-s}{\lambda}} \sgn(w_t w_{s}) \dd s \ \dd t,
\end{equation*}
whose existence we just proved as a consequence of the fact that
\begin{equation*}
    g_0 \mp \int_{0}^\tau \int_{0}^t \frac{1}{\l} \nabla \L(\t_s) e^{-\frac{t-s}{\lambda}} \sgn(w_t w_s) \dd s \ \dd t - \sum_{k=1}^{m+1} (-1)^{(m+1)-k} \left(1 - e^{-\frac{\tau - T_k}{\l}}\right) + 2\sum_{1 \leq i < j \leq m+1} (-1)^{j-i} \left(1 - e^{-\frac{T_j - T_i}{\l}}\right) - g_\tau
\end{equation*}
is well-defined and finite. Thus, for $\tau \in (T_{m+1}, T_{m+2})$,
\begin{align*}
     g_{\tau} & = g_{0} - \int_{0}^\tau \mathrm{m.p.v.} \int_{0}^t \left[ \Big( \frac{\dot{w}_{s}}{w_{s}} \Big)^2 \pm \frac{1}{\l} \nabla \L(\t_s) \right] e^{-\frac{t-s}{\lambda}} \sgn(w_{t} w_{s}) \dd s \ \dd t \\
     & - \sum_{k=1}^{m+1} (-1)^{(m+1)-k} \left(1 - e^{-\frac{\tau - T_k}{\l}}\right) + 2\sum_{1 \leq i < j \leq m+1} (-1)^{j-i} \left(1 - e^{-\frac{T_j - T_i}{\l}}\right).
\end{align*}

\myparagraph{Proof of bounds.} In order to conclude the induction step, we still have to prove the following bounds:
\begin{equation*}
    g_{T_+} - g_{T_-} = O(\eps) \quad \text{and} \quad \dot{g}_{T_+} + \dot{g}_{T_-} + \frac{1}{\l} = O(\eps).
\end{equation*}
Recall that $g_{T\pm\eps} = \log |w_{T\pm\eps}|$ and that $w_T = 0, \ \dot{w_T} \neq 0$. From the Taylor expansion of $w_t$, we know that
\begin{equation*}
    w_{T\pm\eps} = \pm \eps \dot{w}_T + O(\eps^2).
\end{equation*}
Hence, $|w_{T+\eps}/w_{T-\eps}| = 1 + O(\eps)$. Therefore, using the Taylor expansion of the logarithm around 1, we get that
\begin{equation*}
    |g_{T_+} - g_{T_-}| = |\log (1+O(\eps))| = O(\eps).
\end{equation*}
Now, recall that $\dot{g}_{T\pm\eps} = \dot{w}_{T\pm\eps}/w_{T\pm\eps}$ and observe that
\begin{align*}
    & w_{T\pm\eps} = \pm \eps \dot{w}_T + \frac{1}{2}\eps^2 \ddot{w}_T + O(\eps^3) \\
    & \dot{w}_{T\pm\eps} = \dot{w}_T \pm \eps \ddot{w}_T + O(\eps^2).
\end{align*}
Hence,
\begin{equation*}
    \frac{w_{T+\eps}\dot{w}_{T-\eps} + w_{T-\eps}\dot{w}_{T+\eps}}{w_{T+\eps}w_{T-\eps}} = \frac{-\eps^2 \dot{w}_T \ddot{w}_T + O(\eps^3)}{-\eps^2\dot{w}_T^2 + O(\eps^3)} = \frac{\ddot{w}_T}{\dot{w}_T} + O(\eps).
\end{equation*}
Since, $\l \ddot{w}_T + \dot{w}_T \pm \nabla \L(\t_T) \p w_T = 0$ and $w_T = 0$, we get that $\frac{\ddot{w}_T}{\dot{w}_T} = - \frac{1}{\l}$,
which concludes the induction step.

\paragraph{Proof of \Cref{IntegralQuantities}.} Thus, we proved that for $\tau \in (T_m, T_{m+1})$, $m \in \{0, 1, \dots, N\}$,
\begin{align*}
    \ln |w_\tau| & = \ln |w_0| - \int_{0}^\tau \mathrm{m.p.v.} \int_{0}^t \left[ \Big( \frac{\dot{w}_{s}}{w_{s}} \Big)^2 \pm \frac{1}{\l} \nabla \L(\t_s) \right] e^{-\frac{t-s}{\lambda}} \sgn(w_{t} w_{s}) \dd s \ \dd t \\
    & - \sum_{k=1}^{m} (-1)^{m-k} \left(1 - e^{-\frac{\tau - T_k}{\l}}\right) + 2\sum_{1 \leq i < j \leq m } (-1)^{j-i} \left(1 - e^{-\frac{T_j - T_i}{\l}}\right).
\end{align*}
Recall that throughout our inductive proof we worked with a fixed coordinate $i \in [d]$ of $w_\pm$. Different coordinates of $w_\pm$ vanish at different points in time, so writing the sum the last line in a coordinate-agnostic way becomes impossible. Thus, deriving a simple expression for the full $d$-dimensional vector $w_{\pm, \tau}$ for any $\tau \in \R_{\geq 0}$ also becomes impossible. However, remembering that the finite nonzero limits $\lim_{\tau \to \infty} |w_{\pm, \tau}| = |w_{\pm, \infty}|$ exist and letting $\tau \to \infty$ yields an interesting result for the weights at infinity. Indeed, notice that for every vanishing time $T$, $\lim_{\tau \to \infty} \left(1 - e^{-\frac{\tau - T_k}{\l}}\right) = 1$. Hence,
\begin{equation*}
    \frac{1}{\l} \sum_{k=1}^N (-1)^{N-k} \left(1 - e^{-\frac{\tau - T_k}{\l}}\right) = \mathbf{1}_{\{N - \text{ odd}\}}.
\end{equation*}
For every $i \in [d]$, let $N_\pm(i)$ denote the number of vanishing points for the coordinate $w_\pm(i)$.  Let us define the $d$-dimensional parity vectors $P_\pm \in \{0,1\}^d$ such that $P_\pm(i) \equiv N_\pm(i) \mod 2$. Let us also define the $d$-dimensional vectors $Q_\pm \in \R^d$ such that for each coordinate $k \in [d]$,
\begin{equation*}
    Q_\pm(k) \coloneqq -2\sum_{1 \leq i < j \leq N_\pm(k)} (-1)^{j-i} \left(1 - e^{-\frac{T_{\pm, k}(j) - T_{\pm, k}(i)}{\l}}\right),
\end{equation*}
where $0 < T_{\pm, k}(1) < \dots < T_{\pm, k}(N_\pm(k))$ denote the vanishing times of the weight $w_\pm(k)$.
Hence, we obtain the formula
\begin{equation}\label{w-formula}
    |w_{\pm, \infty}| = |w_{\pm, 0}| e^{-(P_\pm + Q_\pm)} \exp\left( -\int_{0}^\infty \mathrm{m.p.v.} \int_{0}^t \left[ \Big( \frac{\dot{w}_{\pm, s}}{w_{\pm, s}} \Big)^2 \pm \frac{1}{\l} \nabla \L(\t_s) \right] e^{-\frac{t-s}{\lambda}} \sgn(w_{\pm, t} w_{\pm, s}) \dd s \ \dd t \right).
\end{equation}
Recall that in \Cref{gradient_integral}, we proved that the limit
\begin{equation*}
    \lim_{t \to \infty} \int_0^t \nabla \L(\t_s) ds = \int_0^\infty \nabla \L(\t_t) \dd t = \frac{1}{\l} \int_0^\infty \int_0^t \nabla \L(\t_s) e^{-\frac{t-s}{\lambda}} \dd s \ \dd t
\end{equation*}
exists and is finite. Therefore, we can decouple $\Big( \frac{\dot{w}_{\pm, s}}{w_{\pm, s}} \Big)^2$ and $\nabla \L(\t_s)$ from the above integral and show that the following integral limits exist and are finite:
\begin{align*}
     \int_{0}^\infty \mathrm{m.p.v.} \int_{0}^t \Big( \frac{\dot{w}_{\pm, s}}{w_{\pm, s}} \Big)^2 e^{-\frac{t-s}{\lambda}} \sgn(w_{\pm, t} w_{\pm, s}) \dd s \ \dd t \quad \text{and} \quad \int_{0}^\infty \int_{0}^t \nabla \L(\t_s) e^{-\frac{t-s}{\lambda}} \sgn(w_{\pm, t} w_{\pm, s}) \dd s \ \dd t.
\end{align*}
Hence, the integral quantities $\O_\pm$ from \Cref{IntegralQuantities} are well-defined and finite. Thus, we finally proved \Cref{IntegralQuantities}.

\paragraph{Trajectory-Dependent Characterisation.} We started this section with a promise for more insightful representations of $\D_\infty = |\a_{+, \infty} \a_{-, \infty}|$ and $\tilde{\t}_0 = (\a_{+, \infty}^2 - \a_{-, \infty}^2)/4$. We now deliver on that promise.

Recall that $\a_{\pm, \infty} = w_{\pm, \infty}\exp\left(\pm \int_0^\infty \nabla \L(\t_t) \dd t\right)$ and notice that
\begin{align*}
    \int_0^\infty \int_0^t \nabla \L(\t_s) e^{-\frac{t-s}{\lambda}} \dd s \ \dd t - \int_{0}^\infty \int_{0}^t \nabla \L(\t_s) e^{-\frac{t-s}{\lambda}} \sgn(w_{\pm, t} w_{\pm, s}) \dd s \ \dd t = 2 \int_{0}^\infty \int_{0}^t \nabla \L(\t_s) e^{-\frac{t-s}{\lambda}} \mathbf{1}_{\{w_{\pm,t} w_{\pm,s} < 0\}} \dd s \ \dd t.
\end{align*}
Hence, using the formula for $w_\pm$ from \Cref{w-formula}, we derive the following:
\begin{align*}
    |\a_{\pm,\infty}| = |w_{\pm,0}|e^{-(P_\pm + Q_\pm)} & \p \exp\left(-\int_{0}^\infty \mathrm{m.p.v.} \int_{0}^t \Big( \frac{\dot{w}_{\pm, s}}{w_{\pm, s}} \Big)^2 e^{-\frac{t-s}{\lambda}} \sgn(w_{\pm, t} w_{\pm, s}) \dd s \ \dd t\right) \\
    & \p \exp\left( \pm \frac{2}{\l} \int_{0}^\infty \int_{0}^t \nabla \L(\t_s) e^{-\frac{t-s}{\lambda}} \mathbf{1}_{\{w_{\pm,t} w_{\pm,s} < 0\}} \dd s \ \dd t \right).
\end{align*}
Now, let
\begin{equation}\label{BigLambda}
    \Lambda_\pm \coloneqq \mp \frac{2}{\l} \int_{0}^\infty \int_{0}^t \nabla \L(\t_s) e^{-\frac{t-s}{\lambda}} \mathbf{1}_{\{w_{\pm,t} w_{\pm,s} < 0\}} \dd s \ \dd t + P_\pm + Q_\pm,
\end{equation}
where the quantities $P_\pm$ and $Q_\pm$ were defined in the previous paragraph. Notice that as we promised underneath \Cref{IntegralQuantities}, $\Lambda_\pm$ vanish whenever the balancedness $\D_t$ remains strictly positive. Using the abbreviation $I_\pm = \O_\pm + \Lambda_\pm$, we get that
\begin{align*}
    & |\a_{\pm,\infty}| = |w_{\pm,0}| \p \exp\left(-I_\pm\right).
\end{align*}
Multiplying $|\a_{+,\infty}|$ by $|\a_{-,\infty}|$, we derive a formula for the asymptotic balancedness:
\begin{align}\label{delta-formula}
\begin{split}
    \D_\infty = \D_0 e^{-(P_+ + P_- + Q_+ + Q_-)} & \p \exp\left( -\int_{0}^\infty \mathrm{m.p.v.} \int_{0}^t \left[ \Big( \frac{\dot{w}_{+, s}}{w_{+, s}} \Big)^2 + \frac{1}{\l} \nabla \L(\t_s) \right] e^{-\frac{t-s}{\lambda}} \sgn(w_{+, t} w_{+, s}) \dd s \ \dd t \right) \\
    & \p \exp\left( -\int_{0}^\tau \mathrm{m.p.v.} \int_{0}^t \left[ \Big( \frac{\dot{w}_{-, s}}{w_{-, s}} \Big)^2 - \frac{1}{\l} \nabla \L(\t_s) \right] e^{-\frac{t-s}{\lambda}} \sgn(w_{-, t} w_{-, s}) \dd s \ \dd t \right) \\
    & \p \exp\left(\frac{2}{\l} \int_{0}^\infty \int_{0}^t \nabla \L(\t_s) e^{-\frac{t-s}{\lambda}} \left[\mathbf{1}_{\{w_{+,t} w_{+,s} < 0\}} - \mathbf{1}_{\{w_{-,t} w_{-,s} < 0\}}\right] \dd s \ \dd t \right).
\end{split}
\end{align}
Now, we can write $\D_\infty$ and $\tilde{\t}_0$ more succinctly as
\begin{equation*}
    \D_\infty = \D_0 \odot \exp \Big( - (I_{+} + I_{-}) \Big)
\end{equation*}
and
\begin{equation*}
\tilde{\t}_0 = \frac{1}{4} \Big( w_{+,0}^2 \odot \exp\left( - 2I_+ \right) - w_{-,0}^2 \odot \exp\left( - 2I_- \right) \Big),
\end{equation*}
which concludes the proof of \Cref{main_mgf:general}.
\qed

\subsubsection{Consequences for Generalisation.}\label{consequences}

We just proved that whenever MGF on a diagonal linear network converges and the balancedness at infinity is nonzero, we can characterize the recovered interpolator through the implicit regularization problem
\begin{align*}
    \t^{\texttt{MGF}} & = \argmin_{\t^\star \in \mathcal{S}} \  D_{\psi_{\D_\infty}}(\t^\star, \tilde{\t}_0) \\
    & = \argmin_{\t^\star \in \mathcal{S}} \left[ \psi_{\D_\infty}(\t^\star) - \ip{\nabla \psi_{\D_\infty}(\tilde{\t}_0) }{\t^\star} \right].
\end{align*}
Since
\begin{equation*}
    \psi_{\D_\infty}(\t) = \frac{1}{4} \sum_{i=1}^d \left( 2\t_i \arcsinh \left( \frac{2\t_i}{\D_{\infty, i}} \right) - \sqrt{4\t_i^2 + \D_{\infty, i}^2} + \D_{\infty, i} \right)
\end{equation*}
and
\begin{equation*}
    \nabla \psi_{\D_\infty}(\t) = \frac{1}{2}\arcsinh\left( \frac{2\t}{\D_\infty} \right),
\end{equation*}
for a small asymptotic balancedness $\D_\infty = O(\D_0) = O(\a^2)$ and small perturbed initialisation $|\tilde{\t}_0| = O(\a^2) \ll |\t^\star|$, we would expect $\psi_{\D_\infty}(\t^\star)$ to dominate $\ip{\nabla \psi_{\D_\infty}(\tilde{\t}_0) }{\t^\star}$. More formally, for a fixed $\t^\star \in \mathcal{S}$ and small $\D_\infty$ and $\tilde{\t}_0$, we have the following asymptotic equivalence:
\begin{equation*}
    \psi_{\D_\infty}(\t^\star) \underset{\scriptscriptstyle \a \to 0}{\sim} D_{\psi_{\D_\infty}}(\t^\star, \tilde{\t}_0).
\end{equation*}
Hence, $\t^{\texttt{MGF}} \approx \argmin_{\t^\star \in \mathcal{S}} \psi_{\D_\infty}(\t^\star) = \t_{\D_\infty}^{\texttt{GF}}$ as we discussed in \Cref{gf_bias}.
So, if $\D_\infty < \D_0$, \Cref{entropy_asymptotics} implies that the MGF predictor will benefit from better sparsity guarantees than the GF solution.

Therefore, to recap, for a small initialisation scale $\a$ and provided that the bounds $\D_\infty = O(\a^2)$ and $\tilde{\t}_0 = O(\a^2)$ hold, we conclude that the asymptotic balancedness at infinity $\D_\infty$ roughly controls the sparsity of the recovered interpolator. And when $\D_\infty < \D_0$, $\t^{\texttt{MGF}}$ will be sparser than $\t_{\a}^{\texttt{GF}}$. Unfortunately, without the assumption that the balancedness $\D_t$ remains strictly positive for all $t \in [0, +\infty]$, we cannot formally compare $\D_\infty$ and $\tilde{\t}_0$ with $\a$. 

Note that even without the bounds $\D_\infty = O(\a^2)$ and $\tilde{\t}_0 = O(\a^2)$, if $|\tilde{\t}_0| \ll |\t^\star|$, then $\psi_{\D_\infty}(\t^\star)$ still dominates $\ip{\nabla \psi_{\D_\infty}(\tilde{\t}_0) }{\t^\star}$. Indeed, our experiments clearly show that the perturbation term $\tilde{\t}_0$ can safely be ignored since $\t^{\texttt{MGF}} \approx \argmin_{\t^\star \in \mathcal{S}} \psi_{\D_\infty}(\t^\star)$ (see the discussion around \Cref{fig:app:mgf_triple_1}.)

\subsection{Non-Vanishing Balancedness}\label{app:main_mgf}

If we work under the assumption that the balancedness $\D_t = |w_{+, t} w_{-, t}|$ never vanishes, then much of the analysis from \Cref{integral_formulas} greatly simplifies. First, the integral quantities $P_\pm$ and $Q_\pm$ from the previous subsection become 0. Second, the multipliers $\sgn(w_{\pm,t}w_{\pm,s})$ become equal to 1 for all $t, s \in \R_{\geq 0}$. Hence, using Fubini's Theorem as in the proof of \Cref{convolution}, we get that
\begin{align*}
    \int_{0}^\tau \mathrm{m.p.v.} \int_{0}^t \left[ \Big( \frac{\dot{w}_{s}}{w_{s}} \Big)^2 \pm \frac{1}{\l} \nabla \L(\t_s) \right] & e^{-\frac{t-s}{\lambda}} \sgn(w_{t} w_{s}) \dd s \ \dd t \\
    & = \int_{0}^\tau \int_{0}^t \left[ \Big( \frac{\dot{w}_{s}}{w_{s}} \Big)^2 \pm \frac{1}{\l} \nabla \L(\t_s) \right] e^{-\frac{t-s}{\lambda}} \dd s \ \dd t \\
    & = \l \int_{0}^\tau \Big( \frac{\dot{w}_{t}}{w_{t}} \Big)^2 \left( 1 -e^{-\frac{\tau-t}{\lambda}} \right) \dd t \pm \int_{0}^\tau \nabla \L(\t_t) \left( 1 -e^{-\frac{\tau-t}{\lambda}} \right) \dd t.
\end{align*}
Therefore, referencing \cref{w-formula}, we can express the evolution of the iterates as follows:
\begin{equation}\label{w-formula-new}
    w_{\pm, \tau} = w_{\pm, 0} \exp\left( - \l \int_{0}^\tau \Big( \frac{\dot{w}_{\pm, t}}{w_{\pm, t}} \Big)^2 \left( 1 -e^{-\frac{\tau-t}{\lambda}} \right) \dd t \right) \exp\left( \mp \int_{0}^\tau \nabla \L(\t_s) \left( 1 -e^{-\frac{\tau-t}{\lambda}} \right) \dd t \right).
\end{equation}
Thus, the balancedness evolves as
\begin{equation}\label{delta-formula-new}
    \D_t = \D_0 \exp\left( - \l \int_{0}^\tau \left[ \Big( \frac{\dot{w}_{+, t}}{w_{+, t}} \Big)^2 + \Big( \frac{\dot{w}_{-, t}}{w_{-, t}} \Big)^2 \right] \left( 1 -e^{-\frac{\tau-t}{\lambda}} \right) \dd t \right).
\end{equation}
Now, from \Cref{gradient_integral}, we know that $\Big( \frac{\dot{w}_{\pm, t}}{w_{\pm, t}} \Big)^2$ is integrable and that \begin{equation*}
    \lim_{\tau \to \infty} \int_{0}^\tau \nabla \L(\t_s) \left( 1 -e^{-\frac{\tau-t}{\lambda}} \right) \dd t = \int_{0}^\infty \nabla \L(\t_s) \dd t
\end{equation*}
exists. Furthermore, from \Cref{regularity}, we know that that
\begin{equation*}
    \lim_{\tau \to \infty} \int_{0}^\tau \Big( \frac{\dot{w}_{\pm, t}}{w_{\pm, t}} \Big)^2 \left( 1 -e^{-\frac{\tau-t}{\lambda}} \right) \dd t = \int_{0}^\infty \Big( \frac{\dot{w}_{\pm, t}}{w_{\pm, t}} \Big)^2) \dd t.
\end{equation*}
Therefore, letting $\tau \to \infty$, we obtain the formulas
\begin{gather}\label{infinity-formulas}
    w_{\pm, \tau} = w_{\pm, 0} \exp\left( - \l \int_{0}^\infty \Big( \frac{\dot{w}_{\pm, t}}{w_{\pm, t}} \Big)^2 \dd t \right) \exp\left( \mp \int_{0}^\infty \nabla \L(\t_s) \dd t \right) \\
    \D_\infty = \D_0 \exp\left( - \l \int_{0}^\infty \left[ \Big( \frac{\dot{w}_{+, t}}{w_{+, t}} \Big)^2 + \Big( \frac{\dot{w}_{-, t}}{w_{-, t}} \Big)^2 \right] \dd t \right).
\end{gather}
Hence, clearly, $\D_\infty < \D_0$.

Finally, let us consider how the perturbed initialisation $\tilde{\t}_0$ looks like when $\D_t$ remains nonzero. Recall that $\tilde{\t}_0 = (\a_+^2 - \a_-^2)/4$, where $\a_{\pm, \infty} = w_{\pm, \infty}\exp\left(\pm \int_0^\infty \nabla \L(\t_t) \dd t\right)$. Thus,
\begin{equation*}
    \a_{\pm, \infty} = w_{\pm, 0} \exp\left( - \l \int_{0}^\infty \Big( \frac{\dot{w}_{\pm, t}}{w_{\pm, t}} \Big)^2 \dd t \right)
\end{equation*}
and
\begin{equation}\label{perturbed-formula}
    \tilde{\t}_0 = \frac{1}{4} \left[ w_{+, 0}^2 \exp\left( - 2 \l \int_{0}^\infty \Big( \frac{\dot{w}_{+, t}}{w_{+, t}} \Big)^2 \dd t \right) - w_{\-, 0}^2 \exp\left( - 2\l \int_{0}^\infty \Big( \frac{\dot{w}_{-, t}}{w_{-, t}} \Big)^2 \dd t \right) \right].
\end{equation}
Now, $\a_{\pm, \infty} < w_{\pm, 0} \leq 2\a$, where $\a = \max(\n{u_0}_\infty, \n{v_0}_\infty)$ stood for the initialisation scale. Hence, $|\tilde{\t}_0| < \a^2$.

Therefore, we just proved

\MainMGF*

\subsection{Behaviour of $\D_\infty$ for Small Values of $\lambda$}\label{behavior_of_delta} Since a precise asymptotic result for small $\lambda$ is technically difficult, in this section we focus on giving some qualitative results. For $\lambda > 0$, recall that our iterates follow
\begin{equation*}
        \l \ddot{w}_{\pm, t}^{(\lambda)} + \dot{w}_{\pm, t}^{(\lambda)} \pm \nabla \L(\t_t^{(\lambda)}) \p w_{\pm, t}^{(\lambda)} = 0,
\end{equation*}
where we explicitly highlight the dependency on $\lambda$. Therefore, we have
\begin{equation*}
 \frac{\dot{w}_{\pm, t}^{(\lambda)}}{w_{\pm, t}^{(\lambda)}} = \mp \nabla \L(\t_t^{(\lambda)}) - \l \frac{\ddot{w}^{(\lambda)}_{\pm, t} }{w_{\pm, t}}
\end{equation*}
and
\begin{equation*}
 \Big( \frac{\dot{w}_{\pm, t}^{(\lambda)}}{w_{\pm, t}^{(\lambda)}} \Big)^2 =  \nabla \L(\t_t^{(\lambda)}) ^2 +  \l^2 \Big(  \frac{\ddot{w}^{(\lambda)}_{\pm, t} }{w_{\pm, t}} \Big)^2 \pm 2 \lambda \L(\t_t^{(\lambda)})  \Big(  \frac{\ddot{w}^{(\lambda)}_{\pm, t} }{w_{\pm, t}} \Big).
\end{equation*}
Informally, we expect $(t \mapsto \nabla L(\t_t^{(\lambda)}))_{0 < \lambda \leq 1} \in L^2(0, +\infty)$ and $(t \mapsto \lambda \frac{\ddot{w}^{(\lambda)}_{\pm, t} }{w_{\pm, t}})_{\lambda} \underset{\lambda \to 0}{\longrightarrow} 0$ in $L^2$-norm (see Theorem 5.1 in \cite{heavy_ball_friction}). Hence, we get 
\begin{equation*}
 \int_0^\infty \Big( \frac{\dot{w}_{\pm, t}^{(\lambda)}}{w_{\pm, t}^{(\lambda)}} \Big)^2 \underset{\lambda \to 0}{\sim} \int_0^\infty \nabla \L(\t_t^{(\lambda)}) ^2 \dd t
\end{equation*}
and 
\begin{align*}
    \D_\infty \underset{\lambda \to 0}{\approx} \Delta_0 \exp \big( - 2 \lambda \int_0^\infty \nabla \L(\theta_s^{(\lambda)})^2 \dd s \big).
\end{align*}

\section{Discrete-Time Results} \label{proof_mgd}
In this section, we cover the proofs of our discrete-time results from \Cref{sec:MGD}. We first recall the SMGD recursion~\eqref{w_mgd} with the $w_{\pm}$-parametrisation from \Cref{reparametrisation}. Initialised at 
$w_{\pm, 0} = w_{\pm, 1} \in \R^d$, for  $k \geq 1$, the iterates follow
\begin{align}\label{app:eq:w+-}
 w_{\pm, k+1} = w_{\pm, k} \mp \g\nabla \L_{\cB_k}(\t_k) \p w_{\pm, k} + \b (w_{\pm, k} - w_{\pm, k-1}).
\end{align}
In what follows, we will adapt our continuous-time proof technique to the discrete case and identify a quantity which follows a momentum mirror descent with time-varying potentials. Our proofs closely follow the proof techniques from \citep{even_pesme_sgd} which considers SGD without momentum.

\subsection{Proof of \Cref{lemma:discrete:I}, \Cref{main_mgd:general} and \Cref{main_mgd}}\label{app:subsec:mgd:positive}

We start by recalling the main-paper results. The first lemma introduces two convergent series which will appear in our main result. 

\lemmaconvergenceS*

The proof of the lemma can be found in the proof of the following main theorem.

\thmdiscrete*

\paragraph{Proving Convergence towards an Interpolator.}
By \Cref{ass:discrete:convergence}, we have that the iterates $w_{\pm, k}$ converge towards limiting weights $w_{\pm, \infty}$ and that the predictors converge towards a vector $\theta^{\texttt{MGF}}$. Taking the limit in \Cref{app:eq:w+-}, we get that $\lim_{k\to\infty} \nabla \L_{\cB_k} (\theta_k) \odot w_{\pm, k} = 0$. By \Cref{ass:discrete:balancedness}, $w_{\pm, \infty}$ have non-zero coordinates. Therefore, 
$\lim_{k\to\infty} \nabla \L_{\cB_k} (\theta_k)= 0$.
For any fixed batch $\cB \subset \{1, \cdots, n\}$, the sampling with or without replacement is such that (almost surely) the set  $M_k \coloneqq \{k \geq 0 , \cB_k = \cB \}$ is infinite. Hence, by continuity of $\nabla L_{\cB}$, $\lim_{k \to \infty, k \in M_k} \nabla L_\cB(\theta_k) = \nabla L_\cB(\theta^\texttt{SMGD})$. Therefore, for all fixed batches $\cB$, $\nabla L_\cB(\theta^\texttt{SMGD}) = 0$ and hence $\theta^\texttt{SMGD}$ interpolates the dataset.

From here on now, for ease of notation, we do the proof for deterministic MGD. The proof for stochastic MGD is exactly the same after replacing $\nabla \L(\theta_k)$ with $\nabla \L_{\cB_k}(\theta_k)$.

\paragraph{Deriving the Momentum Mirror Descent.}
Recall that the set of pairs $(\gamma, \beta)$ such that there exists $k$ where $w_{\pm, k} = 0$ is negligible in $\R^2$. We can hence assume that the iterates are never exactly zero, and we consider the logarithmic reparametrisation of the iterates $w_{\pm, k}$ as 
\[ g_{\pm, k} =
  \begin{cases}
    \ln(w_{\pm, k}), & \text{if } w_{\pm, k} > 0, \\
    \ln(|w_{\pm, k}|) + i \pi, & \text{if } w_{\pm, k} < 0.
  \end{cases}
\]
This way we have that that $w_{\pm, k} = \exp(g_{\pm, k}) $ for all $k$. \Cref{app:eq:w+-} then becomes
\begin{align*}
 \exp( g_{\pm, k+1} ) = \exp( g_{\pm, k}) \mp \g\nabla \L(\t_k) \p \exp( g_{\pm, k}) + \b (\exp( g_{\pm, k}) - \exp( g_{\pm, k-1})).
\end{align*}
Dividing by $\exp( g_{\pm, k})$ yields
\begin{align*}
 \exp( g_{\pm, k+1} - g_{\pm, k} ) = 1 \mp \g\nabla \L(\t_k)  + \b (1 - \exp(- (g_{\pm, k} - g_{\pm, k-1})).
\end{align*}
Now, for $k \geq 1$, let $\delta_{\pm, k} = g_{\pm, k} - g_{\pm, k-1}$ so that we can more compactly write the above recurrence as
\begin{align*}
 \exp( \delta_{\pm, k+1} ) = 1 \mp \g\nabla \L(\t_k)  + \b (1 - \exp(- \delta_{\pm, k}) ).
\end{align*}
The trick, inspired by \cite{even_pesme_sgd}, is to consider the function $q(z) = \exp(z) - (1+z)$ for $z \in \C$. Importantly, note that $q(z) \geq 0$ for $z \in \R$. Using this function, we can now rewrite the recurrence as
\begin{align*}
 \delta_{\pm, k+1} + q(\delta_{\pm, k+1}) = \mp \g\nabla \L(\t_k)  + \b (\delta_{\pm, k} - q(- \delta_{\pm, k})).
\end{align*}
Setting the residues $Q_{\pm, k} \coloneqq q(\delta_{\pm, k+1})  + \b q(- \delta_{\pm, k})$ leads to
\begin{align*}
 \delta_{\pm, k+1} = \b \delta_{\pm, k} \mp \g\nabla \L(\t_k) - Q_{\pm, k}.
\end{align*}
This can be seen as a first-order recurrence relation with variable coefficients. For $\b = 0$ we exactly recover the analysis from \cite{even_pesme_sgd}. For $\b > 0$, since $\delta_{\pm, 1} = 0$, for $m \geq 1$, we can expand the relation as
\begin{align*}
 \delta_{\pm, m+1} = -  \sum_{k=1}^{m} \b^{m-k} \left[ \pm \g\nabla \L(\t_{k}) + Q_{\pm, k}  \right].
\end{align*}
Summing over $m$, we now get for $N \geq 1$ the following expression:
\begin{align*}
 g_{\pm, N+1} - g_{\pm, 1} & = \sum_{m=1}^{N}  \delta_{\pm, m+1} \\
 &=  -\sum_{m=1}^N  \sum_{k=1}^m \b^{m-k} \left [ \pm \g\nabla \L(\t_{k}) + Q_{\pm, k}  \right ] 
\end{align*}
Finally, taking the exponential for $N \geq 1$, we obtain
\begin{align*}
 w_{\pm, N+1} &= w_{\pm, 0} \exp \left( -\sum_{m=1}^N  \sum_{k=1}^m \b^{m-k} \left [ \pm \g\nabla \L(\t_{k}) + Q_{\pm, k}  \right ]  \right) \\ 
 &= w_{\pm, 0} \exp \left( \pm \sum_{m=1}^N  \sum_{k=1}^m \b^{m-k}   Q_{\pm, k}   \right)  \exp \left( \mp \g \sum_{m=1}^N  \sum_{k=1}^m \b^{m-k} \nabla \L(\t_{k})   \right) \\
&= w_{\pm, 0} \exp \left( -\frac{1}{1 - \b} \sum_{m=1}^{N} (1 - \b^{N+1-m}) Q_{\pm, m}  \right) \exp \left( \mp \frac{\g}{1 - \b} \sum_{m=1}^{N} (1 - \b^{N+1-m}) \nabla \L(\t_m) \right),
\end{align*}
where the last equality is obtained by changing the order of summation. 
Following our continuous-time approach, for $N \geq 2$, we define $\a_{\pm, N+1}$ as 
\begin{align}\label{app:def:alphaN}
\begin{split}
    \a_{\pm, N+1} 
    &\coloneqq w_{\pm, 0} \exp \left( \pm \sum_{m=1}^N  \sum_{k=1}^m \b^{m-k}   Q_{\pm, k}   \right)  \\
    &= w_{\pm, 0} \exp \left( -\frac{1}{1 - \b} \sum_{m=1}^{N} (1 - \b^{N+1-m}) Q_{\pm, m}  \right).
\end{split}
\end{align}
We can now write the iterates $w_{\pm, k}$ as
\begin{align*}
    w_{\pm, N+1} = \a_{\pm, N+1}  \exp \big( \mp \g \sum_{m=1}^{N} \sum_{k=1}^m \b^{m-k}   \nabla \L(\t_k) \big).
\end{align*}
Thus, the regression parameter $\theta_N$ becomes
\begin{align*}
    \theta_{N+1} &= \frac{1}{4} (w_{+,N+1}^2 - w_{-,N+1}^2) \\
    &= \frac{1}{4}  \a_{+, N+1}^2  \exp \big( - 2 \g \sum_{m=1}^{N} \sum_{k=1}^m \b^{m-k}   \nabla \L(\t_k) \big)  - \frac{1}{4}  \a_{-, N+1}^2  \exp \big( 2 \g  \sum_{m=1}^{N} \sum_{k=1}^m \b^{m-k}   \nabla \L(\t_k) \big) \\
    &= \frac{1}{2}  \Delta_{N+1} \  \sinh \left( -2 \g  \sum_{m=1}^{N} \sum_{k=1}^m \b^{m-k}   \nabla \L(\t_k) +  \arcsinh \left( \frac{\a_{+,N+1}^2 - \a_{-,N+1}^2 }{2 \Delta_{N+1} } \right) \right),
\end{align*}
where we recall that $\D_{N} = |w_{+, N}w_{-, N}| = |\a_{+, N}\a_{-, N}|$.
Hence, similar to the continuous case,
\begin{align*}
  \frac{1}{2} \arcsinh \left( \frac{2 \t_{N+1}}{\D_{N+1}} \right) - \frac{1}{2} \arcsinh \left( \frac{\a_{+,N+1}^2 - \a_{-,N+1}^2 }{2\D_{N+1}} \right) = - \g \sum_{m=1}^{N} \sum_{k=1}^m \b^{m-k}   \nabla \L(\t_k).
\end{align*}
For $N \geq 1$, the above identity becomes exactly
\begin{align}\label{app:eq:nablaPhiN}
  \nabla \Phi_{N+1}(\theta_{N+1}) = - \g \sum_{m=1}^{N} \sum_{k=1}^m \b^{m-k}   \nabla \L(\t_k),
\end{align}
where the time-varying potential $\Phi_N : \R^d \to \R$ is defined as
\begin{align*}
    \Phi_N(\theta) &= \frac{1}{4} \sum_{i=1}^d \left( 2\t_i \arcsinh \left( \frac{2\t_i}{\D_{N, i}} \right) - \sqrt{4\t_i^2 + \D_{N, i}^2} + \D_{N, i} \right) + \ip{\phi_N}{\t} \\
        & = \psi_{\D_N}(\t) + \ip{\phi_N}{\t},
\end{align*} 
where $\phi_N = \frac{1}{2}\arcsinh\left( \frac{\a_{+, N}^2 - \a_{-, N}^2}{2 \D_N } \right)$ and $\psi_{\D_N}$ is the hyperbolic entropy defined in \Cref{hyperbolic_entropy}. Notice that with this definition we arrive at the following time-varying momentum mirror descent for $N \geq 1$:
\begin{equation} \label{MGDmirror}
    \nabla \Phi_{N+1}(\theta_{N+1}) = \nabla \Phi_N(\theta_N) - \g \nabla \L(\t_N) + \b(\nabla \Phi_N(\theta_N) - \nabla \Phi_{N-1}(\theta_{N-1})).
\end{equation}

\myparagraph{Convergent Quantities.} From \Cref{app:lemma:convergence_alpha}, we have that $\alpha_{\pm, N}$ must converge and that the limiting vectors $\alpha_{\pm, \infty}$ have non-zero coordinates. Therefore, the series $\sum_{m=1}^\infty  \sum_{k=1}^m \b^{m-k}   Q_{\pm, k} $ are convergent and their terms must hence converge to zero:  $\sum_{k=1}^m \b^{m-k}   Q_{\pm, k} \underset{m \to \infty}{\longrightarrow} 0$. Therefore,
\begin{align*}
\a_{\pm, N}  \to \a_{\pm, \infty}  = w_{\pm, 0} \exp \left( -\frac{1}{1 - \b} \sum_{m=1}^{\infty}  Q_{\pm, m}  \right).
\end{align*}

We now develop the formulas for $Q_{\pm, m}$ in order to arrive at the sums $S_\pm$ from \Cref{lemma:discrete:I}.
Recall that for $m \geq 1$, $Q_{\pm, m} = q(\delta_{\pm, m+1})  + \b q(- \delta_{\pm, m})$ and $\delta_{\pm, 1} = q(\delta_{\pm, 1}) = 0$. Therefore,
\begin{align*}
 \sum_{m=1}^{\infty}  Q_{\pm, m}  &= \sum_{m=1}^{\infty}  q(\delta_{\pm, m+1})  + \b q(- \delta_{\pm, m}) \\
 &= \sum_{m=1}^{\infty}  q(\delta_{\pm, m+1})  + \b q(- \delta_{\pm, m+1}).
\end{align*}
Since $\delta_{\pm,m+1} = g_{\pm, m+1} - g_{\pm, m}$, we have 
\[ \delta_{\pm, m+1} =
  \begin{cases}
    \ln\left(\frac{w_{\pm, m+1}}{w_{\pm, m}}\right) & \text{if } w_{\pm, m+1} \text{ and } w_{\pm, m} \text{ have the same sign}, \\
    \ln\left(\left|\frac{w_{\pm, m+1}}{w_{\pm, m}}\right|\right) + \sgn(w_{\pm, m}) i \pi & \text{if they have different signs}.
  \end{cases}
\]
It remains to notice that since $q(z) = \exp(z) - (1 + z)$, we get that 
\begin{align*}
&q(\ln(z)) =  (z - 1) - \ln(z)  \ & \text{for } z \in \mathbb{R}_{> 0}, \\
&q(\ln(|z|) \pm i \pi ) =   (z - 1) - (\ln(|z|) \pm i \pi) \ & \text{for } z \in \mathbb{R}_{< 0}.
\end{align*}
Therefore letting $r(z) = (z - 1) - \ln(|z|)$ as in \Cref{lemma:discrete:I}, we get
\begin{align*}
q(\delta_{\pm, m+1}) &= r \Big (\frac{w_{\pm, m+1}}{w_{\pm, m}} \Big ) - \xi_{\pm, m} \sgn(w_{\pm, m}) i \pi \\
q(- \delta_{\pm, m+1}) &= r \Big (\frac{w_{\pm, m}}{w_{\pm, m+1}} \Big ) + \xi_{\pm, m} \sgn(w_{\pm, m}) i \pi,
\end{align*}
where $\xi_{\pm, m} = 0$ if $\sgn(w_{\pm, m+1}) = \sgn(w_{\pm, m})$ and $1$ otherwise. This leads to 
\begin{align*}
\frac{1}{1 - \b} \sum_{m=1}^\infty Q_{\pm, m}   
 &=  \frac{1}{1 - \b} \sum_{m=1}^\infty \big [ r \Big (\frac{w_{\pm, m+1}}{w_{\pm, m}} \Big )  + \b r \Big (\frac{w_{\pm, m}}{w_{\pm, m+1}} \Big ) \Big ]  -   \sum_{m=1}^\infty \xi_{\pm, m} \sgn(w_{\pm, m}) i \pi  \\ 
 &= S_{\pm} -   \sum_{m=1}^\infty \xi_{\pm, m} \sgn(w_{\pm, m}) i \pi < \infty.
\end{align*}
The last equality is due to the definition of $S_\pm$ from \Cref{lemma:discrete:I}, and the last inequality is due to the summability of $(Q_{\pm, m})_m$. This therefore proves lemma \Cref{lemma:discrete:I}. Now notice that
\begin{align*}
\a_{\pm, \infty}^2  = w_{\pm, 0}^2 \exp \left( - 2 S_\pm\right).
\end{align*} 
Since $\Delta_\infty =  | \alpha_{+, \infty} \alpha_{-, \infty}  |$, we finally get that
 \begin{equation*}
\D_\infty  = \D_0 \odot \exp \Big( - (S_{+} + S_{-})  \Big).
\end{equation*}

\paragraph{Implicit Regularisation Problem.}
Notice that
\begin{align*}
  \nabla \Phi_{N+1}(\theta_{N+1}) = - \g \sum_{m=1}^{N} \sum_{k=1}^m \b^{m-k}   \nabla \L(\t_k) \in \texttt{span}(x_1, \cdots, x_n).
\end{align*}
Let $\Phi_{\infty}(\theta) \coloneqq \psi_{\Delta_\infty}(\theta) + \langle \phi_\infty, \theta \rangle$ and consider
\begin{align*}
    \nabla \Phi_{\infty}(\theta^{\texttt{MGD}}) = (\nabla \Phi_{\infty}(\theta^{\texttt{MGD}})  - \nabla \Phi_{\infty}(\theta_N)) +  (\nabla \Phi_{\infty}(\theta_N) - \nabla \Phi_{N}(\theta_N)) + \nabla \Phi_{N}(\theta_N).
\end{align*}
The first two terms converge to $0$: the first due to the convergence $\theta_N \to \theta^{\texttt{MGD}}$ and the second due to the uniform convergence of $\nabla \Phi_N$ to $\nabla \Phi_\infty$ on compact sets. The last term is in  $\texttt{span}(x_1, \cdots, x_n)$ for all $N$. Therefore, we get that $\nabla \Phi_\infty(\theta_\infty) \in \texttt{span}(x_1, \cdots, x_n)$, and following the exact same proof as in the continuous-time framework, we finally get that 
\begin{align*}
    \theta^{\texttt{MGD}} = \argmin_{\t^\star \in \mathcal{S}} D_{\psi_{\Delta_\infty}}(\t^\star, \tilde{\t}_0)
\end{align*}
where 
\begin{align*}
\tilde{\t}_0 &= \frac{\a_{+, \infty}^2 - \a_{-, \infty}^2}{4} \\ 
&= \frac{1}{4} \Big( w_{+,0}^2 \odot \exp\left( -2 S_+   \right) - w_{-,0}^2 \odot \exp\left( - 2 S_-   \right) \Big).
\end{align*}

\qed

We recall and prove the following corollary.

\MainMGD*

\begin{proof}
    The corollary follows from the fact that if the iterates $w_{\pm, k}$ do not change sign, then since $r(z) \geq 0$ for $z > 0$, we get that $S_{\pm} > 0$ and $\Delta_\infty < \Delta_0$. Furthermore, $| \tilde{\theta}_0 | < \max( w_{+,0}^2 ,  w_{-,0}^2 ) / 4  \leq \alpha^2$
\end{proof}

\subsection{Link to the Continuous-Time Result.}  
In this subsection we link our continuous results with the discrete when the iterates do not cross zero. Indeed, at first sight, the discrete-time expression for $\D_\infty$ might seem quite different from its continuous-time counterpart:
\small
\begin{align*}
    \D_\infty^{\texttt{MGD}} & = \Delta_0 \exp \left( -\frac{1}{1 - \b} \sum_{k=1}^{\infty} \left[ r \left( \frac{w_{+, k+1}}{w_{+, k} } \right) + r \left(  \frac{w_{-, k+1}}{w_{-, k} } \right) \right] + \b \left[ r \left (\frac{w_{+, k}}{w_{+, k+1}}  \right) + r \left ( \frac{w_{-, k}}{w_{-, k+1}}  \right) \right] \right)\\
    \D_\infty^{\texttt{MGF}} & = \Delta_0 \exp\left( -\l \int_0^\infty \left(\frac{\dot{w}_{+, t}}{w_{+, t}}\right)^2 + \left(\frac{\dot{w}_{-, t}}{w_{-, t}}\right)^2 \dd t \right).
\end{align*}
\normalsize
However, upon closer inspection, by letting the discretisation step $\eps = \sqrt{\l\g} = \frac{\g}{(1-\b)}$ from \Cref{prop:discretisation} go to 0, we can recover the continuous-time result. Indeed, as $\eps \to 0$, we expect successive iterates $w_{\pm, k}$ to be close and hence  $w_{\pm, k+1} / w_{\pm, k} \approx 1$. Now, since $r(z) \sim_{z \to 1} (z - 1)^2 / 2$, we roughly have
\begin{align*}
    r \Big( \frac{w_{\pm, k+1}}{w_{\pm, k}} \Big) &\approx \frac{1}{2} \big( \frac{w_{\pm, k+1}  - w_{\pm, k}}{w_{\pm, k}}  \big)^2
\end{align*}
and
\begin{align*}
    r \Big( \frac{w_{\pm, k}}{w_{\pm, k+1}} \Big) &\approx \frac{1}{2} \big( \frac{w_{\pm, k+1}  - w_{\pm, k}}{w_{\pm, k+1}}  \big)^2
    \approx \frac{1}{2} \big( \frac{w_{\pm, k+1}  - w_{\pm, k}}{w_{\pm, k}}  \big)^2
\end{align*}
Putting the approximations together:
\begin{align*}
    \frac{1}{1-\b} \sum_k \Big[ r \Big(  \frac{w_{\pm, k+1}}{w_{\pm, k} }  \Big) + \b r \Big( \frac{w_{\pm, k}}{w_{\pm, k+1}}  \Big) \Big] 
    &\approx \frac{1}{2} \frac{\eps (1 + \b) }{1-\b} \sum_k \left( \frac{w_{\pm, k+1}  - w_{\pm, k}}{\varepsilon}  \right)^2 \frac{1}{(w_{\pm, k})^2} \cdot \eps \\
    &\approx \frac{1+ \b}{2} \frac{\g}{(1-\b)^2}\int_0^\infty \left(  \frac{\dot{w}_{\pm, t}}{w_{\pm, t}} \right )^2 \dd t \\
    &= \frac{1+ \b}{2} \l \int_0^\infty \left(  \frac{\dot{w}_{\pm, t}}{w_{\pm, t}} \right )^2 \dd t.
\end{align*}
Notice that in order for $\l$ to remain constant and $\eps$ to go to 0, we must both have $\g \to 0$ and $\b \to 1$. Hence, $(1+\b)/2 \to 1$, and we recover the continuous-time expression for the balancedness.

However, note that when the iterates cross zero it is unclear to the authors how the continuous formula and its discrete counterpart compare.

\paragraph{Another Safe-Check Computation.}
Recall that MGD with stepsize $\gamma$ and momentum parameter $\beta$ corresponds to the discretisation of MGF with $\lambda = \gamma / (1-\beta)^2$ and discretisation step $\varepsilon = \sqrt{\lambda \gamma}$. 
To check the consistency between the discrete time equations and continuous time equations, we look at the value of $\exp( - \frac{t - s}{\lambda})$ and times '$t = m \varepsilon$' and '$s = k \varepsilon$': 
\begin{align*}
     \exp( - \frac{t - s}{\lambda})   &= \exp( - \frac{(m - k) \varepsilon}{\lambda})  \\
    &= \exp( - (m-k) (1 - \beta)) \\ 
    &= [\exp(\beta - 1)]^{m-k } \\ 
    & \sim_{\beta \to 1}  \beta^{m - k}.
\end{align*}
This small computation serves as a safe-check, affirming the correspondence between the continuous-time analysis expression $\exp( - \frac{t - s}{\lambda})$ and its discrete-time counterpart $\beta^{m - k}$.

\section{Technical Lemmas}\label{technical_lemmas}

In this section we present various technical lemmas which allow us to prove our main results.
For $\D \in \R_{> 0}^d$, we recall the definition of the hyperbolic entropy function \citep{pmlr-v117-ghai20a} $\psi_{\D} : \R^d \to \R$  at scale $\D$:
\begin{equation*}
\psi_\D(\t) = \frac{1}{4} \sum_{i=1}^d \left( 2\t_i \arcsinh \left( \frac{2\t_i}{\D_i} \right) - \sqrt{4\t_i^2 + \D_i^2} + \D_i \right).
\end{equation*}

The following lemma shows that the potential behaves as the $\ell_1$-norm as $\D$ approaches $0$.

\begin{lemma}\label{entropy_asymptotics}
    For $\t \in \R^d$ the following asymptotic equivalence holds:
    \begin{equation*}
        \psi_\D(\t) \underset{\scriptscriptstyle \D \to 0}{\sim} \frac{1}{4} \sum_{i=1}^d \ln\left(\frac{1}{\D_i}\right) |\t_i|.
    \end{equation*}
\end{lemma}

\begin{proof}
    The lemma easily follows from the asymptotic convergence $$\arcsinh(x) \underset{\scriptscriptstyle |x| \to \infty}{\sim} \sgn(x) \ln|x|.$$ 
\end{proof}

The following lemma is a classical result which gives a closed-form expression to the solution of a first order ODE.

\begin{lemma}\label{convolution}
    Let $f : \R_{\geq 0} \to \R^d$ be a differentiable function and let $g : \R_{\geq 0} \to \R^d$ be a continuous function such that for some $\l \neq 0$, $$\l \dot{f} + f + g = 0, \en \forall t \in \R_{\geq 0}.$$ Then,
    \begin{equation*}
        f(t) = f(0) e^{-\frac{t}{\l}} -\frac{1}{\l} \int_0^t g(s) e^{-\frac{(t-s)}{\l}} ds.
    \end{equation*}
    Moreover, we have the following formula for the integral of $f(t)$:
    \begin{equation*}
        \int_0^T f(t) dt = \l f(0) (1 - e^{-\frac{T}{\l}}) - \int_0^T g(t) (1 - e^{-\frac{(T-t)}{\l}}) dt.
    \end{equation*}
\end{lemma}

\begin{proof}
    If we integrate the identity $\frac{d}{dt}\left[f(t)e^{t/\l}\right] = -\frac{1}{\l}g(t)e^{t/\l}$, we get that
    \begin{equation*}
        f(t) = f(0) e^{-\frac{t}{\l}} - \frac{1}{\l} \int_0^t g(s) e^{-\frac{(t-s)}{\l}} ds.
    \end{equation*}
    As for the second part of the lemma, notice that
    \begin{equation*}
        \int_0^T f(t) dt = \int_0^T \left[ f(0) e^{-\frac{t}{\l}} -\frac{1}{\l} \int_0^t g(s) e^{-\frac{(t-s)}{\l}} ds \right] dt.
    \end{equation*}
    Hence, using Fubini, we get
    \begin{align*}
        \int_0^T \int_0^t g(s) e^{-\frac{(t-s)}{\l}} ds dt & = \int_0^T \int_0^T g(s) \mathbf{1}_{s \leq t}(s,t) e^{-\frac{(t-s)}{\l}} ds dt \\
        & = \int_0^T g(s) \int_0^T  \mathbf{1}_{s \leq t}(s,t) e^{-\frac{(t-s)}{\l}} dt ds \\
       & = \int_0^T g(s) \int_s^T e^{-\frac{(t-s)}{\l}} dt ds \\
       & = \int_0^T g(s) \l(1 - e^{-\frac{(T-s)}{\l}}) ds,
    \end{align*}
    which concludes the proof of the lemma.
\end{proof}

The following lemma gives various properties on integrability and convergence of the solution $f$ of the aforementioned ODE.

\begin{lemma}\label{regularity}
    Let $f : \R_{\geq 0} \to \R^d$ be a differentiable function such that $f(0) = 0$ and let $g : \R_{\geq 0} \to \R^d$ be a continuous function such that for some $\l \neq 0$,
    \begin{equation*}
        \l \dot{f} + f + g = 0, \en \forall t \in \R_{\geq 0}.
    \end{equation*}
    If $g \in L^\infty(0, +\infty)$, then $f \in L^\infty(0, +\infty)$ and $\n{f}_\infty \leq \n{g}_\infty$. Moreover, if $g \in L^1(0, +\infty)$, then the following hold:
    \begin{itemize}
        \item $f \in L^1(0, +\infty)$ and $\int_0^t |f(s)| ds \leq \int_0^t |g(s)| ds, \en \forall t \in [0, +\infty]$;
        \item $\ds \lim_{t \to \infty} f(t) = 0$;
        \item $\ds \int_0^\infty f = - \int_0^\infty g$.
    \end{itemize}
\end{lemma}

\begin{proof}
    First, assume $g \in L^\infty(0, \infty)$. From Lemma \ref{convolution}, we have that $f(t) = -\frac{1}{\l} \int_0^t g(s) e^{-\frac{(t-s)}{\l}} ds$. Hence,
    \begin{align*}
        |f(t)| & \leq \frac{\n{g}_\infty}{\l} \int_0^t e^{-\frac{(t-s)}{\l}} ds \\
        & = \n{g}_\infty (1 - e^{-t/\l}) \leq \n{g}_\infty,
    \end{align*} which proves the first assertion.
    
    Second, assume $g \in L^1(0, \infty)$. Then, $|f(t)| \leq \frac{1}{\l} \int_0^t |g(s)| e^{-\frac{(t-s)}{\l}} ds$. Therefore,
    \begin{align*}
        \int_0^t |f(s)| ds & \leq \int_0^t |g(s)| (1-e^{-\frac{(t-s)}{\l}}) ds \\
        & \leq \int_0^t |g(s)| ds \leq \n{g}_{L^1}.
    \end{align*}
    
    Moving on, we will show that $\lim_{t \to \infty} f(t) = 0$. Recall that $f(t) = -\frac{1}{\l} \int_0^t g(s) e^{-\frac{(t-s)}{\l}} ds$. Then, 
    \begin{align*}
        \left| \int_0^t g(s) e^{-\frac{(t-s)}{\l}} ds \right| & = \left| \int_0^{t/2} g(s) e^{-\frac{(t-s)}{\l}} ds + \int_{t/2}^t g(s) e^{-\frac{(t-s)}{\l}} ds \right| \\
        & \leq e^{-\frac{t}{2\l}} \int_0^{t/2} |g(s)| ds + \int_{t/2}^\infty |g(s)| ds \\
        & \xrightarrow{t \to \infty} 0.
    \end{align*}
    Finally, notice that
    \begin{align*}
        & \lim_{t \to \infty} \left[ \l \int_0^t \dot{f} + \int_0^t (f + g) \right] = 0 \iff \\
        & \l \lim_{t \to \infty} f(t) + \int_0^\infty (f + g) = 0 \iff \\
        & \int_0^\infty f + \int_0^\infty g = 0,
    \end{align*} where we used that $\lim_{t \to \infty} f(t) = 0$ and the linearity of the Lebesgue integral.
\end{proof}

With the help of \Cref{convolution} and \Cref{regularity}, we can finally prove \Cref{small_lambda}, which considers ODE \eqref{MGF:diag} and establishes the positivity of the balancedness for small $\l$.

\SmallLambda*

\begin{proof}
    We consider \ref{mgf:lambda} with the diagonal-linear-network loss $F(w) = \L(u \p v)$, where $w = (u, v)$. From the energy of the system, defined in \Cref{energy} as $E_t = F(w_t) + \frac{\l}{2} \n{\dot{w}_t}_2^2$ with derivative $\dot{E}_t = -\n{\dot{w}_t}_2^2$, we get that
    \begin{equation*}
        \L(\t_t) = \frac{\n{y}_2^2}{2n} - \frac{\l}{2} \n{\dot{w}_t}_2^2 - \int_0^t \n{\dot{w}_s}_2^2 ds.
    \end{equation*}
    Hence, since the LHS of the above equation is nonnegative, we get
    \begin{equation*}
        \int_0^\infty \n{\dot{w_t}}^2 dt \leq \frac{\n{y}^2}{2n}.
    \end{equation*}
    Therefore,
    \begin{equation*}
        \int_0^\infty |\dot{u}_t^2 - \dot{v}_t^2| dt < \frac{\n{y}^2}{2n} \mathbf{1}.
    \end{equation*}
    Consequently, $\dot{u}_t^2 - \dot{v}_t^2 \in L^1(0, \infty)$.
    Now, notice that from ODE \eqref{MGF:diag}, we obtain
    \begin{align*}
    & \l(\ddot{u}_tu_t - \ddot{v}_tv_t) + (\dot{u}_tu_t - \dot{v}_tv_t) = 0 \iff \\
    & \l \frac{d}{dt}(\dot{u}_tu_t - \dot{v}_tv_t) + (\dot{u}_tu_t - \dot{v}_tv_t) - \l(\dot{u}_t^2 - \dot{v}_t^2) = 0.
    \end{align*} Applying Lemma \ref{convolution} yields
    \begin{equation*}
        \dot{u}_tu_t - \dot{v}_tv_t = \int_0^t (\dot{u}_s^2 - \dot{v}_s^2) e^{-\frac{(t-s)}{\l}} ds
    \end{equation*}
    and
    \begin{equation} \label{delta_evolution}
    u_t^2 - v_t^2 = \D_0 + 2\l \int_0^t (\dot{u}_s^2 - \dot{v}_s^2)(1 - e^{-\frac{(t-s)}{\l}}) ds.
    \end{equation}
    Applying Lemma \ref{regularity} allows us to conclude that for every $t \in [0, +\infty]$,
    \begin{align*}
        \D_t & \geq \D_0 - 2\l \int_0^t |\dot{u}_s^2 - \dot{v}_s^2| ds \\
        & > \D_0 - \frac{\l \n{y}_2^2}{n} \mathbf{1} \geq 0,
    \end{align*}
where the last inequality is due to the inequality assumption over $\lambda$.
\end{proof}

Our final technical lemma helps with the proof of \Cref{main_mgd:general}. The definition of the quantities $Q_{\pm, m}$ can be found in the proof of this theorem.
\begin{lemma}\label{app:lemma:convergence_alpha}
The quantities $\alpha_{\pm, N}$ defined in \cref{app:def:alphaN}:
\begin{equation*}
\a_{\pm, N+1} = \a \exp \left( -\frac{1}{1 - \b} \sum_{m=1}^{N} (1 - \b^{N+1-m}) Q_{\pm, m}  \right),
\end{equation*}
converge as $N \to \infty$ to vectors $\alpha_{\pm, \infty}$ with non-zero coordinates.
\end{lemma}

\begin{proof}
    
From \Cref{ass:discrete:convergence} and \Cref{ass:discrete:balancedness}, we have that the iterates $w_{\pm, N}$ converge towards vectors $w_{\pm, \infty}$ such that $\Delta_\infty = |w_{+, \infty} \odot w_{-, \infty}|$ has non-zero coordinates. This means that there exists $N_0 > 0$ such that $w_{\pm, N}$ do not change sign for $N \geq N_0$. Consequently, the imaginary parts of $g_{\pm, N}$ are constant (equal to $0$ or $\pi$ depending on the sign of $w_{\pm, \infty}$) for $N \geq N_0$, and $\delta_{\pm, N} \in \R$ for $N \geq N_0$. This finally means that $Q_{\pm, N} \geq 0$ for $N \geq N_0$ and 
\begin{align*}
     \sum_{m=1}^{N} (1 - \b^{N+1-m}) Q_{\pm, m}  = \sum_{m=1}^{N_0} (1 - \b^{N+1-m}) Q_{\pm, m} +  \sum_{m=N_0 + 1}^{N} (1 - \b^{N+1-m}) Q_{\pm, m} 
\end{align*}
The first term converges to 
$\sum_{m=1}^{N_0} Q_{\pm, m}$ as $N \to \infty$. The second term is increasing because $Q_{\pm, N}$ are positive for $N \geq N_0$ and $(1 - \beta^{N+1 - m})$ is increasing. Therefore, the second term also converges to a finite value since otherwise $\alpha_{\pm, \infty} = 0$, which contradicts $ \D_\infty = |\alpha_{+, \infty} \alpha_{-, \infty}| \neq 0$.
\end{proof}

\section{Additional Experiments}\label{experiments}

In this section of the appendix, we clarify experimental details and discuss additional experiments.

\subsection{MGF: A Good Continuous Surrogate}

Most of our experiments deal with 2-layer diagonal networks, but before we constrain ourselves to that tractable setting, we present a couple of experiments on more general architectures. These experiments highlight our observation from \Cref{sec:prelim} that \ref{mgf:lambda} serves as a good continuous proxy for \ref{mgd} even for complicated non-convex losses $F$ and large step sizes $\g$. We provide evidence for that conclusion by showing that the single parameter $\l = \g/(1-\b)^2$ controls the generalisation performance of models trained with \ref{mgd}.

\myparagraph{Teacher-Student Fully Connected Network.}
We detail the experimental setting which leads to \Cref{fig:teacher}. We consider a teacher-student setup where the teacher is a one-hidden-layer fully-connected ReLU network with $5$ hidden neurons and the student is a one-hidden-layer fully-connected ReLU network with $20$ hidden neurons. We randomly generate $15$ inputs $x_i \in \R^2$ according to a standard multivariate normal distribution. Each $y_i$ corresponds to the output by the teacher network on input $x_i$. The student is trained using momentum gradient descent with a square loss. \Cref{fig:teacher} corresponds to the test loss after the student reaches $10^{-5}$ training error. Each grid point corresponds to the same data set and initialisation of the student network. We observe that the quantity $\lambda = \frac{\gamma}{(1-\beta)^2}$ aligns well with the level lines of the test loss as expected from \Cref{prop:discretisation}.

\myparagraph{Deep Linear Network.}
The network used for \Cref{fig:app:deep_linear_network} contains 5 layers with widths (30, 60, 120, 60, 1) and was trained for 1000 epochs for each pair of momentum parameter $\b$ and step size $\g$. Each network weight was randomly initialised according to $\mathcal{N}(0, 0.1^2)$ with fixed randomness for each $(\g, \b)$-trial. The training data was chosen as follows: $(x_i)_{i=1}^n \stackrel{\scriptscriptstyle \text{i.i.d.}}{\sim} \mathcal{N}(\mu \mathbf{1}, \s^2 I_d)$ and $y_i = \ip{x_i}{\t_s^\star}$ for $i \in [n]$ where $\t_s^\star$ is $s$-sparse with nonzero entries equal to $1/\sqrt{s}$, where $(n, d, s) = (20, 30, 5)$ and $(\mu, \s) = (1, 1)$. We show results averaged over 5 replications.

\begin{figure}[h!]
    \centering
    \includegraphics[width=1\textwidth]{{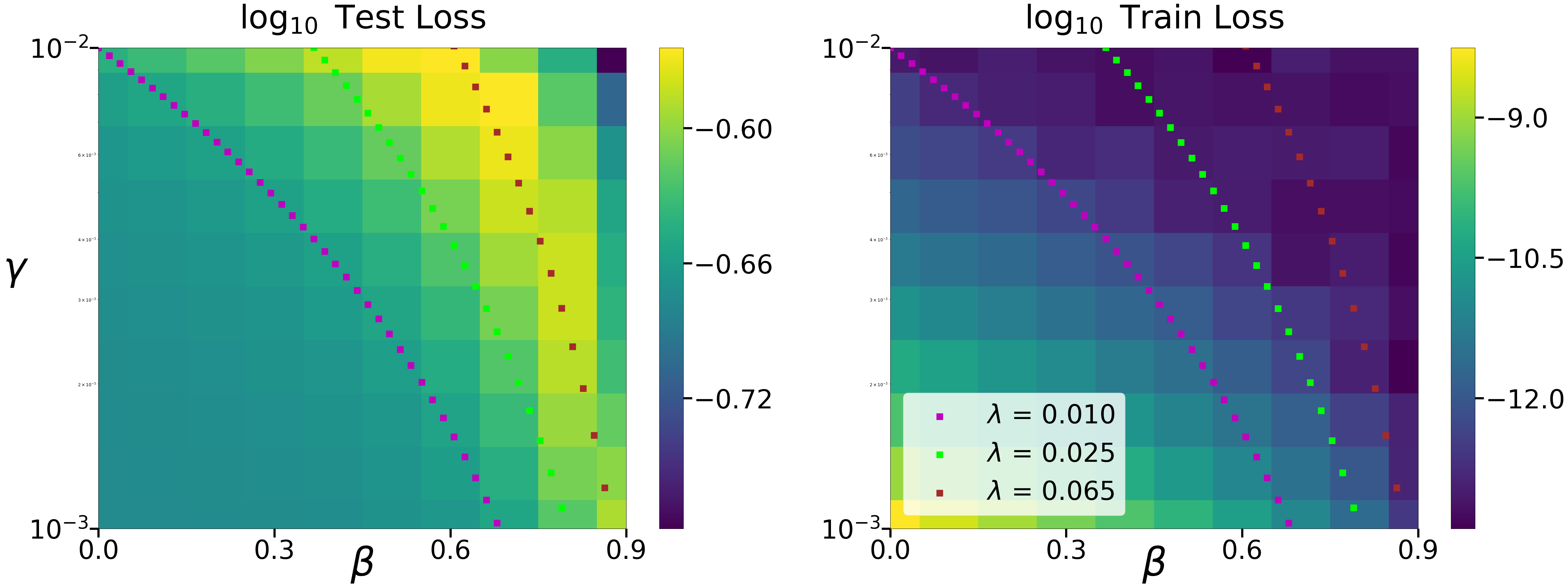}}
    \vspace{-3mm}
    \caption{
    Test and train loss of a fully connected deep linear network trained with \ref{mgd} in a noiseless sparse overparametrised regression setting. The test loss appears considerably correlated with the intrinsic parameter $\l = \g/(1-\b)^2$, evincing that \ref{mgf:lambda} approximates \ref{mgd} sufficiently well even on complex architectures.
    } 
    \label{fig:app:deep_linear_network}
\end{figure}

\myparagraph{2-Layer Diagonal Linear Network.} The plots from \Cref{fig:app:mgd_mgf} were obtained for a 2-layer diagonal linear network trained in the noiseless sparse overparametrised regression setting described above. The first network layer was initialised with the uniform initialisation $\a \mathbf{1}$, where $\a = 0.01$, and the weights of the second layer were set to 0. The momentum gradient flow evolution of the weights was simulated with the default version of the ODE solver $\texttt{scipy.integrate.odeint}$.

\begin{figure}[h!]
    \centering
    \includegraphics[trim=500 0 550 0, clip,width=1\textwidth]{{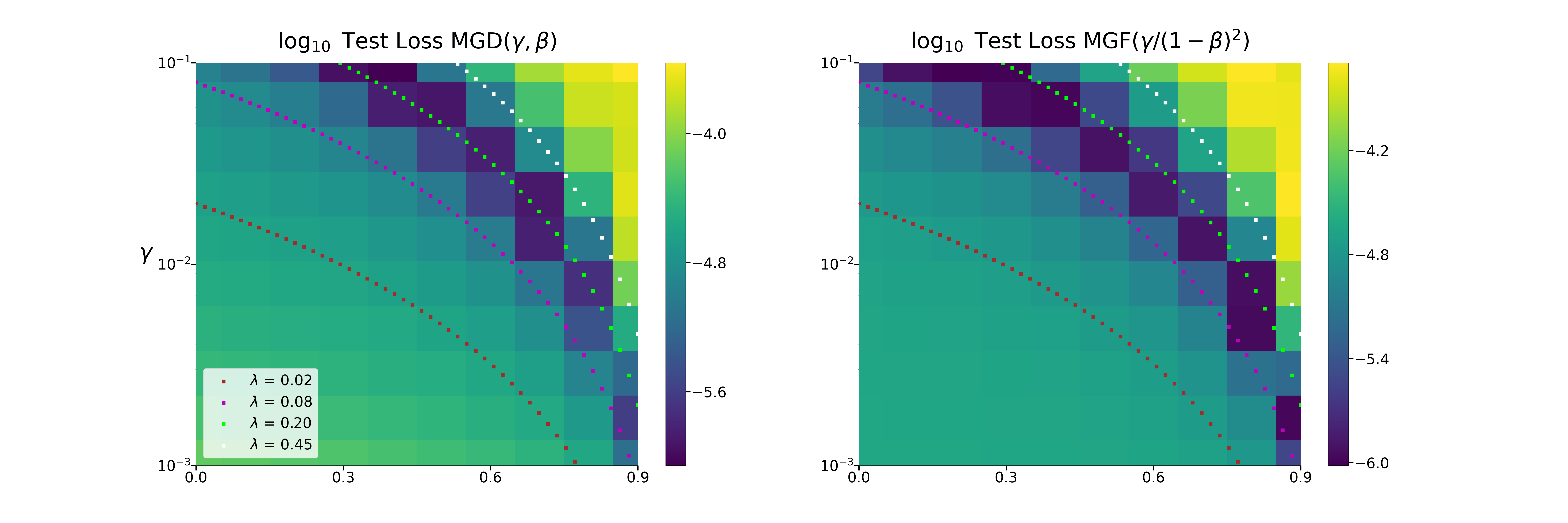}}
    \vspace{-3mm}
    \caption{
    \emph{Left}: Decimal logarithm of the test loss of a 2-layer diagonal linear network trained with \ref{mgd} for 1 million epochs.
    \emph{Right}: Decimal logarithm of the test loss of a 2-layer diagonal linear whose weights evolved according to \ref{mgf:lambda} -- where  $\l = \g/(1-\b)^2$ -- and converged to an interpolator of the training dataset. We observe an almost one-to-one correspondence in terms of generalisation capacity, which demonstrates that \ref{mgf:lambda} serves as a suitable continuous surrogate for \ref{mgd} in the diagonal linear setting. 
    } 
    \label{fig:app:mgd_mgf}
\end{figure}

\subsection{Experiments with Diagonal Linear Networks}

Having seen empirical proof that \ref{mgf:lambda} approximates well the optimisation trajectory of \ref{mgd} on complicated models, we proceed with experiments that illustrate the conclusions of our results for 2-layer diagonal linear networks. In particular, we provide experimental evidence that both in the continuous and discrete-time cases, the recovered interpolators by MGD and MGF satisfy
\begin{equation*}
    \t^{\texttt{MGF/MGD}} = \argmin_{\t^\star \in \mathcal{S}} D_{\psi_{\Delta_\infty}}(\t^\star, \tilde{\theta}_0) \approx \argmin_{\t^\star \in \mathcal{S}} \psi_{\Delta_\infty}(\t^\star),
\end{equation*}
as we explain underneath \Cref{main_mgf:general}, \Cref{main_mgd:general}, and in \Cref{consequences}. Indeed, we observe that the perturbation term $\tilde{\t}_0$ can be safely ignored even without the assumption of strictly positive balancedness. The asymptotic balancedness $\D_\infty$ then uniquely controls the properties of the recovered solution. 
We now specify our experimental setting.

\myparagraph{Experimental Details.} We work in the noiseless sparse overparametrised regression setting with uncentered data. More precisely, we let $(x_i)_{i=1}^n \stackrel{\scriptscriptstyle \text{i.i.d.}}{\sim} \mathcal{N}(\mu \mathbf{1}, \s^2 I_d)$ and $y_i = \ip{x_i}{\t_s^\star}$ for $i \in [n]$ where $\t_s^\star$ is $s$-sparse with nonzero entries equal to $1/\sqrt{s}$. We train a 2-layer diagonal linear network with (M)GD and (M)GF with the uniform initialisation $u_0 = \a \mathbf{1}$, where $\a = 0.01$ and $v_0 = 0$. In order to simulate gradient flow or momentum gradient flow on the network weights, we use the vanilla version of the ODE solver $\texttt{scipy.integrate.odeint}$. For most of the incoming plots, we have fixed $(n, d, s, \s) = (20, 30, 5, 1)$ and we let $\mu \in \{0, 0.5, 1, 1.5\}$. In what follows, all plots show results averaged over 5 replications.

\subsubsection{Continuous-Time Plots}

We first present a set of 3 continuous-time plots (\Cref{fig:app:mgf_triple_1}) for the setting where the input data follows a Gaussian distribution $\mathcal{N}(\mu \mathbf{1}, I_d)$ with $\mu = 1$.

\myparagraph{Experimental Setup.}
For a sampled dataset $(X, y)$, we train our diagonal network with \ref{mgf:lambda}, $\l \in [0,1]$, and initialisation $(u_0, v_0) = (\a \cdot \mathbf{1}, 0)$ until convergence to an interpolator \footnote{We know that $\t^{\texttt{MGF}}$ interpolates the dataset $(X, y)$ because we also record the \textbf{Train Loss} ($\t^{\texttt{MGF}}$), which falls under~$10^{-20}$.} $\t^{\texttt{MGF}}$. During the training of \ref{mgf:lambda}, we also take note of whether the balancedness $\D_t$ remains strictly positive at all times, thereby checking the explanatory range of \Cref{small_lambda_regime}.
Having completed the MGF training, we plot the \textbf{Test Loss} of $\t^{\texttt{MGF}}$, the $\ell_2$\textbf{-Norm of} $\D_\infty$, and the $\ell_1$\textbf{-Norm of} $\t^{\texttt{MGF}}$ in order to visualise the gain in generalisation performance.

\myparagraph{Insignificance of $\tilde{\theta}_0$.} Now, recall from \Cref{main_mgf:general} that $\t^{\texttt{MGF}} = \argmin_{\t^\star \in \mathcal{S}} \  D_{\psi_{\Delta_\infty}}(\t^\star, \tilde{\t}_0)$ and that for $\Vert \tilde{\theta}_0 \Vert_\infty \ll \n{\t^{\texttt{MGF}}}_\infty$, $D_{\psi_{\Delta_\infty}}(\t^\star, \tilde{\t}_0) \approx \psi_{\Delta_\infty}(\theta^\star)$. We proved that for small values of $\l$, the balancedness remains strictly positive at all times, which allowed us to show that  $\Vert \tilde{\theta}_0 \Vert_\infty < \alpha^2$. 
We conjecture that
$\t^{\texttt{MGF}} \approx \argmin_{\t^\star \in \mathcal{S}} \psi_{\Delta_\infty}(\t^\star)$ continues to hold for larger values of $\l$. We experimentally test this claim by measuring the precise distance between $\t^{\texttt{MGF}}$ and $\t_{\Delta_\infty}^{\texttt{GF}} = \argmin_{\theta^\star \in \mathcal{S}} \psi_{\Delta_\infty}(\theta^\star) $. Indeed, we initialise a gradient flow with initial balancedness equal to  $\Delta_\infty$ and such that $\theta_0 = 0$, which converges to the predictor $\t_{\Delta_\infty}^{\texttt{GF}}$ as discussed in \Cref{gf_bias}. Hence, we can calculate the \textbf{Normalised Distance between} $\t^{\texttt{MGF}}$ \textbf{and} $\t_{\Delta_\infty}^{\texttt{GF}}$ equal to $\n{\t^{\texttt{MGF}} - \t_{\Delta_\infty}^{\texttt{GF}}}_2/\n{\t_{\Delta_\infty}^{\texttt{GF}}}_2$, and we obtain that $\n{\t^{\texttt{MGF}} - \t_{\Delta_\infty}^{\texttt{GF}}}_2/\n{\t_{\Delta_\infty}^{\texttt{GF}}}_2 < 0.01$ for $\l \in (0,1)$.

\begin{figure}[h!]
    \centering
    \includegraphics[width=1\textwidth]{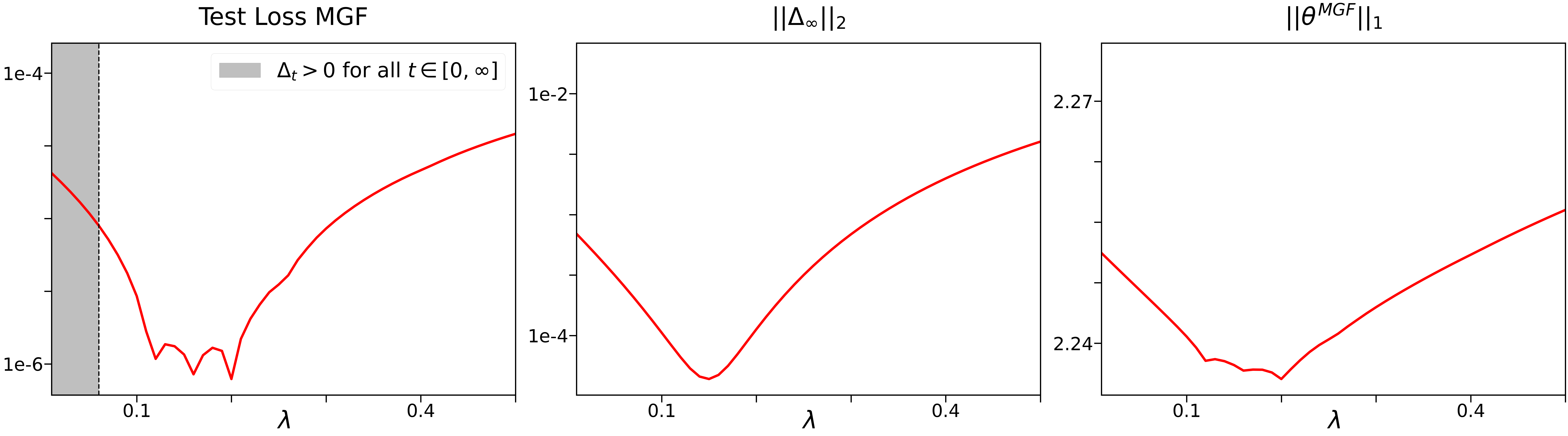}
    \vspace{-3mm}
    \caption{
    Continuous-time experiments on uncentered data with mean $\mu = 1$. Here, $\t^{\texttt{MGF}}$ denotes the interpolator recovered by \ref{mgf:lambda} and $\D_\infty$ stands for the balancedness at infinity for \ref{mgf:lambda}. We observe that the test loss and sparsity of $\t^{\texttt{MGF}}$ correlate with the magnitude of $\D_\infty$ as predicted by \Cref{main_mgf:general}.
    } 
    \label{fig:app:mgf_triple_1}
\end{figure}

\begin{figure}[h!]
    \centering
    \includegraphics[width=0.85\textwidth]{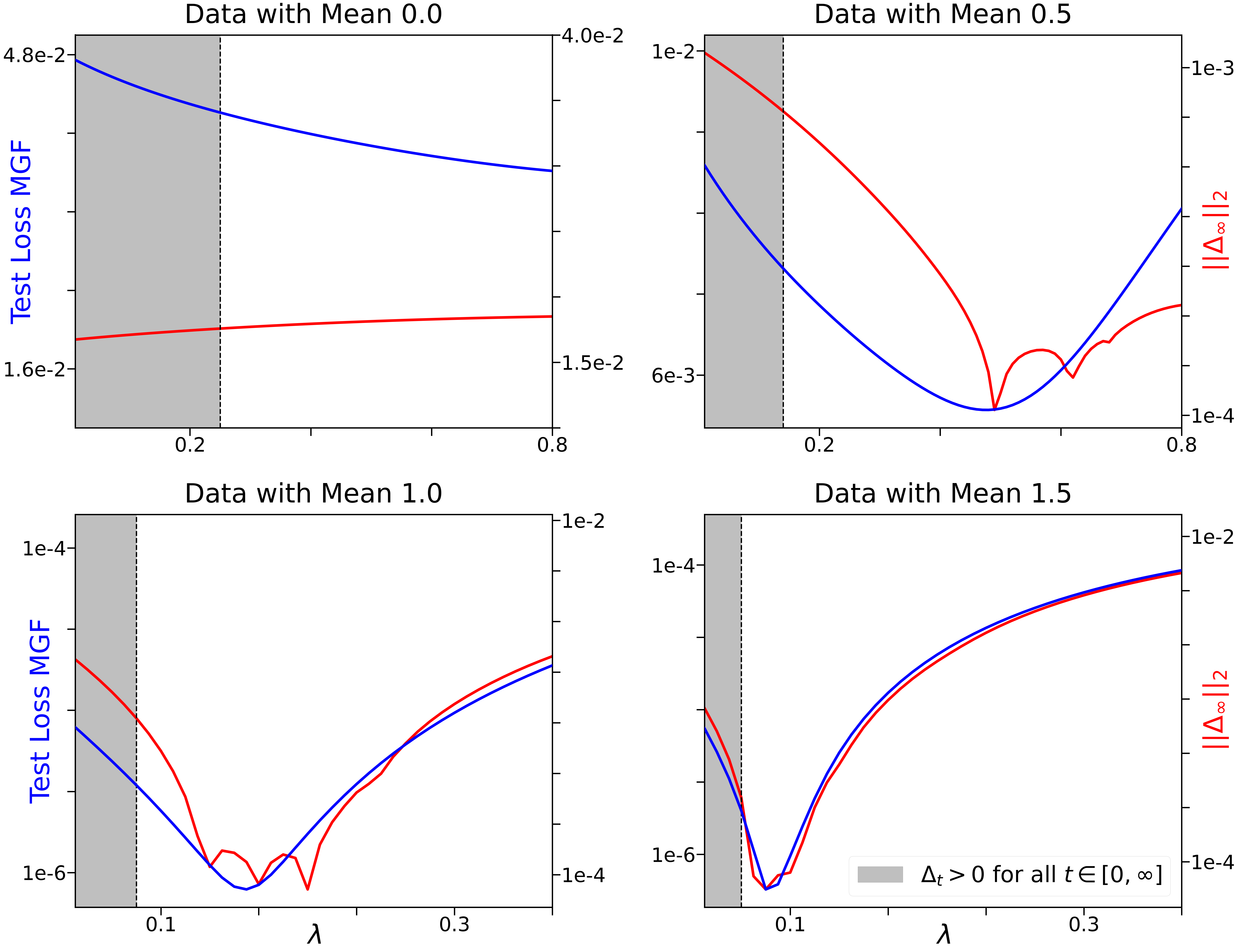}
    \vspace{-3mm}
    \caption{
    We observe that for uncentered data the magnitude of the balancedness at infinity $\D_\infty$ correlates with the test loss of the interpolator selected by \ref{mgf:lambda}. However, this relationship breaks for centered data.
    } 
    \label{fig:app:final_comparison}
\end{figure}

\myparagraph{Insights from Continuous-Time Experiments.}
First, we observe that no matter the mean of the data distribution\footnote{We performed the continuous-time experiments depicted in \Cref{fig:app:final_comparison} for data with mean $\mu = 0, 0.5, 1, 1.5$.} or the size of $\l \in (0,1)$, the normalised distance between $\t^{\texttt{MGF}}$ and $\t_{\Delta_\infty}^{\texttt{GF}}$ is always upper-bounded by $0.01$. 
Hence, we can empirically confirm our conjecture from \Cref{main_mgf:general} that $\t^{\texttt{MGF}} \approx \t_{\Delta_\infty}^{\texttt{GF}}$ for larger $\l$ when the balancedness changes sign.
Second, we see that regardless of the mean of the dataset, the balancedness at infinity (i.e., the effective initialisation $\Delta_\infty$) controls the generalisation behavior of the recovered interpolator. We can explain this observation again through the approximate equivalence $\t^{\texttt{MGF}} \approx \argmin_{\theta^\star \in \mathcal{S}} \psi_{\Delta_\infty}(\theta^\star)$.

\myparagraph{The Effect of the Data Mean.}
In \Cref{fig:app:final_comparison}, we summarise our empirical results for data with various means. Notice that there exists a difference between the generalisation behavior for centered and uncentered data. Indeed, for centered data (top left), the key quantity $\lambda$ has little impact on the sparsity of the recovered solution. This circumstance is reminiscent of the observations from \citep{pmlr-v162-nacson22a} and \citep{even_pesme_sgd}. However, for uncentered data, we observe an interval $\mathcal{I}_{\mathcal{D}_x} = (0, \l_{\max})$ (which depends on the data distribution $\mathcal{D}_x$) for which MGF with $\l \in \mathcal{I}_{\mathcal{D}_x}$ outperforms GF in terms of generalisation. Furthermore, there appears to exist a constant $\l^\star_{\mathcal{D}_x} \in \mathcal{I}_{\mathcal{D}_x}$ (roughly corresponding to the minimum magnitude of $\D_\infty$) which brings about the most improvement compared to gradient flow. We note that the following tendency seems to hold empirically:
\begin{equation*}
    \lim_{|\mu| \to +\infty} \l^\star_{\mathcal{D}_x} = 0. 
\end{equation*}

\subsubsection{Discrete-Time Plots}

For the sake of brevity\footnote{We performed discrete-time experiments for data with means $\mu = 0, 0.5, 1, 1.5$.}, we only present a single set of plots for the discrete-time noiseless sparse recovery given in \Cref{fig:mgd}. Our input data follows a unit-mean Gaussian distribution $\mathcal{N}(\mathbf{1}, I_d)$.

\myparagraph{Experimental Setup.}
For a sampled dataset $(X, y)$ and hyperparameter pair ($\b, \g$), we train our 2-layer diagonal linear network with \ref{mgd} initialised at $(u_0, v_0) = (\a \mathbf{1}, 0)$ for 1 million epochs (which suffices for convergence\footnote{Again, we record the \textbf{Train Loss} ($\t_{\gamma, \beta \a}^{\texttt{MGD}}$), which falls under $10^{-8}$.}). During the \ref{mgd} training, we also take note of whether the iterates $w_{\pm, k}$ change sign or not thereby checking the explanatory range of \Cref{main_mgd}.
Having completed the MGD training, we plot the \textbf{Test Loss} of $\t^{\texttt{MGD}}$, the $\ell_2$\textbf{-Norm of} $\D_\infty$, and the $\ell_1$\textbf{-Norm of} $\t^{\texttt{MGD}}$ in order to visualise the gain in generalisation performance.

\myparagraph{Insignificance of $\tilde{\theta}_0$.} Recall from \Cref{main_mgd:general} that $\t^{\texttt{MGD}} = \argmin_{\t^\star \in \mathcal{S}} D_{\psi_{\Delta_\infty}}(\t^\star, \tilde{\t}_0)$. Again, we want to characterise the recovered interpolator as $\t^{\texttt{MGD}} \approx \argmin_{\t^\star \in \mathcal{S}} \psi_{\Delta_\infty}(\theta^\star)$. In order to verify empirically that the effect of the perturbation term is negligible, we follow the same strategy as in the continuous-time case. We initialise a gradient flow with initial balancedness equal to  $\Delta_\infty$ and $\theta_0 = 0$, which converges to the predictor $\t_{\Delta_\infty}^{\texttt{GF}}$ as discussed in \Cref{gf_bias}. Hence, we can calculate the \textbf{Normalised Distance between} $\t^{\texttt{MGD}}$ \textbf{and} $\t_{\Delta_\infty}^{\texttt{GF}}$ equal to $\n{\t_{\g, \b, \a}^{\texttt{MGD}} - \t_{\Delta_\infty}^{\texttt{GF}}}_2/\n{\t_{\Delta_\infty}^{\texttt{GF}}}_2$, and we find that $\n{\t_{\g, \b, \a}^{\texttt{MGD}} - \t_{\Delta_\infty}^{\texttt{GF}}}_2/\n{\t_{\Delta_\infty}^{\texttt{GF}}}_2 < 0.01$ for all pairs $(\g, \beta)$ in \Cref{fig:mgd}.
This experimentally shows that $\t^{\texttt{MGD}} \approx \argmin_{\t^\star \in \mathcal{S}} \psi_{\Delta_\infty}(\theta^\star)$ and that the asymptotic balancedess is the key quantity which predicts the recovered solution.

\myparagraph{Insights from Discrete-Time Experiments.} 
As predicted by \Cref{main_mgd:general}, a more balanced solution (center plot) leads to a solution with a lower $\ell_1$-norm (right plot), which in turn translates to better generalisation (left plot). Finally, as proven in \Cref{main_mgd}, the trajectories for which the iterates do not cross zero satisfy $\Delta_\infty < \Delta_0$, where $\Delta_0$ (approximately) corresponds to the asymptotic balancedness for the pair $(\b, \g) = (0, 10^{-3})$ in the bottom left corner of the center plot. Clearly, the pairs $(\b, \g)$ for which $w_{\pm, k}$ do not change sign lead to better generalisation than the pair $(0, 10^{-3})$.
Again, we note that for centered data the story changes, and we lose the clear correspondence between small $\n{\D_\infty}_2$ and small $\n{\t^{\texttt{MGD}}}_1$.

\end{document}